\pdfoutput=1

\documentclass[english,11pt]{article}
\include{macros}

\begin{document}

\title{Adaptive and Robust Multi-Task Learning}\blfootnote{Author names are sorted alphabetically.}

\author{Yaqi Duan\thanks{Leonard N. Stern School of Business, New York University. Email: \texttt{yaqi.duan@stern.nyu.edu}.}
	\and Kaizheng Wang\thanks{Department of IEOR and Data Science Institute, Columbia University. Email: \texttt{kaizheng.wang@columbia.edu}.}
}

\date{This version: September 2023}

\maketitle

\begin{abstract}
We study the multi-task learning problem that aims to simultaneously analyze multiple datasets collected from different sources and learn one model for each of them. We propose a family of adaptive methods that automatically utilize possible similarities among those tasks while carefully handling their differences. We derive sharp statistical guarantees for the methods and prove their robustness against outlier tasks. Numerical experiments on synthetic and real datasets demonstrate the efficacy of our new methods.
\end{abstract}


\section{Introduction}\label{sec-intro}

Multi-task learning (MTL) solves a number of learning tasks simultaneously. It has become increasingly popular in modern applications with data generated by multiple sources.
When the tasks share certain common structures, a properly chosen MTL algorithm can leverage that to improve the performance. 
However, task relatedness is usually unknown and hard to quantify in practice; heterogeneity can even make multi-task approaches perform worse than independent task learning, which trains models separately on their individual datasets. In this paper, we study MTL from a statistical perspective and develop a family of reliable approaches that adapt to the unknown task relatedness and are robust against outlier tasks with possibly contaminated data.

To set the stage, let $m \geq 1$ be the number of tasks and $\{ \cX_j \}_{j=1}^m$ be sample spaces. For every $j \in [m]$, let $\cP_j$ be a probability distribution over $\cX_j$, $\cD_j = \{ \bxi_{ji} \}_{i=1}^{n_j}$ be samples drawn from $\cP_j$, and $\ell_j :~ \RR^d \times \cX_j \to \RR$ be a loss function. The $j$-th task is to estimate the population loss minimizer
\[
\btheta^{\star}_j \in \argmin_{\btheta \in \RR^d} \EE_{\bxi \sim \cP_j} \ell_j ( \btheta , \bxi )
\]
from the data. For instance, in multi-task linear regression, each sample $\bxi_{ji}$ can be written as $(\bx_{ji},y_{ji})$, where $\bx_{ji} \in \RR^d$ is a covariate vector and $y_{ji}$ is a response. The loss function is $\ell_j ( \btheta, (\bx, y) ) = ( \bx^{\top} \btheta - y )^2$.

Define the empirical loss function of the $j$-th task as $f_j (\btheta) = \frac{1}{n_j} \sum_{ i = 1 }^{n_j} \ell_j ( \btheta , \bxi_{ji} )$. 
Many MTL methods \citep{Car97} are formulated as constrained minimization problems of the form
\begin{align}
\min_{ \bTheta \in \Omega } \bigg\{ \sum_{ j = 1 }^m w_j f_j ( \btheta_j ) \bigg\},
\label{eqn-intro-mtl}
\end{align}
where $\bTheta = (\btheta_1,\cdots,\btheta_m) \in \RR^{d\times m}$, $\{ w_j \}_{j=1}^m$ are weight parameters (e.g., $w_j = n_j$), and $\Omega \subseteq \RR^{d\times m}$ encodes the prior knowledge of task relatedness. Independent task learning corresponds to $\Omega = \RR^{d\times m}$. Setting $\Omega = \{ \bbeta \bm{1}_m^{\top}:~\bbeta \in \RR^d \}$ yields the data pooling strategy, where we simply merge all datasets to train a single model. It is also easy to construct parameter spaces so that the learned parameter vectors share part of their coordinates, cluster around a few points, lie in a low-dimensional subspace, etc.
In general, the hard constraint $\bTheta \in \Omega$ in \eqref{eqn-intro-mtl} is overly rigid. 
When $\Omega$ fails to reflect the task structures, the model misspecification may have a huge negative impact on the performance.

To resolve the aforementioned issue, we propose to solve an augmented program
\begin{align}
\min_{ \bTheta \in \RR^{d\times m},~\bGamma \in \Omega } \bigg\{ \sum_{ j = 1 }^m w_j [ f_j ( \btheta_j ) + \lambda_j \| \btheta_j - \bgamma_j \|_2 ] \bigg\},
\label{eqn-intro-armul}
\end{align}
obtain an optimal solution $(\widehat{\bTheta} , \widehat{\bGamma})$ and then use $\widehat{\bTheta}$ as the final estimate. Here $\{ \lambda_j \}_{j=1}^m$ are regularization parameters and $\bGamma = (\bgamma_1,\cdots,\bgamma_m) $. 
Each task receives its own estimate $\widehat{\btheta}_j$, while the penalty terms shrink $\widehat\bTheta$ toward a prototype $\widehat{\bGamma}$ in the prescribed model space $\Omega$ so as to promote relatedness among tasks.
Our framework \eqref{eqn-intro-armul} can deal with different levels of task relatedness if we properly tune the regularization parameters $\{ \lambda_j \}_{j=1}^m$.
When $\Omega$ nicely captures the underlying structure, we can pick sufficiently large $\{ \lambda_j \}_{j=1}^m$ so that the cusp of the $\ell_2$ penalty at zero enforces the strict equality $\widehat{\bTheta} = \widehat{\bGamma}$. The new procedure then reduces to the classical formulation \eqref{eqn-intro-mtl}. On the other hand, when $\Omega$ fails to reflect the structure, we take small $\lambda_j$'s to guarantee each $\widehat{\btheta}_j$'s fidelity to its associated data. Observe that
\[
\widehat{\btheta}_j \in \argmin\nolimits_{\btheta \in \RR^d} \{ f_j(\btheta) + \lambda_j \| \btheta - \widehat\bgamma_j \|_2 \}, \qquad\forall j \in [m].
\]
In words, $\widehat{\btheta}_j$ minimizes a perturbed version of the loss function $f_j$ associated to the $j$-th task. When $\lambda_j$ is not too large, the perturbation has limited influence and $\widehat{\btheta}_j$ stays close to the output of independent task learning $ \argmin_{\btheta \in \RR^d} f_j(\btheta) $.
This provides a safenet in case $\Omega$ is significantly misspecified.
We see that strong regularization helps utilize task relatedness if that exists, while weak regularization better deals with heterogeneity. 

Interestingly, there is a simple choice of $\{ \lambda_j \}_{j=1}^m$ that provides the best of both worlds, regardless of whether the prescribed model space $\Omega$ captures the underlying structure or not.
Roughly speaking, when $n_1 = \cdots = n_j = n$, our theory suggests choosing $w_j = 1$ and $\lambda_j = c \sqrt{ \frac{ d }{n} }$ for some constant $c$; when $\{ n_j \}_{j=1}^m$ are different, our general results recommend $w_j = n_j$ and $\lambda_j = c \sqrt{ \frac{ d }{n_j} }$. In both cases, the factor $c$ is shared by all of the $m$ tasks. The estimator has a single tuning parameter rather than $m$ different ones, which is practically appealing.
Thanks to the \emph{unsquared} $\ell_2$ penalties in \eqref{eqn-intro-armul}, the procedure automatically enforces an appropriate degree of relatedness among the learned models.

Moreover, the method can tolerate a reasonable fraction of exceptional tasks that are dissimilar to others or even have their data contaminated. Given the above merits, we name the framework as \emph{\underline{A}daptive and \underline{R}obust \underline{MU}lti-task \underline{L}earning}, or ARMUL for short. 

\paragraph*{Main contributions}
Our contributions are two-fold. \vspace{-0.6em}
\begin{itemize} \itemsep = -0.1em
\item (Methodology) We introduce a flexible framework for multi-task learning. It works as a wrapper around any MTL method of the form \eqref{eqn-intro-mtl}, enhancing its ability to handle heterogeneous tasks. 
\item (Theory) We establish sharp guarantees for the framework on its adaptivity and robustness. Our analysis provides one customized statistical error bound for every single task.
\end{itemize}

\paragraph*{Related work} 
Our work relates to a vast literature on integrative data analysis \citep{Len02}. A classical example is simultaneous estimation of multiple Gaussian means. 
The specification of our method in this scenario is related to various shrinkage estimates \citep{JSt61,EMo72,DJH92}.
See \Cref{sec-warmup} for more discussions.
An extension of multi-task mean estimation is linear regression with multiple responses \citep{BFr97}, which is a special form of multi-task linear regression with shared covariates.
\cite{COS15} studied Stein-type shrinkage estimates for multi-task linear regression with Gaussian data.
\cite{NWa08}, \cite{OWJ11}, \cite{LPV11}, \cite{TSo16}, \cite{AOC19}, \cite{MSB19}, \cite{XBa21} and \cite{CLX21} investigated high-dimensional (generalized) linear MTL where the tasks have similar sparsity patterns.
There are also MTL approaches proposed to enforce other types of model similarities such as clustering structures \citep{EPo04,EMP05,JBV08,MPS16}, low-rank structures \citep{AZh05,AEP08,KDa12}, among others.
The above list is far from being exhaustive.

Our study is largely motivated by the great empirical success in MTL with parameter augmentation \citep{EPo04,JSR10,CZY11}. Our idea of non-smooth regularization originates from the seminal works by \cite{DJH92} and \cite{DJo94} on adaptive sparse estimation.
Beyond the coordinate-wise sparsity of vectors, recent studies have developed the sum of $\ell_2$ penalties to promote column-wise sparsity in matrix estimation problems such as
robust PCA \citep{LLY10,MTr11,XCS12} and robust low-rank MTL \citep{PTJ10,CZY11}. 
Our design of the penalty is closely related to theirs.
The ARMUL penalty also looks similar to the group lasso penalty $\sum_{ \ell = 1 }^d ( \sum_{ j = 1 }^m |\theta_{j \ell}|^2 )^{1/2}$ for variable selection in sparse MTL \citep{LPV11}. 
While the group lasso sums up the norms of rows (variables), ours does that to the columns (tasks).

Below we provide a selective overview of existing theories that are closely connected to our analysis of adaptivity and robustness.
\cite{WZR20} and \cite{MFA20} analyzed the impact of task relatedness on linear models and one-hidden-layer neural networks when there are two tasks.
\cite{BKT19,DCG19} studied online MTL and showed the benefit of task relatedness.
\cite{KFA20} investigated multi-task PAC learning with adversarial corruptions. They assumed homogeneous tasks and focused on robustness against different types of adversaries.
\cite{HKp20} studied the adaptation in nonparametric MTL under the Bernstein class condition. 
\cite{DHK20} and \cite{TJJ20} considered representation learning from multiple datasets when the true statistical models share common latent structures. 
In the agnostic learning framework, \cite{Bax00} and \cite{MPR16} presented generalization bounds on the average risk across tasks, and \cite{BDS03} studied task-specific error bounds.


\paragraph*{Outline} The rest of the paper is organized as follows. \Cref{sec-warmup} studies multi-task Gaussian mean estimation as a warm-up example.
\Cref{sec-methodologies} presents the methodology. \Cref{sec-theory} conducts a sharp analysis of adaptivity and robustness. \Cref{sec-numerical} verifies the theories and tests the methodology through numerical experiments.
Finally, \Cref{sec-discussions} concludes the paper and discusses possible future directions.

\paragraph*{Notation}
The constants $c_1, c_2, C_1, C_2, \cdots$ may differ from line to line. 
Define $x_+ = \max\{ x, 0 \}$ for $x \in \RR$.
We use the symbol $[n]$ as a shorthand for $\{ 1, 2, \cdots, n \}$ and $| \cdot |$ to denote the absolute value of a real number or cardinality of a set. 
For nonnegative sequences $\{ a_n \}_{n=1}^{\infty}$ and $\{ b_n \}_{n=1}^{\infty}$, we write $a_n \lesssim b_n$ or $a_n = O(b_n)$ or $b_n = \Omega(a_n)$ if there exists a positive constant $C$ such that $a_n \leq C b_n$. In addition, we write $a_n \asymp b_n$ if $a_n \lesssim b_n$ and $b_n \lesssim a_n$; $a_n = o(b_n)$ if $a_n = O(c_n b_n)$ for some $c_n \to 0$.
Let $\bm{1}_d$ be the $d$-dimensional all-one vector and $\{ \be_j \}_{j=1}^d$ canonical bases of $\RR^d$.
Define $\SSS^{d-1} = \{ \bx \in \RR^d:~ \| \bx \|_2 = 1 \}$ and $B(\bx, r) = \{ \by \in \RR^d:~ \| \by - \bx \|_2 \leq r \}$ for $\bx \in \RR^d$ and $r \geq 0$. 
For any matrix $\bA$, we use $\ba_j$ to refer to its $j$-th column and let $\Range(\bA)$ be its column space. $\| \bA \|_2 = \sup_{\|\bx\|_2 = 1 } \| \bA \bx \|_2$ denotes the spectral norm and $\| \bA \|_{\mathrm{F}}$ denotes the Frobenius norm. 
Define $\| X \|_{\psi_2} = \sup_{p \geq 1} \{ p^{-1/2} \EE^{1/p} |X|^p \}$ and $\| X \|_{\psi_1} = \sup_{p \geq 1} \{ p^{-1} \EE^{1/p} |X|^p \}$ for a random variable $X$; $\|\bX\|_{\psi_2} = \sup_{\| \bu \|_2 = 1} \|\dotp{\bu}{\bX}\|_{\psi_2} $ for a random vector $\bX$.

\section{Warm-up: estimation of multiple Gaussian means}\label{sec-warmup}

In this section, we consider the multi-task mean estimation problem as a warm-up example. We first introduce the setup and a simple estimation procedure. We relate the estimator to soft thresholding and Huber's location estimator. Then, we show that it automatically adapts to the unknown task relatedness and is robust against a small fraction of tasks with contaminated data. 
Finally, we discuss the connection between our estimator and several fundamental topics in statistics and machine learning.

\subsection{Problem setup}\label{sec-warmup-setup}

Suppose we want to simultaneously estimate the mean parameters of $m \geq 2$ Gaussian distributions $\{ N(\theta^{\star}_j, 1)  \}_{j=1}^m$. For each $j \in [m]$, we collect $n$ i.i.d.~samples $\{  x_{ji} \}_{i=1}^n$ from $N( \theta^{\star}_j, 1)$. The $m$ datasets $\{  x_{1i} \}_{i=1}^n, \cdots, \{  x_{mi} \}_{i=1}^n$ are independent.
This is an extensively studied problem in statistics \cite{Ste81,EHa16} and a canonical example in multi-task learning, where the $j$-th learning task is to estimate $\theta^{\star}_j$.
\vspace{-0.6em}
\begin{itemize} \itemsep = -0.2em
	\item Without additional assumptions, it is natural to conduct maximum likelihood estimation (MLE). Due to the independence of datasets, MLE amounts to estimating each $\theta^{\star}_j$ by the sample mean $\bar{x}_j = \frac{1}{n} \sum_{ i = 1 }^n  x_{ji}$ of its associated data. The mean squared error is $\EE (\bar{x}_j - \theta^{\star}_j)^2 = \frac{1}{n}$.
	
	\item  If the parameters are very close, we may estimate them by the pooled sample mean $\bar{x} = \frac{1}{mn}
	\sum_{j=1}^m \sum_{i=1}^n  x_{ji}$. In the ideal case $\theta^{\star}_1 =\cdots = \theta^{\star}_m$, data pooling reduces the mean squared error to $\frac{1}{mn}$.

	\item We may use Bayesian procedures if $\{ \theta^{\star}_j \}_{j=1}^m$ are independently drawn from some known prior distribution. When the prior itself has unknown parameters, empirical Bayes methods \citep{JSt61,EMo73} can be applied.
\end{itemize}

Since it is often hard to precisely quantify the prior knowledge in practice, we want an estimation procedure that automatically adapts to the unknown similarity among the tasks. 
Ideally, the procedure should also be robust against outlier tasks that are dissimilar to others or even contain corrupted data. 
To introduce our method, we first present optimization perspectives of MLE and its pooled version. Up to an affine transform, the negative log-likelihood function for the $j$-th task is equal to
\[
f_j (\theta) = \frac{1}{2n}  \sum_{i=1}^n (  x_{ji} - \theta )^2  , \qquad\forall \theta \in \RR.
\]
MLE returns one estimator $\bar{x}_j = \argmin_{\theta_j \in \RR }  f_j (\theta_j) $ for each task, whereas data pooling outputs the same estimator $\bar{x} = \argmin_{\theta \in \RR } 
\sum_{ j=1 }^m f_j (\theta) 
 $ for all tasks.

We propose to solve a convex optimization problem
\begin{align}
(\widehat{\theta}_1, \cdots, \widehat{\theta}_m , \widehat{\theta}) \in \argmin_{ \theta_1,\cdots,\theta_m,\theta \in \RR } \bigg\{ \sum_{j=1}^{m} [ f_j (\theta_j) + \lambda |\theta_j - \theta | ] \bigg\} 
\label{eqn-warmup-armul}
\end{align}
and use $\{ \widehat{\theta}_j \}_{j=1}^m$ to estimate $\{ \theta^{\star}_j \}_{j=1}^m$. Here $\lambda \geq 0$ is a penalty parameter and $\widehat\theta$ serves as a global coordinator. Similar to MLE, each task receives one individual estimator based on its loss function. Moreover, the penalty terms drive those estimators toward a common center. When $\lambda = 0$, $\widehat{\theta}_j = \bar{x}_j$. When $\lambda = \infty$, $\widehat{\theta}_j = \widehat{\theta} = \bar{x}$. Therefore, the method interpolates between MLE and its pooled version. We will derive a simple choice of $\lambda$ with guaranteed quality outputs. 


\subsection{Adaptivity and robustness}\label{sec-warmup-theory}

It is easily seen from \eqref{eqn-warmup-armul} that
\begin{align}
 \widehat{\theta} \in \argmin_{\theta \in \RR}  \sum_{j=1}^{m} \widetilde{f}_j(\theta)
\qquad\text{and}\qquad
\widehat{\theta}_j \in \argmin_{\theta \in \RR} \{ f_j (\theta) + \lambda |\theta - \widehat\theta | \}, ~~\forall j \in [m] ,
\label{eqn-warmup-expressions}
\end{align}
where $\widetilde{f}_j(\theta) = \min_{\xi \in \RR} \{ f_j(\xi) + \lambda |\theta - \xi| \}$ is the infimal convolution \citep{HLe13} of a quadratic loss $f_j(\cdot)$ and an absolute value penalty $\lambda|\cdot|$. 
It is well-known that such infimal convolution is closely related to the Huber loss function \citep{Hub64} with parameter $\lambda$:
\begin{align*}
	\rho_{\lambda}(x) = \begin{cases}
		x^2/2 &,\mbox{ if } |x| \leq \lambda \\
		\lambda (|x| - \lambda/2) &,\mbox{ if } |x| > \lambda 
	\end{cases}.
\end{align*}
See, for example, Section 6.1 of \cite{DMo16}.
Based on that, we have the following elementary characterizations of $\widehat{\theta}$ and $\{ \widehat{\theta}_j \}_{j=1}^m$. The proof is deferred to Appendix C.1.

\begin{lemma}\label{thm-warmup-analytical}
We have $ \widetilde{f}_j (\theta) = \rho_{\lambda}(\theta - \bar{x}_j) + \frac{1}{2n}  \sum_{i=1}^n (  x_{ji} - \bar{x}_j )^2$,
\begin{align*}
&\widehat{\theta} \in \argmin_{\theta \in \RR}  \sum_{j=1}^{m} \rho_{\lambda}(\theta - \bar{x}_j) ,\\
&\widehat{\theta}_j 
= \widehat{\theta} + \sgn( \bar{x}_j - \widehat{\theta} )  (  |\bar{x}_j - \widehat{\theta}| -  \lambda   )_{+}  = \bar{x}_j - \min \{ \lambda, |  \bar{x}_j - \widehat{\theta} | \} \sgn( \bar{x}_j - \widehat{\theta}), \quad\forall j \in [m].
\end{align*}
\end{lemma}

According to \Cref{thm-warmup-analytical}, the global coordinator $\widehat{\theta}$ in \eqref{eqn-warmup-armul} is a Huber estimator applied to sample means $\{ \bar{x}_j \}_{j=1}^m$ of individual datasets. 
The estimators $\{ \widehat{\theta}_j \}_{j=1}^m$ for $\{ \theta^{\star}_j \}_{j=1}^m$ are shrunk toward $\widehat{\theta}$ by soft thresholding. Intuitively, we may view the procedure \eqref{eqn-warmup-armul} as a combination of hypothesis testing and parameter estimation. The first step is to test the homogeneity hypothesis $H_0:~\theta^{\star}_1=\cdots=\theta^{\star}_m$, with $\lambda$ controlling the significance level.
When $\{ \bar{x}_j \}_{j=1}^m$ are close enough, e.g.~$\max_{j \neq k} | \bar{x}_j - \bar{x}_k | \leq \lambda$, the parameters $\{ \theta^{\star}_j \}_{j=1}^m$ do not seem to be significantly different. 
We apply data pooling and get $ \widehat{\theta}_1 = \cdots = \widehat{\theta}_m  = \widehat{\theta} = \bar{x}$. The exact equality $\widehat{\theta}_j = \widehat{\theta}$ is enforced by the cusp of the absolute value penalty $|\cdot|$ at zero.
When all but a small fraction of $\{ \bar{x}_j \}_{j=1}^m$ are close, the robustness property of the Huber loss makes $\widehat{\theta}$ a good summary of the majority; their corresponding $\widehat{\theta}_j$'s are equal to $\widehat{\theta}$.
In general, the estimators $\{ \widehat{\theta}_j \}_{j=1}^m$ can be different. It is worth pointing out that $| \widehat{\theta}_j - \bar{x}_j | \leq \lambda$ always holds, thanks to the Lipschitz smoothness of $|\cdot|$. This guarantees $\widehat{\theta}_j$'s fidelity to its associated dataset $\{ x_{ji} \}_{i=1}^n$. Hence the proposed method easily handles heterogeneous tasks.

To analyze the statistical property of \eqref{eqn-warmup-armul}, we need to gauge the relatedness among tasks. 
\begin{definition}[Parameter space]\label{defn-warmup-space}
For any $\varepsilon \in [0, 1]$ and $\delta \geq 0$, define
\begin{align*}
\Omega (\varepsilon, \delta ) = \Big\{ \btheta^{\star} \in \RR^m:~\min_{\theta \in \RR}\max_{j \in S} |\theta^{\star}_j - \theta  | \leq \delta \text{ and }
|S^c| / m \leq \varepsilon 
\text{ for some } S \subseteq [m]
\Big\}.
\end{align*}
We associate every $\btheta^{\star} \in \Omega (\varepsilon, \delta ) $ with a subset $S = S (\btheta^{\star} )$ of $[m]$ that satisfies the above requirements.
\end{definition}


\begin{assumption}[Task relatedness]\label{as-warmup-similarity}
The $m$ datasets $\{ x_{1i} \}_{i=1}^n,\cdots,\{ x_{mi} \}_{i=1}^n$ are statistically independent and there exists $\btheta^{\star} \in \Omega (\varepsilon, \delta )$ such that for any $j \in [m]$, $ \{ x_{ji} \}_{i=1}^n$ are i.i.d.~$N(\theta^{\star}_j , 1)$.
\end{assumption}

We say the $m$ tasks are $(\varepsilon,\delta)$-related when Assumption \ref{as-warmup-similarity} holds.
In words, all but an $\varepsilon$ fraction of the mean parameters $\{ \theta^{\star}_j \}_{j=1}^m$ live in an interval with half-width $\delta$; the others can be arbitrary. Smaller $\varepsilon$ and $\delta$ imply more similarity among tasks. The extreme case $\varepsilon = \delta = 0$ corresponds to $\theta^{\star}_1=\cdots = \theta^{\star}_m$. 
Any $m$ tasks of Gaussian mean estimation are $(0,  \max_{j \in [m] } |\theta^{\star}_j  | )$-related.

\Cref{thm-warmup-armul} below characterizes the estimation errors. The proof can be found in Appendix C.2.

\begin{theorem}[Adaptivity and robustness]\label{thm-warmup-armul}
Let Assumption \ref{as-warmup-similarity} hold. Choose any $t \geq 2$ and $\lambda = 6 \sqrt{\frac{2 ( \log m + t ) }{n}}$.
There is a universal constant $C>0$ such that with probability at least $1 - e^{-t}$,
\begin{align*}
& \max_{j \in S } | \widehat{\theta}_j  - \theta^{\star}_j |
< C \bigg(
\sqrt{ \frac{t}{mn} } +   \min\bigg\{   \delta ,~ \sqrt{\frac{\log m + t}{n} } \bigg\}  +   \varepsilon  \sqrt{\frac{\log m + t}{n} } 
\bigg) , 
\\
& \max_{j \in [m] \backslash S  } | \widehat{\theta}_j  - \theta^{\star}_j |
< C \sqrt{\frac{\log m + t}{n} } , \notag\\
&  	\frac{1}{m} \sum_{j=1}^{m}  | \widehat{\theta}_j  - \theta^{\star}_j |^2  
 \leq  C \bigg(
 \frac{t}{mn}  +   \min\bigg\{   \delta^2 ,~   \frac{\log m + t}{n}   \bigg\}  +    \varepsilon \cdot \frac{\log m + t}{n} 
\bigg) .
\end{align*}
\end{theorem}

\begin{remark}[Data contamination]
We can further relax the assumption on task relatedness to allow the datasets $\{ \cD_{j} \}_{j \notin S   }$ to be arbitrarily contaminated. In that case, the results for $ \max_{j \in S  } | \widehat{\theta}_j  - \theta^{\star}_j |$ in \Cref{thm-warmup-armul} continue to hold.
\end{remark}


\Cref{thm-warmup-armul} provides maximum error bounds for the ``good'' tasks in $S $ and ``bad'' tasks in $[m] \backslash S  $, as well as the mean squared error (MSE) over all tasks. 
A crude error bound $\max_{j \in [m]} | \widehat{\theta}_j  - \theta^{\star}_j | \lesssim \sqrt{\frac{\log m}{n}}$ always holds regardless of $( \varepsilon, \delta)$. On the other hand, elementary calculation shows that $\max_{j \in [m]} | \bar{x}_j  - \theta^{\star}_j | \gtrsim \sqrt{\frac{\log m}{n}}$ with constant probability. Therefore, the new method is always comparable to MLE. That provides a safe net. 

Moreover, the suggested penalty parameter $\lambda = 6 \sqrt{\frac{2 ( \log m + t ) }{n}}$ in \Cref{thm-warmup-armul} does not depend on $\varepsilon$ or $\delta$ at all. The estimator automatically adapts to the unknown task relatedness, achieving higher accuracy when $\varepsilon$ and $\delta$ are small. Up to logarithmic factors, the MSE bound reads
\begin{align}
\frac{1}{m} \sum_{j=1}^{m}  | \widehat{\theta}_j  - \theta^{\star}_j |^2  
\lesssim
\frac{1}{mn}  +   \min\bigg\{   \delta^2 ,~   \frac{1}{n}   \bigg\}  +   \frac{\varepsilon}{n} .
\label{eqn-warmup-mse}
\end{align}
The first term $ \frac{1}{mn } $ is the MSE of the pooled sample mean $\bar{x}$ in the most homogeneous case $\theta^{\star}_1=\cdots=\theta^{\star}_m$. When $\varepsilon = \delta = 0$, only this term exists and our procedure reduces to data pooling. 
The second term $ \min \{   \delta^2 ,  \frac{1}{n}  \}  $ is non-decreasing in the discrepancy $\delta$ among $\{ \theta^{\star}_j \}_{j \in S}$. It increases first and then flattens out, never exceeding the error rate of MLE. When $\varepsilon = 0$, we have
	\[
\frac{1}{m} \sum_{j=1}^{m}  | \widehat{\theta}_j  - \theta^{\star}_j |^2  
	\lesssim 
	\frac{1}{ mn } 
	+
	\min\bigg\{   \delta^2 ,~ \frac{1}{n}  \bigg\}  
	\asymp 
	\min\bigg\{  \frac{1}{mn } + \delta^2 ,~ \frac{1}{n}  \bigg\} .
	\]
Here $\frac{1}{n} $ and $\frac{1}{ mn } + \delta^2 $ are the MSEs of the MLE and its pooled version, respectively. Therefore, the new method achieves the smaller error between the two. 
It is closely related to robust inference procedures considered by \cite{HLe52}, \cite{Bic84} and others that $\mathrm{(i)}$ perform well when the parameter of interest $\btheta^{\star}$ truly lives in a small set (e.g. $\theta_1^{\star} = \cdots = \theta_m^{\star}$), and $\mathrm{(ii)}$ are nearly minimax optimal over a larger parameter space (e.g. $\RR^m$). Our analysis covers a continuum of parameter spaces $\{ \Omega( \varepsilon , \delta ) :~ 0\leq \varepsilon \leq 1 , \delta \geq 0 \}$ while those studies mostly look at the two extremes.


When $\varepsilon > 0$, the third term $\frac{ \varepsilon }{n} $ in \eqref{eqn-warmup-mse} is the price we pay for not knowing the index set $S^c$ of tasks that may be very different from the others. As an illustration, suppose that $\delta = 0$ and $\{ \theta^{\star}_j \}_{j \in S}$ are all equal to some $\theta^{\star}$. Then, $\{ \bar{x}_j \}_{j \in S}$ are i.i.d.~$N(\theta^{\star},1/n)$ and $\{ \bar{x}_j \}_{j \in S^c}$ can be arbitrary. $\widehat{\theta}$ is a Huber estimator of $\theta^{\star}$ based on $\varepsilon$-contaminated data $\{ \bar{x}_j \}_{j=1}^m$. Our error bound has optimal dependence on $\varepsilon$ up to a logarithmic factor \citep{Hub04}, whereas the pooled MLE can be ruined by a single outlier task.


We now present a minimax lower bound for an idealized problem with known $\varepsilon$ and $\delta$.
It is a special case ($d = 1$) of \Cref{thm-minimax-00} for multivariate Gaussians.

\begin{theorem}[Minimax lower bound]\label{thm-warmup-minimax}
There exist universal constants $C,c>0$ such that for any $\varepsilon \in [0, 1]$ and $\delta \geq 0$,
\begin{align*}
\inf_{\widehat{\btheta}} \sup_{ \btheta^{\star} \in \Omega (\varepsilon, \delta) } \PP_{\btheta^{\star} } \bigg[
\frac{1}{m} \sum_{j=1}^{m} | \widehat{\theta}_j  - \theta^{\star}_j |^2
\geq C \bigg( \frac{1}{ mn } 
+
\min\bigg\{   \delta^2 ,~  \frac{1}{n}  \bigg\}  
+  \frac{\varepsilon}{  n } \bigg)
\bigg] \geq c.
\end{align*}
\end{theorem}


The ARMUL estimator achieves the oracle error up to a $\log m$ factor without knowing $\varepsilon$ and $\delta$.
It would be interesting to investigate whether the logarithmic term is a fundamental price of adaptation, as is the case with sparse Gaussian mean estimation \citep{DJo94}.

For any given $\btheta^{\star} \in \RR^m$, there exist infinitely many pairs of $( \varepsilon , \delta )$ that make Assumption \ref{as-warmup-similarity} hold. For instance, when $\btheta^{\star} = \be_1$, we can take any $(\varepsilon, \delta)$ in the set
\[
\{ (\varepsilon , \delta) \in [0, 1] \times [0, +\infty) :~ \varepsilon \geq 1/m \text{ or } \delta \geq 1 \}.
\]
The MSE bound in \Cref{thm-warmup-armul} holds simultaneously for all of those $(\varepsilon , \delta)$.
Unfortunately, the bound is not directly computable from data. On the one hand, $\varepsilon$ and $\delta$ are not uniquely defined. On the other hand, even if we set $\varepsilon = 0$, the estimation error of $\delta$ will be of order $1/\sqrt{n}$. This results in an error up to $O(1/n)$ in the estimated MSE bound and makes it meaningless, because $O(1/n)$ is the largest possible value of our MSE bound (up to a $\log m$ factor).
Similar phenomenon arises in nonparametric estimation. As \cite{Low97} pointed out, ``although an estimate may be adaptive for squared error loss it may be impossible to make a data dependent claim on how well you have done''.

\subsection{Discussions}\label{sec-warmup-discussions}

The estimation procedure and theory in this section have deep connections to several fundamental topics in statistics and machine learning.

\subsubsection{James-Stein estimators}

For the Gaussian mean estimation problem in \Cref{sec-warmup-setup}, a sufficient statistic is $\sqrt{n} (\bar{x}_1,\cdots,\bar{x}_m)^{\top} \sim N(\btheta^{\star}, \bI_m)$. Therefore, we may assume $n = 1$ in the original problem without loss of generality. The goal then becomes estimating $\btheta^{\star} \in \RR^m$ from a single sample $\bx \sim N(\btheta^{\star}, \bI_m)$. The MLE is $\widehat{\btheta}^{\mathrm{MLE}} = \bx$. In a seminal paper, \cite{JSt61} proposed to shrink the MLE toward zero and introduce a new estimator $\widehat{\btheta}^{\mathrm{JS,0}} = (1 - \frac{m-2}{\| \bx \|_2^2}) \bx$. Surprisingly, when $m \geq 3$, the $\ell_2$ risk of $\widehat{\btheta}^{\mathrm{JS,0}}$ is always strictly smaller than that of $\widehat{\btheta}^{\mathrm{MLE}}$:
\begin{align}
\EE_{\btheta^{\star}} \| \widehat{\btheta}^{\mathrm{JS,0}} -  \btheta^{\star} \|_2^2
< \EE_{\btheta^{\star}} \| \widehat{\btheta}^{\mathrm{MLE}} - \btheta^{\star} \|_2^2 , \qquad\forall \btheta^{\star} \in \RR^d.
\label{eqn-warmup-JS}
\end{align}
The shrinking point does not have to be $\bm{0}$. They also introduced another estimator 
\[
\widehat{\theta}^{\mathrm{JS}}_j = \bar{x} + \bigg( 1 - \frac{m - 3}{ \sum_{ j = 1 }^m (x_j - \bar{x})^2 } \bigg) (x_j - \bar{x}) , \qquad\forall j \in [m],
\] 
whose entries are shrunk toward the pooled sample mean $\bar{x}$. They proved the same dominance as \eqref{eqn-warmup-JS} for $\widehat{\btheta}^{\mathrm{JS}} $ when $m \geq 4$. The gain is the most significant when $\{ \theta^{\star}_j \}_{j=1}^m$ are close and $m$ is large. In the ideal case $\theta^{\star}_1 = \cdots = \theta^{\star}_m$, we derive from equation (7.14) in \cite{EHa16} that $\EE_{\btheta^{\star}} \| \widehat{\btheta}^{\mathrm{JS}} -  \btheta^{\star} \|_2^2 = 3$, which is within a constant factor (3) times the $\ell_2$ risk of the pooled sample mean. The MLE has risk $\EE_{\btheta^{\star}} \| \widehat{\btheta}^{\mathrm{MLE}} - \btheta^{\star} \|_2^2 = m$.

\cite{EMo73} adopted an empirical Bayes approach to the simultaneous estimation problem and derived class of estimators that dominate the MLE. The positive part version of the James-Stein estimator
\[
\widehat{\theta}^{\mathrm{JS+}}_j = \bar{x} + \bigg( 1 - 
\frac{m - 3}{ \sum_{ j = 1 }^m (x_j - \bar{x})^2 } 
\bigg)_{+} (x_j - \bar{x}) , \qquad\forall j \in [m]
\] 
is one example, which avoids negative shrinkage factor.
The lemma below connects $\widehat{\btheta}^{\mathrm{JS}} $ and $\widehat{\btheta}^{\mathrm{JS+}} $ to multi-task learning with ridge regularization \citep{EPo04}. See the proof in Appendix C.3.

\begin{lemma}\label{lem-warmup-mean}
Let $\lambda > 0$ and  
\begin{align}
(\widetilde{\theta}_1, \cdots, \widetilde{\theta}_m , \widetilde{\theta}) \in \argmin_{ \theta_1,\cdots,\theta_m,\theta \in \RR } \bigg\{ \sum_{j=1}^{m} [ (\theta_j - x_j)^2 + \lambda ( \theta_j - \theta )^2 ] \bigg\} .
\label{eqn-lem-warmup-mean}
\end{align}
We have $\widetilde{\theta} = \bar{x} $ and $\widetilde{\theta}_j  = \bar{x} + \frac{1}{1+\lambda} (x_j - \bar{x})$, $\forall j \in [m]$. If we define $S = \sum_{ j = 1 }^m (x_j - \bar{x})^2 $, then
\begin{itemize} \itemsep = -0.2em
\item $\widetilde{\btheta}  = \widehat{\btheta}^{\mathrm{JS}}$ when $ S > m - 3 $ and $\lambda = \frac{m - 3}{S -(m-3)}$; 
\item $\widetilde{\btheta}  = \widehat{\btheta}^{\mathrm{JS+}}$ when $\lambda = \frac{ \min \{ S , m - 3 \} }{  S - \min \{ S , m - 3 \} }
$, with the convention that $c/0 = +\infty$ for any $c>0$.
\end{itemize}
\end{lemma}

Our estimator $\widehat{\btheta}$ is defined by the $\ell_1$-regularized program \eqref{eqn-warmup-armul} that differs from \eqref{eqn-lem-warmup-mean} in the penalty function. As a result, the entries $\{ \widehat{\theta}_j \}_{j=1}^m$ are shrunk toward a Huber estimator $\widehat{\theta}$ instead of the pooled sample mean used by James-Stein estimators, see \Cref{thm-warmup-analytical}. The non-smooth $\ell_1$ penalty can shrink the difference $\widehat{\theta}_j - \widehat{\theta}$ to exact zero.
The relation between Huber loss, quadratic loss and $\ell_1$ penalty function has also been used by \cite{Gan07}, \cite{Ant07}, \cite{SOw11} and \cite{CDa19} in wavelet thresholding and robust statistics. 

The James-Stein estimators $\widehat{\btheta}^{\mathrm{JS,0}}$, $\widehat{\btheta}^{\mathrm{JS}}$ and $\widehat{\btheta}^{\mathrm{JS+}}$ are tailored for the Gaussian mean problem. Their strong theoretical guarantees such as \eqref{eqn-warmup-JS} are built upon analytical calculations of the $\ell_2$ risk under the Gaussianity assumption. In contrast, our estimator $\widehat{\btheta}$ is constructed from penalized MLE framework \eqref{eqn-warmup-armul}, which easily extends to general multivariate $M$-estimation problems. We want the estimator to benefit from possible similarity among tasks while still being reliable in unfavorable circumstances, see \Cref{thm-warmup-armul}. In the worst-case, the price of generality is an extra logarithm factor in the risk.

\subsubsection{Limited translation estimators}
James-Stein estimators improve over the MLE in terms of $\ell_2$ risk, which measures the average performance over parameters $\{ \theta^{\star}_j \}_{j=1}^m$. There is no guarantee on the individuals. It is well-known that the estimators underperform MLE by a large margin for $ \theta^{\star}_j $'s far from the bulk. To make matters worse, such exceptional cases also significantly reduce the overall $\ell_2$ efficacy.
\cite{EMo72} and \cite{Ste81} proposed limited translation estimators that restrict the amount of shrinkage. Hence, those estimators cannot deviate far from the MLE. By carefully setting the restrictions, they are able to control the maximum ($\ell_{\infty}$) error over all parameters. 
According to \Cref{thm-warmup-analytical}, our estimator $\widehat{\btheta}$ also has limited translation bounded by $\lambda$. \Cref{thm-warmup-armul} presents a sharp bound on the $\ell_{\infty}$ error.

\subsubsection{Soft-thresholding for sparse estimation}
When the mean vector $\btheta^{\star}$ is assumed to be sparse, it is natural to shrink many entries of the estimator to exact zero. \cite{DJH92} studied the $\ell_1$-regularized estimator
\[
\widehat{\btheta}^{ \ell_1} \in \argmin_{ \btheta  \in \RR^m } \bigg\{ \sum_{j=1}^{m} [ (\theta_j - x_j)^2 + \lambda | \theta_j | ] \bigg\} 
\]
and its minimax optimality. For each $j \in [m]$, $\widehat{\theta}^{\ell_1}_j = x_j - \min\{ \lambda, |x_j| \} \sgn(x_j)$ is a soft-thresholded version of $x_j$. If $|x_j| \leq \lambda$, then $\widehat{\theta}^{\ell_1}_j = 0$.
Soft-thresholding and $\ell_1$ regularization have wide applications in statistics, including parameter estimation subject to good risk properties at zero \citep{Bic83}, ideal spatial adaptation \citep{DJo94}, variable selection \citep{Tib96}, etc.
Our use of the $\ell_1$ penalty in \eqref{eqn-warmup-armul} is inspired by this line of research.
By \Cref{thm-warmup-analytical}, the difference $\widehat{\theta}_j - \widehat{\theta}$ between individual estimator and the global coordinator is soft-thresholded. Merging some $\widehat{\theta}_j$'s to $\widehat{\theta}$ pools the information across similar tasks.
Soft-thresholding has been used for combining the information in a small, high-quality dataset and a less costly one with a possibly different distribution, see \cite{COS15} and \cite{CZY21}.
Our formulation \eqref{eqn-warmup-armul} handles multiple datasets.

\subsubsection{Homogeneity of parameters} An extension of sparsity is homogeneity, which refers to the phenomenon that parameters in similar subgroups are close to each other. Various methods are developed to exploit such structure in high-dimensional regression, including fused lasso \citep{TSR05}, grouping pursuit \citep{SHu10} and CARDS \citep{KFW15}. Our method \eqref{eqn-warmup-armul} uses one global coordinator to utilize the homogeneity when a majority of parameters live within the same small region. In \Cref{sec-methodologies} we will incorporate more than one coordinators to deal with multiple clusters of parameters.

\subsubsection{Minimax lower bounds}
For sparse Gaussian mean estimation, \cite{DJH92} derived the minimax lower bound on the $\ell_2$ risk over $ \{ \btheta^{\star}  \in \RR^m:~ \| \btheta^{\star} \|_0 \leq \varepsilon m \}$ for $\varepsilon \in (0, 1)$, with precise constant factors. Here $\| \bx \|_0 = |\{ i:~x_i \neq 0 \}|$ is the $\ell_0$ pseudo-norm. Their parameter space is a subset of ours with $\delta = 0$. We aim to cover broader regimes but make no endeavor to optimize the constants.
In a recent work, \cite{CZL21} studied fundamental limits of multi-task and federated learning. Their definition of task relatedness is similar to ours in Assumption \ref{as-warmup-similarity} with $\varepsilon = 0$. They construct $m$ logistic models whose discrepancies are quantified by some parameter $\delta$, and derive a minimax lower bound on the estimation error of the form $  \min \{ \frac{1}{\sqrt{mn} }  +  \delta ,~ \frac{1}{\sqrt{n}} \}  $. From there they show that the optimal rate is achieved by either MLE or its pooled version. 
Our lower bound in \Cref{thm-warmup-minimax} is proved for the canonical Gaussian mean problem and allows an $\varepsilon$ fraction of the tasks to be arbitrarily different from the others. In that case, neither MLE nor pooled MLE is optimal.

\section{Methodologies}\label{sec-methodologies}

In this section, we present our framework for Adaptive and Robust MUlti-task Learning (ARMUL). We focus on three important cases and provide algorithms for their efficient implementations. 

\subsection{Adaptive and robust multi-task learning}\label{sec-methodologies-armul}

Let $m \in \ZZ_+$. For every $j \in [m]$, let $\cP_j$ be a probability distribution over a sample space $\cX_j$ and $\ell_j:~\RR^d \times \cX_j \to \RR$ be a loss function. 
Suppose that we collect $n_j$ i.i.d.~samples $\cD_j = \{ \bxi_{ji} \}_{i=1}^{n_j}$ from $\cP_j$ for every $j$, and the $m$ datasets $\{ \cD_j \}_{j=1}^m$ are independent. The $j$th learning task is to minimize the population risk $\EE_{\bxi \sim \cP_j} \ell_j ( \btheta ; \bxi )$ by estimating the population risk minimizer $\btheta^{\star}_j \in \argmin_{ \btheta  \in \RR^d } \EE_{\bxi \sim \cP_j} \ell_j ( \btheta ; \bxi )$
based on $\cD_j$. For statistical estimation in well-specified models, $\btheta^{\star}_j$ is the true parameter and $\ell_j$ can be the negative log-likelihood function.
Multi-task learning (MTL) targets all of the $m$ tasks simultaneously. The difficulty comes from the unknown task relatedness. It is often unclear whether and how a task can be better resolved by incorporating the information in other tasks.

Define the $j$th empirical loss function $f_j(\btheta) = \frac{1}{n_j} \sum_{ i = 1 }^{n_j} \ell( \btheta; \bxi_{ji} )$.
Many MTL algorithms can be formulated as constrained loss minimization problems of the form
\begin{align}
\min_{\bTheta \in \Omega} \bigg\{ \sum_{j=1}^{m} w_j f_j (\btheta_j) \bigg\},
\label{eqn-mtl-obj}
\end{align}
where $w_j$ and $\btheta_j$ are the weight and the model parameter of the $j$th task; $\bTheta = (\btheta_1,\cdots,\btheta_m) \in \RR^{d\times m}$; $\Omega \subseteq \RR^{d\times m}$ encodes the prior knowledge of task relatedness. Below are several examples.

\begin{example}[Independent task learning]\label{eg-stl}
A na\"{i}ve approach is independent task learning which minimizes the $m$ empirical loss functions separately. That is equivalent to \eqref{eqn-mtl-obj} with $\Omega = \RR^{d\times m}$.
\end{example}

\begin{example}[Data pooling]\label{eg-pooling}
In the other extreme, one may pool all the data together, solve the consensus program $\min_{\bbeta \in \RR^d }  \{ 
\sum_{j=1}^{m} w_j  f_j ( \bbeta )  \} $ and output one estimate for all tasks. We have $\Omega = \{ \bbeta \bm{1}_m^{\top} :~\bbeta \in \RR^d \}$.
\end{example}

\begin{example}[Clustered MTL]\label{eg-clustered}
The one-size-fits-all strategy above can be extended to clustered MTL, which handles multiples clusters of similar tasks. One may solve the program
\begin{align}
\min_{ 
	\substack{
		\bbeta_1,\cdots, \bbeta_K \in \RR^d \\
		z_1,\cdots,z_m \in [K]
} } \bigg\{ 
\sum_{j=1}^{m} w_j f_j ( \bbeta_{z_j}  ) 
\bigg\}
\label{eqn-obj-clustered-mtl}
\end{align}
to get the estimated labels $\{ \widehat z_j \}_{j=1}^m$ and cluster centers $\{ \widehat\bbeta_j \}_{j=1}^K$. The estimated model parameters for the $m$ tasks are $\{ \widehat{\bbeta}_{ \widehat z_j } \}_{j=1}^m$. This method corresponds to $\Omega = \{ \bB \bZ :~ \bB \in \RR^{d\times K},~ \bZ \in \{ 0, 1 \}^{K \times m},~\bZ^{\top} \bm{1}_K = \bm{1}_m  \}$.
\end{example}

\begin{example}[Low-rank MTL]\label{eg-lowrank}
By further relaxing the discrete class indicators in \eqref{eqn-obj-clustered-mtl} to continuous latent variables, one gets a formulation for low-rank MTL
\begin{align}
\min_{ 
	\bB \in \RR^{d\times K}, ~\bZ \in \RR^{K \times m}
} \bigg\{ 
\sum_{j=1}^{m} w_j  f_j ( \bB \bz_j  ) 
\bigg\}.
\label{eqn-obj-low-rank-mtl}
\end{align}
An optimal solution $(\widehat\bB , \widehat\bZ)$ yields estimated model parameters $\{ \widehat{\bB} \widehat{\bz}_j \}_{j=1}^m$ that lie in the range of $\widehat\bB$. We have $\Omega = \{ \bB \bZ :~ \bB \in \RR^{d\times K},~ \bZ \in \RR^{K\times m} \}$. 
\end{example}

\begin{example}[Hard parameter sharing]
A popular approach of MTL with neural networks is to learn a network shared by all tasks for feature extraction, plus task-specific linear functions that map features to final predictions \citep{Car97}. Thus, models of the $m$ tasks share part of their parameters. It can be viewed as \eqref{eqn-mtl-obj} with $\Omega$ of the form
\[
\bigg\{
\begin{pmatrix}
\bbeta \bm{1}_m^{\top} \\
\bGamma
\end{pmatrix}:~ \bbeta \in \RR^{d - K}, ~ \bGamma \in \RR^{K \times m}
\bigg\}.
\]
where $K$ is the number of features, $\bbeta$ consists of weight parameters of the neural network, and the columns of $\bGamma$ are parameters of task-specific linear functions. This is a combination of independent task learning and data pooling. When the neural network is replaced with a linear transform, it is equivalent to low-rank MTL.
\end{example}

We propose a framework named \emph{\underline{A}daptive and \underline{R}obust \underline{MU}lti-task \underline{L}earning}, or ARMUL for short: solve an augmented program
\begin{align}
(\widehat{\bTheta}, \widehat{\bGamma})
\in \argmin_{
\bTheta \in \RR^{d\times m},~ \bGamma \in \Omega} \bigg\{ \sum_{j=1}^{m} w_j [ f_j (\btheta_j) + \lambda_j \| \btheta_j - \bgamma_j \|_2 ] \bigg\},
\label{eqn-armul}
\end{align}
and use the columns $\{ \widehat{\btheta}_j \}_{j=1}^m$ of $\widehat{\bTheta}$ as the estimated model parameters for $m$ tasks. Here $\{ \lambda_j \}_{j=1}^{m}$ are non-negative regularization parameters. 
Setting all of $\lambda_j$'s to zero or infinity result in independent task learning or the constrained program \eqref{eqn-mtl-obj}, respectively. 
The framework \eqref{eqn-armul} is a relaxation of \eqref{eqn-mtl-obj} so that the estimated models better fit their associated data.
The method \eqref{eqn-warmup-armul} for multi-task mean estimation is a special case, with $d = 1$, $\Omega = \{ \theta \bm{1}_m^{\top} :~\theta \in \RR \}$ and $f_j$ being the square loss.

\begin{remark}[Relaxation]
One could also consider the following relaxation of \eqref{eqn-mtl-obj}:
\begin{align}
\min_{\bTheta \in \Omega_r } \bigg\{ \sum_{j=1}^{m} w_j f_j (\btheta_j) \bigg\}
\label{eqn-mtl-obj-relaxation},
\end{align}
where
\[
\Omega_r = \bigg\{ \bTheta \in \RR^{d \times m}:~ \exists \bGamma \in \Omega \text{ s.t. } \sum_{j=1}^{m} w_j \lambda_j \| \btheta_j - \bgamma_j \|_2 \leq r  \bigg\}.
\]
The programs \eqref{eqn-mtl-obj-relaxation} and \eqref{eqn-mtl-obj} share the same form.
Choosing a positive $r$ helps deal with possible misspecification of the space $\Omega$ for the true parameter $ \bTheta^{\star} = ( \btheta_1^{\star}  , \cdots , \btheta_m^{\star}  )$.
Also, there exists some $r \geq 0$ such that the constrained program \eqref{eqn-mtl-obj-relaxation} is equivalent to the penalized program \eqref{eqn-armul}. Selecting $r$ and $\{ \lambda_j \}_{j=1}^m$ for \eqref{eqn-mtl-obj-relaxation} can be difficult when the amount of misspecification is unknown. 
On the other hand, our theory shows that \eqref{eqn-armul} enjoys strong guarantees while being agnostic to the misspecification.
\end{remark}

We see from \eqref{eqn-armul} that $\widehat{\bGamma}$ solves a constrained problem $\min_{\bGamma \in \Omega} 
\big\{ 
\sum_{j=1}^{m} w_j \widetilde{f}_j(\bgamma_j)
\big\}$
similar to \eqref{eqn-mtl-obj}, where $\widetilde{f}_j(\bgamma) = \min_{\bxi \in \RR^d} \{ f_j(\bxi) + \lambda_j \| \bgamma - \bxi  \|_2 \}$ is the infimal convolution of the loss function $f_j(\cdot)$ and the $\ell_2$ penalty $\lambda_j \| \cdot \|_2$. Since the latter is $\lambda_j$-Lipschitz, as long as $f_j$ is convex, the infimal convolution $\widetilde{f}_j$ is always convex and $\lambda_j$-Lipschitz (Lemma F.4 in the supplementary material) just like the Huber loss function in \Cref{thm-warmup-analytical}. This makes our method robust against a small fraction of tasks which are dissimilar to others or even contain contaminated data.
Meanwhile, the fact
\begin{align}
&\widehat{\btheta}_j \in \argmin\nolimits_{\btheta \in \RR^d} \{ f_j (\btheta) + \lambda_j \| \btheta - \widehat\bgamma_j \|_2 \}, ~~\forall j \in [m] 
\label{eqn-armul-1}
\end{align}
shows that $\widehat{\bTheta}$ is shrunk toward $\widehat{\bGamma} \in \Omega$.
When the set $\Omega$ accurately reflects the relations among $m$ underlying models and $\lambda_j$ is not too small, the cusp of the $\ell_2$ norm penalty at zero forces $\widehat\bTheta = \widehat\bGamma \in \Omega$. 
When $\lambda_j$ is not too large and $f_j$ is strongly convex near its minimizer, the Lipschitz smoothness of the $\ell_2$ penalty ensures the closeness between $\widehat{\btheta}_j$ and $\argmin_{\btheta \in \RR^d} f_j(\btheta)$. Hence, the new method will at least be comparable to independent task learning. 
In \Cref{sec-theory} and Appendix D we will conduct a formal analysis of the adaptivity and robustness. 
The theory suggests choosing $w_j = n_j$ and $\lambda_j \asymp \sqrt{ \frac{d + \log m}{n_j} }$ to achieve the goal.


\subsection{Implementations}\label{sec-implementations}

For efficient implementation of ARMUL, we define $\bV = \bTheta - \bGamma$ and transform the program \eqref{eqn-armul} to a more convenient form
\begin{align}
\min_{
	\bV \in \RR^{d\times m},~ \bGamma \in \Omega} \bigg\{ \sum_{j=1}^{m} w_j [ f_j ( \bgamma_j + \bv_j) + \lambda_j \| \bv_j \|_2 ] \bigg\}.
\label{eqn-armul-0}
\end{align}
We will optimize the two blocks of variables $\bV$ and $\bGamma$ in an alternating manner. Assume that $\{ f_j \}_{j=1}^m$ are differentiable. 
If $\bGamma$ is fixed, \eqref{eqn-armul-0} decomposes into $m$ independent programs
\begin{align}
\min\nolimits_{ \bv_j \in \RR^{d} } \{   f_j ( \bgamma_j + \bv_j) + \lambda_j \| \bv_j \|_2 \} , \qquad j \in [m].
\label{eqn-armul-v}
\end{align}
A natural algorithm for handling non-smooth convex regularizers such as $\| \cdot \|_2$ is proximal gradient descent \citep{PBo14}. The iteration for solving \eqref{eqn-armul-v} is
\begin{align}
\bv_j^{t+1} = \mathrm{prox}_{\eta \lambda_j} \Big(
\bv_j^t - \eta \nabla f_j (\bgamma_j + \bv_j^t)
\Big),\qquad t = 0, 1,\cdots,
\end{align}
where $\eta$ is the step-size and we define $ \mathrm{prox}_c (\bx) = (1 - \frac{c}{\| \bx \|_2})_+ \bx$.
If $\bV$ is fixed, \eqref{eqn-armul-0} reduces to a constrained program 
\begin{align}
\min_{\bGamma \in \Omega} \bigg\{ \sum_{j=1}^{m} w_j  f_j ( \bgamma_j + \bv_j)  \bigg\}
\label{eqn-armul-gamma}
\end{align}
of the form \eqref{eqn-mtl-obj} with shifted loss functions. We will choose algorithms according to $\Omega$. The whole procedure above is summarized in Algorithm \ref{alg-armul}. For simplicity, we only perform a single iteration of proximal gradient descent. Numerical experiments show that this already gives satisfactory results.

\begin{algorithm}[t]
	{\bf Input:} loss functions $\{ f_j \}_{j=1}^m$, weights $\{ w_j \}_{j=1}^m$, penalty parameters $\{ \lambda_j \}_{j=1}^m$, step-size $\eta_v$, number of iterations $T$, initial guesses $\bV^{0} \in \RR^{d\times m}$ and $\bGamma^{0} \in \Omega$.\\
	{\bf For $t = 0,1,\ldots, T - 1$}\\
	\hspace*{0.5cm} Compute $\bV^{t+1}$ by
	\begin{align*}
	\bv_j^{t+1} =
	\bigg(
	1 - \frac{\eta_v \lambda_j}{ \| \bv_j^{t} - \eta_v \nabla f_j ( \bgamma_j + \bv_j^{t} ) \|_2 }
	\bigg)_+ \Big( \bv_j^{t} - \eta_v \nabla f_j (\bgamma_j^t + \bv_j^{t} ) \Big) , \qquad j \in [m].
	\end{align*}
	\hspace*{0.5cm} Compute $\bGamma^{t+1}$.\\
	{\bf Return:} $\widehat{\bTheta} = \bGamma^{T} + \bV^T$.
	\caption{Adaptive and robust multi-task learning (ARMUL)}
	\label{alg-armul}
\end{algorithm}

Having introduced the general procedure, we now focus on three important cases of ARMUL \eqref{eqn-armul} and derive the updating rules for their $\bGamma$'s. 
Their Python implementations are available at \texttt{https://github.com/kw2934/ARMUL/}.

\begin{enumerate}
\item Vanilla ARMUL: $\Omega = \{ \bbeta \bm{1}_m^{\top} :~\bbeta \in \RR^d \}$. The original program \eqref{eqn-armul} is equivalent to
\begin{align}
\min_{
\bTheta \in \RR^{d\times m},~ \bbeta \in \RR^d} \bigg\{ \sum_{j=1}^{m} w_j [ f_j (\btheta_j) + \lambda_j \| \btheta_j - \bbeta \|_2 ] \bigg\}.
\label{eqn-armul-vanilla-0}
\end{align}
It is jointly convex in $(\bTheta, \bbeta)$ as long as $\{ f_j \}_{j=1}^m$ are convex functions. The intermediate program \eqref{eqn-armul-gamma} is equivalent to an unconstrained one
\begin{align*}
\min_{\bbeta \in \RR^d} \bigg\{ \sum_{j=1}^{m} w_j  f_j ( \bbeta + \bv_j)  \bigg\}.
\end{align*}
We can update $\bbeta$ by gradient descent.

\item Clustered ARMUL: $\Omega = \{ \bB \bZ :~ \bB \in \RR^{d\times K},~ \bZ \in \{ 0, 1 \}^{K \times m},~\bZ^{\top} \bm{1}_K = \bm{1}_m  \}$. The original program \eqref{eqn-armul} is equivalent to
\begin{align}
\min_{
	\bTheta \in \RR^{d\times m},~ \bB \in \RR^{d\times K}, ~\bz \in [K]^{m} 
} \bigg\{ \sum_{j=1}^{m} w_j [ f_j (\btheta_j) + \lambda_j \| \btheta_j - \bbeta_{z_j} \|_2 ] \bigg\}.
\label{eqn-armul-clustered-0}
\end{align}
The intermediate program \eqref{eqn-armul-gamma} is equivalent to
\begin{align*}
\min_{ \bB \in \RR^{d\times K}, ~\bz \in [K]^{m}
} \bigg\{ 
\sum_{j=1}^{m} w_j f_j ( \bbeta_{z_j} + \bv_j )
\bigg\} .
\end{align*}
When $\bz$ is fixed, we update $\bB$ by gradient descent; when $\bB$ is fixed, we update $\bz$ with its optimal value
\[
\Big( \argmin_{z \in [K]} f_1 ( \bbeta_{z} + \bv_1 ), \cdots, \argmin_{z \in [K]} f_m ( \bbeta_{z} + \bv_m ) \Big).
\]
We can repeat the above steps multiple times.

\item Low-rank ARMUL: $\Omega = \{ \bB \bZ :~ \bB \in \RR^{d\times K},~ \bZ \in \RR^{K\times m} \}$. The original program \eqref{eqn-armul} is equivalent to
\begin{align}
\min_{
	\bTheta \in \RR^{d\times m},~ \bB \in \RR^{d\times K}, ~\bZ \in \RR^{ K \times m} 
} \bigg\{ \sum_{j=1}^{m} w_j [ f_j (\btheta_j) + \lambda_j \| \btheta_j -  \bB\bz_j \|_2 ] \bigg\}.
\label{eqn-armul-lowrank}
\end{align}
The intermediate program \eqref{eqn-armul-gamma} is equivalent to
\begin{align*}
\min_{ 
\bB \in \RR^{d\times K}, ~\bZ \in \RR^{K \times m}
} \bigg\{ 
\sum_{j=1}^{m} w_j f_j ( \bB \bz_{j} + \bv_j )  
\bigg\}.
\end{align*}
When $\bB$ or $\bZ$ is fixed, we update the other by gradient descent. Again, the procedure can be repeated.
\end{enumerate}

Algorithm \ref{alg-armul} returns the estimated model parameters $\{ \widehat{\btheta}_j \}_{j=1}^m \subseteq \RR^d$ for $m$ tasks. As a by-product, vanilla ARMUL yields a center $\widehat{\bbeta} \in \RR^d$; clustered ARMUL yields $K$ centers $\{ \widehat{\bbeta}_k \}_{k=1}^K \subseteq \RR^d$ together with $m$ cluster labels $\{ \widehat{z}_j \}_{j=1}^m \subseteq [K]$; low-rank ARMUL yields a $K$-dimensional subspace $\Range(\widehat{\bB}) \subseteq \RR^d$ and $m$ coefficient vectors $\{ \widehat{\bz}_j \}_{j=1}^m \subseteq \RR^K$. These quantities reveal intrinsic structures of the task population: the model parameters concentrate around one point, multiple points or a low-dimensional linear subspace. 
Such knowledge is valuable for dealing with new tasks of similar types.

\section{Theoretical analysis}\label{sec-theory}

In this section, we conduct a non-asymptotic analysis of vanilla, clustered and low-rank ARMUL algorithms. 
Our theoretical investigation shows that the proposed estimators automatically adapt to the unknown task relatedness. 
The study under statistical settings is built upon the deterministic results in Appendix A, which could be of independent interest.

\subsection{Problem setup}

Recall the setup in \Cref{sec-methodologies-armul} where $\{ \cP_j \}_{j=1}^m$ are probability distributions over sample spaces $\{ \cX_j \}_{j=1}^m$ and $\{ \ell_j \}_{j=1}^m$ are loss functions.
We draw $m$ independent datasets $\{ \cD_j \}_{j=1}^m$, where $\cD_j = \{ \bxi_{ji} \}_{i=1}^{n_j}$ are i.i.d.~from $\cP_j$. For each $j$, define the population loss function and its minimizer
\[
F_j (\btheta) = \EE_{\bxi \sim \cP_j} \ell_j ( \btheta , \bxi )
\qquad\text{and}\qquad
\btheta^{\star}_j \in \argmin\nolimits_{ \btheta  \in \RR^d } F_j ( \btheta ) .
\]
Define the $j$-th empirical loss function $f_j(\btheta) = \frac{1}{n_j} \sum_{ i = 1 }^{n_j} \ell_j ( \btheta, \bxi_{ji} )$.
To facilitate illustration, throughout this section we focus on the case where $n_1= \cdots = n_m = n$. 
We estimate $\{ \btheta^{\star}_j \}_{j=1}^m$ by the solutions $\{ \widehat{\btheta}_j \}_{j=1}^m$ computed from the program \eqref{eqn-armul} with $\lambda_1 = \cdots = \lambda_m = \lambda$ and $w_1 = \cdots = w_m = 1$. 
We defer discussions on general sample sizes $\{ n_j \}_{j=1}^m$ to Appendix D.

To analyze the estimation error, we make the following standard assumptions.

\begin{assumption}[Regularity]\label{as-stat-1}
For any $j \in [m]$ and $\bxi \in \cX_j$, $\ell_j (\cdot,\bxi):~\RR^d \to \RR$ is convex and twice differentiable. Also, there exist absolute constants $c_1,c_2 > 0$ and $c_1 < \rho, L, M <c_2$ such that $\rho \bI \preceq \nabla^2 F_j (\btheta) \preceq L \bI$ holds for all $ \btheta \in B(\btheta^{\star}_j, M)$ and $j \in [m]$.
\end{assumption}

\begin{assumption}[Concentration]\label{as-stat-2}
There exist $0 \leq \sigma, \tau, p < c$ for an absolute constant $c$ such that for any $j \in [m]$, we have
\begin{align*}
& \| \nabla \ell_j ( \btheta^{\star}_j , \bxi_{j1}  ) \|_{\psi_2} \leq \sigma , \\
& \Big\| \Big\langle \Big( \nabla^2 \ell_j ( \btheta , \bxi_{j1}  ) - \EE  [ \nabla^2 \ell_j ( \btheta , \bxi_{j1}   ) ]
\Big)
\bv , \bv \Big\rangle \Big\|_{\psi_1} \leq \tau^2, \qquad \forall \btheta \in B(\btheta^{\star}_j, M),~ \bv \in \SSS^{d-1} , \\
& 
\EE Q_j(\bxi_{j1}) \leq \tau^3 d^{p} ,
\end{align*}	
where we define
\[
Q_j (\bxi) = \sup_{ \substack{ \btheta_1, \btheta_2 \in B( \btheta^{\star}_j, M ) \\ \btheta_1 \neq \btheta_2 } } \frac{ \| \nabla^2\ell_j ( \btheta_2 , \bxi  ) - \nabla^2\ell_j ( \btheta_1 , \bxi  ) \|_2 }{ \| \btheta_2 - \btheta_1 \|_2 }, \qquad\forall \bxi \in \cX_j.
\]
The gradients of $\ell_j$ are taken with respect to its first argument.
\end{assumption}

The regularity assumption requires the Hessian of the population loss function $F_j$ to be bounded from below and above near its minimizer $\btheta^{\star}_j$. The concentration assumption implies light tails and smoothness of the empirical gradient and Hessian. They are commonly used in statistical machine learning, see \cite{MBM18} and the references therein. Below we present several examples for illustration. 

\begin{example}[Gaussian mean estimation]\label{eg-gaussian}
Let $\cX_j = \RR^d$, $\cP_j = N(\btheta_j^{\star} , \bI_d )$ and $\ell_j (\btheta,  \bxi  ) = \| \bxi - \btheta \|_2^2$. Then, $\nabla \ell_j (\btheta  , \bxi ) = 2 (\btheta - \bxi)$ and $\nabla^2 \ell_j (\btheta,  \bxi  ) = 2 \bI_d $. Assumptions \ref{as-stat-1} and \ref{as-stat-2} clearly hold.
\end{example}

\begin{example}[Linear regression]
Let $\bxi_{ji} = ( \bx_{ji} , y_{ji} ) \in \RR^d \times \RR$, where $ \bx_{ji}  $ is the covariate vector and $y_{ji}$ is the response.
Consider the square loss $\ell_j (\btheta,  (\bx, y) ) = (y - \bx^{\top} \btheta)^2$ and let $\varepsilon_{ji} = y_{ji} - \bx_{ji}^{\top} \btheta^{\star}_j$ be the residual of the best linear prediction. Then, $\nabla \ell_j (\btheta  ,  (\bx , y ) ) = 2 \bx  ( \bx^{\top} \btheta  - y  ) $ and $\nabla^2 \ell_j (\btheta,  (\bx, y) ) = 2 \bx  \bx^{\top} $. Assumption \ref{as-stat-1} holds when the eigenvalues of $ \EE_{ (\bx, y) \sim \cP_j} (\bx \bx^{\top})$ are bounded from above and below.
Note that $\nabla \ell_j (\btheta^{\star}_j  ,  (\bx_{ji} , y_{ji} ) ) = - 2 \bx_{ji} \varepsilon_{ji}$.
If $\bx_{ji}$ is sub-Gaussian and $\varepsilon_{ji}$ is bounded, then Assumption \ref{as-stat-2} holds. 
It is worth pointing out that most of our results continue to hold up to logarithmic factors when $\varepsilon_{ji}$ is unbounded but light-tailed.
\end{example}

\begin{example}[Logistic regression]
Let $\bxi_{ji} = ( \bx_{ji} , y_{ji} ) \in \RR^d \times \{ 0, 1 \} $, where $ \bx_{ji}  $ is the covariate vector and $y_{ji}$ is the binary label.
Define the logistic loss $\ell_j (\btheta,  (\bx, y) ) = b( \bx^{\top} \btheta ) - y \bx^{\top} \btheta $ where $b(t) = \log (  1 + e^{t} )$. We have $\nabla \ell_j  (\btheta,  (\bx, y) ) = \bx [ b'(\bx^{\top} \btheta) - y ]$, $\nabla^2 \ell_j  (\btheta,  (\bx, y) ) = b''(\bx^{\top} \btheta)  \bx \bx^{\top}$, $b'(t) = 1/ (1 + e^{-t}) \in [0, 1]$ and $b''(t) = e^t / (1 + e^{t})^2 = 1/(2+e^t + e^{-t}) \in (0, 1/4]$. Hence $0 \prec \nabla^2 F_j (\btheta) \preceq (1/4) \bI$ for all $\btheta$, and Assumption \ref{as-stat-1} easily holds for bounded $\| \btheta^{\star}_j \|_2 $ and $M$.
When $\bx_{ji}$ is sub-Gaussian, so is $\nabla \ell_j  (\btheta,  (\bx_{ji}, y_{ji}) ) $; for any $\btheta \in \RR^d$ and $\bv \in \SSS^{d-1}$, $\langle \bv , \nabla^2 \ell_j (\btheta, (\bx_{ji}, y_{ji}) ) \bv \rangle = b''(\bx_{ji}^{\top} \btheta) (  \bx_{ji}^{\top} \bv )^2$ is sub-exponential. From $\sup_{t \in \RR} | b'''(t) | < \infty$ and
\begin{align*}
\| \nabla^2 \ell_j (\btheta_2, (\bx, y) ) - \nabla^2 \ell_j (\btheta_1, (\bx, y) ) \|_2 = | b''(\bx^{\top} \btheta_2) - b''(\bx^{\top} \btheta_1) | \cdot \| \bx \|_2^2 \lesssim   \| \btheta_2 - \btheta_1 \|_2\| \bx \|_2^3
\end{align*}
we obtain that $Q_j(\bx, y) \leq \| \bx \|_2^3$. 
According to Remark 2.3 in \cite{HKZ12}, if $\| \bx_{ji} \|_{\psi_2} \lesssim 1$, then $\EE Q_j(\bx_{ji}, y_{ji}) \leq \EE \| \bx_{ji} \|_2^3 \lesssim d^{3/2}$. Based on the above, Assumption \ref{as-stat-2} holds.
\end{example}

\subsection{Personalization}\label{sec-personalization}

Independent task learning estimates each $\btheta^{\star}_j$ by the minimizer $\widetilde{\btheta}_j = \argmin_{ \btheta  \in \RR^d } f_j(\btheta)$ of its associated empirical loss, without referring to other tasks. ARMUL \eqref{eqn-armul} with $\lambda_1 = \cdots = \lambda_m = \lambda$ and $w_1 = \cdots = w_m = 1$ also yields one personalized model for each task. Below we show their closeness and provide a way of choosing $\lambda$ so that ARMUL is at least comparable to independent task learning.

For any constraint set $\Omega \subseteq \RR^{d\times m}$, the output $\widehat{\btheta}_j$ of ARMUL \eqref{eqn-armul} always satisfies \eqref{eqn-armul-1}. Therefore, $\widehat{\btheta}_j$ and $\widetilde{\btheta}_j $ minimize similar functions. 
The penalty term $\lambda \| \btheta - \widehat{\bgamma}_j \|_2$ in \eqref{eqn-armul-1} can be viewed as a perturbation added to the objective function $f_j$. According the following theorem, it can only perturb the minimizer by a limited amount. See Appendix D.1 for stronger results for general $\{ n_j \}_{j=1}^m$ and their proof.

\begin{theorem}[Personalization]\label{thm-personalization-00}
Let Assumptions \ref{as-stat-1} and \ref{as-stat-2} hold. There exist constants $C$, $C_1$ and $C_2$ such that under the conditions $ \lambda < \rho M / 4$, $n  > C_1 d ( \log n  )(  \log m )$ and $0 \leq t <  C_2 n  / ( d \log n )$, the following holds with probability at least $1 -  e^{-t}$:
\begin{align*}
& \| \widetilde{\btheta}_j - \btheta^{\star}_j \|_2 \leq
C \sigma \sqrt{ \frac{ d + \log m + t}{n} } 
\qquad\text{and}\qquad
 \| \widehat{\btheta}_j - \widetilde{\btheta}_j \|_2 \leq \frac{ 2 \lambda }{\rho} , \qquad\forall j \in [m].
\end{align*}
\end{theorem}

The distance between the estimates $\widehat{\btheta}_j$ and $\widetilde{\btheta}_j $ returned by ARMUL and independent task learning is bounded using the penalty level $\lambda$ and the strong convexity parameter $\rho$. 
Intuitively, when the empirical loss function $f_j$ is strongly convex in a neighborhood of its minimizer $\widetilde{\btheta}_j $, the Lipschitz penalty function does make much difference. The unsquared $\ell_2$ penalty is crucial.
In \Cref{thm-warmup-analytical} we showed this phenomenon for mean estimation in one dimension, where the $\ell_2$ penalty becomes the absolute value. \Cref{thm-personalization-00} guarantees the fidelity of ARMUL outputs to their associated datasets for general $M$-estimation. 

By Assumptions \ref{as-stat-1} and \ref{as-stat-2}, we have $\sigma, \rho^{-1} \lesssim 1$. \Cref{thm-personalization-00} implies that when $\lambda \lesssim   \sqrt{ \frac{d + \log m}{n} }$, the bound $\| \widehat{\btheta}_j - \btheta_j^{\star} \|_2 \lesssim
  \sqrt{ \frac{ d + \log m }{n} } $ simultaneously holds for all $j \in [m]$ with high probability. In that case, the ARMUL achieves the same parametric error rate $O(   \sqrt{ \frac{d + \log m}{n} } )$ of independent task learning. The $\log m$ term results from the simultaneous control over $m$ tasks.

The above results on personalization hold for general ARMUL with arbitrary constraint set $\Omega$.
In the subsections to follow, we will investigate three important cases of ARMUL (vanilla, clustered and low-rank) to study the adaptivity and robustness.

\subsection{Vanilla ARMUL}\label{sec-vanilla}

In this subsection, we analyze the vanilla ARMUL estimators $\{ \widehat\btheta_j \}_{j=1}^m$ returned by
\begin{align}
	( \widehat{\bTheta}, \widehat{\bbeta} ) \in
	\argmin_{
		\bTheta \in \RR^{d\times m},~ \bbeta \in \RR^d} \bigg\{ \sum_{j=1}^{m} [  f_j (\btheta_j) + \lambda \| \btheta_j - \bbeta \|_2 ] \bigg\}.
	\label{eqn-armul-vanilla-00}
\end{align}
We introduce an assumption on task relatedness. It is a multivariate extension of Assumption \ref{as-warmup-similarity}.

\begin{assumption}[Task relatedness]\label{as-armul-vanilla-00}
	For any $\varepsilon \in [0, 1]$ and $\delta \geq 0$, define
\begin{align*}
	\Omega (\varepsilon, \delta ) = \Big\{ \bTheta  \in \RR^{d \times m}:~\min_{\theta \in \RR^d}\max_{j \in S} | \btheta_j - \btheta  | \leq \delta \text{ and }
	|S^c| / m \leq \varepsilon 
	\text{ for some } S \subseteq [m]
	\Big\}.
\end{align*}
Assume that $\bTheta^{\star} \in \Omega (\varepsilon, \delta ) $ holds for some $\varepsilon, \delta \geq 0$. Let $S $ be a subset of $ [m]$ that satisfies the requirements in the definition.
\end{assumption}

When Assumption \ref{as-armul-vanilla-00} holds, we say the $m$ tasks are $(\varepsilon,\delta)$-related. It is worth pointing out that any $m$ tasks are $(0, \max_{j\in[m]}   \| \btheta^{\star}_j \|_2  )$-related. 
Smaller $\varepsilon$ and $\delta$ imply stronger similarity among the tasks.
The theorem below presents upper bounds on estimation errors of vanilla ARMUL \eqref{eqn-armul-vanilla-00}. 
See Appendix D.2 for stronger results for general $\{ n_j \}_{j=1}^m$ and their proof.

\begin{theorem}[Vanilla ARMUL]\label{thm-vanilla-00}
	Let Assumptions \ref{as-stat-1}, \ref{as-stat-2} and \ref{as-armul-vanilla-00} hold. 
	There exist positive constants $\{ C_i \}_{i=0}^5$ such that under the conditions $n  > C_1 d ( \log n  )(  \log m )$, $0 \leq t <  C_2 n  / ( d \log n )$, $C_3   \sigma  \sqrt{ \frac{ d + \log m + t}{ n }  } < \lambda < C_4 \sigma  $ and $0 \leq \varepsilon < C_5$, the following bounds hold with probability at least $1 -  e^{-t}$:
	\begin{align*}
		& \max_{j \in S} \| \widehat{\btheta}_j  - \btheta^{\star}_j \|_2  
		\leq
		C_0 \bigg(  \sigma \sqrt{ \frac{ d  + t}{ mn } }
		+       \min \{    \delta ,   \lambda   \}  +  \varepsilon\lambda  
		\bigg) , \\
		& \max_{j \in S^c } \| \widehat{\btheta}_j  - \btheta^{\star}_j \|_2  
		\leq    C_0 \lambda , \\
		& \frac{1}{m} \sum_{j=1}^{m} [ F_j ( \widehat{\btheta}_j ) - F_j ( \btheta^{\star}_j ) ]
\leq 		 \frac{L}{m} \sum_{j=1}^{m} \| \widehat{\btheta}_j  - \btheta^{\star}_j \|_2^2  \leq C_0 L \bigg(
		 \sigma^2  \frac{ d  + t}{ mn } 
		+       \min \{    \delta^2 ,   \lambda^2  \}  +   \varepsilon  \lambda^2
		\bigg).
	\end{align*}
	Moreover, there exists a constant $C_6$ such that under the conditions $\varepsilon = 0$ and $    \delta < C_6 \sigma \sqrt{ \frac{d + \log m}{ n } }$, we have $\widehat{\btheta}_1 = \cdots = \widehat{\btheta}_m = \argmin_{\btheta \in \RR^d} \{ \sum_{j = 1}^m   f_j (\btheta) \}$ with probability at least $1 - e^{-t}$.
\end{theorem}

\Cref{thm-vanilla-00} simultaneously controls the estimation errors for all individual tasks. This implies the bounds on the MSE  $\frac{1}{m} \sum_{j=1}^{m} \| \widehat{\btheta}_j  - \btheta^{\star}_j \|_2^2$ and the average excess risk $ \frac{1}{m} \sum_{j=1}^{m} [ F_j ( \widehat{\btheta}_j ) - F_j ( \btheta^{\star}_j ) ]$. The results suggest choosing $\lambda = C \sqrt{ \frac{d + \log m}{n} } $ for some constant $C$.  In practice, $C$ can be selected by cross-validation to optimize the performance. When $\lambda \asymp \sqrt{ \frac{d + \log m}{n} } $, the MSE bound reads
\begin{align}
\frac{1}{m} \sum_{j=1}^{m} \| \widehat{\btheta}_j  - \btheta^{\star}_j \|_2^2   \lesssim  \frac{ d }{ mn } 
+       \min \bigg\{    \delta^2 ,    \frac{d }{n}   \bigg\}  +  \frac{ \varepsilon d }{n}  ,
\label{eqn-vanilla-mse}
\end{align}
where $\lesssim$ hides logarithmic factors.

For any $\varepsilon$ and $\delta$, a simple bound $\| \widehat{\btheta}_j  - \btheta^{\star}_j \|_2  
\lesssim   \lambda   $ always holds for all $j \in [m]$, which echos \Cref{thm-personalization-00}. \Cref{thm-vanilla-00} implies more refined results. 
\begin{itemize}
\item (Reduction to data pooling) When $\varepsilon = \delta = 0$, all target parameters are the same. 
The parameter space becomes $\Omega (0, 0) = \{ \bbeta \bm{1}_m^{\top}:~ \bbeta \in \RR^d \}$. 
Data pooling is a natural approach, whose MSE is $O(   d / mn  )$. According to \eqref{eqn-vanilla-mse}, the vanilla ARMUL has the same error rate.
In fact, it coincides with data pooling with high probability, thanks to the cusp of the unsquared $\ell_2$ penalty at zero. 

\item (Adaptivity) The relatedness parameters $\varepsilon$ and $\delta$ quantify the amount of model misspecification incurred in data pooling. 
As $\varepsilon$ and $\delta$ increase, the MSE upper bound \eqref{eqn-vanilla-mse} smoothly transits from that for data pooling to that for independent task learning. We will see in \Cref{thm-minimax-00} below that for every $( \varepsilon, \delta )$, the error bound is minimax optimal over $\Omega ( \varepsilon , \delta )$. Therefore, vanilla ARMUL automatically adapts to the unknown relatedness $(\varepsilon, \delta)$ of the tasks.
Meanwhile, we need an estimate on the noise level $\sigma$. Since $\sigma$ is determined by individual tasks rather than their relatedness, it is easy to estimate using traditional independent task learning methods. We also note that knowledge about the noise level is commonly assumed in adaptive statistical estimation, including adaptation to smoothness in nonparametric regression \citep{Lep91} and adaptation to sparsity in high-dimensional estimation \citep{DJH92}.

\item (Robustness) Vanilla ARMUL only pays a limited price $\frac{\varepsilon d }{ n } $ for the outlier tasks with unknown index set $S^c$ and arbitrary difference from the others.
For the Gaussian mean problem (\Cref{eg-gaussian}) with $\delta = 0$, our bounds on $ \max_{j \in S} \| \widehat{\btheta}_j  - \btheta^{\star}_j \|_2  $ and $ \frac{1}{m} \sum_{j=1}^{m} \| \widehat{\btheta}_j  - \btheta^{\star}_j \|_2^2  $ recover Theorem 6 in \cite{CDa19}.
In addition, we can allow the datasets $\{ \cD_j \}_{S^c}$ to be arbitrarily contaminated, in which case the bound on $\max_{j \in S } \| \widehat{\btheta}_j - \btheta^{\star}_j \|_2 $ in \Cref{thm-vanilla-00} continues to hold. 
\end{itemize}

To close this subsection, we use multi-task Gaussian mean estimation to get minimax lower bounds on the MSE. The proof can be found in Appendix E.1.

\begin{theorem}[Minimax lower bound]\label{thm-minimax-00}
Consider the setup in \Cref{eg-gaussian} and let Assumption \ref{as-armul-vanilla-00} hold.
There exist universal constants $C,c>0$ such that for any $(\varepsilon , \delta)$,
\begin{align*}
\inf_{\widehat{\bTheta} } \sup_{  \bTheta^{\star} \in \Omega ( \varepsilon , \delta ) } 
\PP_{\bTheta^{\star}}
\bigg[
\frac{1}{m} \sum_{j=1}^{m} \| \widehat{\btheta}_j  - \btheta^{\star}_j \|_2^2
\geq C \bigg(
\frac{d}{mn} +
\min \bigg\{
\delta^2 , \frac{d}{n}
\bigg\} + \frac{\varepsilon d}{n}
\bigg)
\bigg] 
\geq c.
\end{align*}
\end{theorem}

\subsection{Clustered ARMUL}

In this subsection, we study clustered ARMUL 
\begin{align}
( \widehat{\bTheta}, \widehat{\bB}, \widehat{\bz} )
\in \argmin_{
	\bTheta \in \RR^{d\times m},~ \bB \in \RR^{d\times K}, ~\bz \in [K]^{m} 
} \bigg\{ \sum_{j=1}^{m} [ f_j (\btheta_j) + \lambda  \| \btheta_j - \bbeta_{z_j} \|_2 ] \bigg\}.
\label{eqn-armul-clustered-1}
\end{align}
Here $K \geq 2$ is the target number of clusters. 
Clustered multi-task learning works the best when $\{ \btheta^{\star}_j \}_{j=1}^m$ concentrate around $K$ well-separated centers. Yet, such regularity conditions are difficult to verify and may not hold in practice. We introduce a relaxed version of that as our technical assumption.

\begin{assumption}[Task relatedness]\label{as-armul-clustered}
There exist $\varepsilon ,\delta \geq 0$, $K \geq 2$, $\{ \bbeta^{\star}_k \}_{k=1}^K \subseteq \RR^d$, $\{ z_j^{\star} \}_{j=1}^m \subseteq [K]$, $S \subseteq [m]$ and absolute constants $c_1, c_2 > 0$ such that the followings hold:
\begin{itemize}
\item (Similarity) $\max_{j \in S } \| \btheta^{\star}_j - \bbeta^{\star}_{z_j^{\star}} \|_2 \leq \delta$ and $|S^c| \leq \varepsilon m$;
\item (Separation) $\min_{k \neq \ell} \| \bbeta^{\star}_{k} - \bbeta^{\star}_{\ell} \|_2 \geq c_1$;
\item (Balancedness) $\min_{k \in [K]} | \{ j \in [m] :~ z_j^{\star} = k \} | \geq c_2 m / K$.
\end{itemize}
\end{assumption}

When $\varepsilon = \delta = 0$, the target parameters $\{ \btheta^{\star}_j \}_{j=1}^m$ consist of only $K$ distinct points $\{ \bbeta^{\star}_k \}_{k=1}^K$ with constant separations. Also, there is no vanishingly small cluster. Assumption \ref{as-armul-clustered} allow for any possible tasks as long as we use large enough $\delta $. For instance, we can take $\bbeta^{\star}_k = k \be_1$ for all $k$, $z^{\star}_j = ( j \mod K ) + 1$ for all $j$, $\varepsilon = 0$ and $\delta =  K + \max_{j \in [m] } \| \btheta^{\star}_j \|_2 $ to make Assumption \ref{as-armul-clustered} hold.

The theorem below presents upper bounds on estimation errors of clustered ARMUL \eqref{eqn-armul-clustered-1} when $\delta = 0$, whose proof is in Appendix E.2.

\begin{theorem}[Clustered ARMUL]\label{thm-clustered}
Let Assumptions \ref{as-stat-1}, \ref{as-stat-2} and \ref{as-armul-clustered} hold with $\varepsilon = 0$. 
There exist positive constants $\{ C_i \}_{i=0}^5$ such that under the conditions $n  > C_1 K d ( \log n  )(  \log m )$, $0 \leq t <  C_2 n  / ( d \log n )$, $C_3 K \sigma \sqrt{ \frac{d + \log m + t}{n} }< \lambda < C_4 \sigma  $ and $0 \leq \varepsilon < C_5 / K^2$, the following bound holds for the estimator $\widehat{\bTheta}$ in \eqref{eqn-armul-clustered-1} with probability at least $1 -  e^{-t}$:
\begin{align*}
& 
 \frac{1}{m} \sum_{j=1}^{m} [ F_j ( \widehat{\btheta}_j ) - F_j ( \btheta^{\star}_j ) ]
\leq 
L \cdot \max_{j \in [m] } \| \widehat{\btheta}_j  - \btheta^{\star}_j \|_2^2
\leq
C_0 L \bigg(  \frac{ \sigma^2 K( d  + t ) }{mn} 
+     \min \{ K^2  \delta^2 ,   \lambda^2   \}    
\bigg) .
\end{align*}
In addition, there exists a positive constant $C_6$ that makes the following holds: when $ \delta \leq \frac{C_6 \sigma}{K} \sqrt{ \frac{d + \log m}{n} }$, with probability at least $1 - e^{-t}$ there is a permutation $\tau$ of $[K]$ such that
\begin{itemize}
\item $\widehat{\btheta}_j = \widehat{\bbeta}_{\widehat{\bz}_j}$ and $\widehat{z}_j = \tau(z^{\star}_j)$ hold for all $j \in [m]$;
\item $\widehat{\bbeta}_k = \argmin_{\bbeta \in \RR^d}  \{ \sum_{j:~z_j^{\star} = \tau^{-1}(k) }  f_j ( \bbeta )  \} $ hold for all $k \in [K]$.
\end{itemize}
\end{theorem}

Take $\lambda = C K \sigma \sqrt{ \frac{ d + \log m}{n} }$ for some large constant $C$. By \Cref{thm-clustered}, clustered ARMUL satisfies
\begin{align*}
 \max_{j \in [m] } \| \widehat{\btheta}_j  - \btheta^{\star}_j \|_2  
& \lesssim 
\min\bigg\{ \sigma
\sqrt{ \frac{ K d }{mn} }
+     K  \delta      ,~
K \sigma \sqrt{ \frac{d + \log m}{n} }
\bigg\} .
\end{align*}
If $\delta = 0$ and $\{ z^{\star}_j \}_{j=1}^m$ are known, then we should pool all the data in each cluster.
Since each cluster has $O(m/K)$ tasks and $O(mn/K)$ samples, the estimation error has order $O( \sigma \sqrt{\frac{Kd}{mn}} )$. Clustered ARMUL achieves the same rate without knowing $\{ z^{\star}_j \}_{j=1}^m$.
As $\delta$ grows from $0$ to $+\infty$, the error bound gradually become $O(K \sigma \sqrt{ \frac{d + \log m}{n} })$. This is the error rate of independent task learning up to a factor of $K$ and an additive term $\log m$.
The theorem also states that when the discrepancy $\delta$ is small, all cluster labels $\{ z^{\star}_j \}_{j=1}^m$ are perfectly recovered up to a global permutation. The estimated centers $\{ \widehat{\bbeta}_k \}_{k=1}^K$ minimize empirical losses on pooled data in the corresponding clusters. The final estimates $\{ \widehat{\btheta}_{j} \}_{j=1}^m$ coincide with their cluster centers. 

When $\varepsilon > 0$, there can be tasks that are arbitrarily different from the others. We can prove that clustered ARMUL with cardinality constraints manages to utilize the task relatedness in a robust way. See Appendix E.3 for formal results including a minimax lower bound.

\subsection{Low-rank ARMUL}

In this subsection, we study the estimators $\{ \widehat{\btheta}_j \}_{j=1}^m$ returned by low-rank ARMUL 
\begin{align}
( \widehat{\bTheta}, \widehat{\bB}, \widehat{\bz} )
\in \argmin_{
	\bTheta \in \RR^{d\times m},~ \bB \in \RR^{d\times K}, ~\bZ \in \RR^{K \times m}
} \bigg\{ \sum_{j=1}^{m}  [ f_j (\btheta_j) + \lambda  \| \btheta_j - \bB \bz_j \|_2 ] \bigg\}.
\label{eqn-armul-lowrank-1}
\end{align}
Here $K \geq 1$ is the target rank. 
Ideally, we would adopt low-rank multi-task learning when $\{ \btheta^{\star}_j \}_{j=1}^m$ span a $K$-dimensional linear subspace and $K$ is much less than $d$. In other words, $\bTheta^{\star} = \bB^{\star} \bZ^{\star}$ holds for some $\bB^{\star} \in \RR^{d\times K}$ and $\bZ^{\star} \in \RR^{K\times m}$. To handle possible misspecification of the low-rank model, we introduce the following notion of task relatedness. Here we denote by $\cO_{d, K}$ the set of all $d\times K$ matrices with orthonormal columns.

\begin{assumption}[Task relatedness]\label{as-armul-lowrank}
There exist $\varepsilon,\delta \geq 0$, $K \in \ZZ_+$, $\bB^{\star} \in \cO_{d, K}$, $\{ \bz_j^{\star} \}_{j=1}^m \subseteq \RR^K$, $S \subseteq [m]$ and absolute constants $c_1, c_2 > 0$ such that the followings hold:
\begin{itemize}
	\item (Similarity) $\max_{j \in S } \| \btheta^{\star}_j - \bB^{\star} \bz_j^{\star} \|_2 \leq \delta $ and $|S^c| \leq \varepsilon m$;
	\item (Balancedness and signal strength) $\max_{j \in [m]} \| \bz_j^{\star} \|_2 \leq c_1$ and $\frac{K}{m} \sum_{j=1}^m \bz_j^{\star} \bz_j^{\star \top} \succeq c_2 \bI_K$.
\end{itemize}
\end{assumption}

Note that $\bB^{\star} \bZ^{\star} = (  \bB^{\star} \bR ) (\bR^{-1} \bZ^{\star})$ holds for any non-singular $\bR \in \RR^{K\times K}$. Without loss of generality, in Assumption \ref{as-armul-lowrank} we let $\bB^{\star}$ have orthonormal columns. The parameters $\{ \btheta^{\star}_j \}_{j \in S}$ are approximated by vectors $\{ \bB^{\star} \bz_j^{\star} \}_{j \in S}$ living in a $K$-dimensional linear subspace $\Range (\bB^{\star})$. The approximation errors are bounded by $\delta / \sqrt{n}$, which can be arbitrarily large.
The coefficient vectors $\{ \bz^{\star}_j \}_{j=1}^m$ are assumed to be uniformly bounded and spread out in all directions. 
The upper bound $\max_{j \in [m]} \| \bz_j^{\star} \|_2 \leq c_1$ and the lower bound $\frac{K}{m} \sum_{j=1}^m \bz_j^{\star} \bz_j^{\star \top} \succeq c_2 \bI_K$ imply that at least a constant fraction of $\bz^{\star}_j$'s are bounded away from $\bm{0}$.

The following theorem depicts the adaptivity of low-rank ARMUL to the unknown task relatedness. See Appendix E.4 for its proof. Here we only consider the case $\varepsilon= 0$ and focus on the impact of dissimilarity $\delta$. The general case ($\varepsilon > 0$) is left for future work.

\begin{theorem}[Low-rank ARMUL]\label{thm-lowrank}
	Let Assumptions \ref{as-stat-1}, \ref{as-stat-2} and \ref{as-armul-lowrank} hold, with $\varepsilon = 0$. 
	There exist positive constants $\{ C_i \}_{i=0}^5$ such that under the conditions $n  > C_1 K d ( \log n  )(  \log m )$, $0 \leq t <  C_2 n  / ( d \log n )$ and $C_3 K \sigma \sqrt{ \frac{d + \log m + t}{n}  } < \lambda < C_4 \sigma  $, the following bound holds for the estimator $\widehat{\bTheta}$ in \eqref{eqn-armul-clustered-1} with probability at least $1 -  e^{-t}$:
\begin{align*}
& \frac{1}{m} \sum_{j=1}^{m} [ F_j ( \widehat{\btheta}_j ) - F_j ( \btheta^{\star}_j ) ]
\leq L \max_{j \in [m] } \| \widehat{\btheta}_j  - \btheta^{\star}_j \|_2^2 \\
& \leq C_0  K^2 \bigg( \frac{  \sigma^2  d }{mn}  +   \frac{ \sigma^2 ( 1 + \log m + t  ) }{   n }  
+     \min \{  \delta^2 ,   \lambda^2 / K^2  \}    
\bigg)
\end{align*}
In addition, there exists a positive constant $C_6$ such that when $ \delta \leq \frac{ C_6 \sigma }{ K } \sqrt{ \frac{ d + \log m }{n}   }$, $\widehat{\bTheta} = \widehat{\bB} \widehat\bZ$ holds with probability at least $1 - e^{-t}$.
\end{theorem}

Suppose that $K$ is bounded and take $\lambda = C \sigma \sqrt{ \frac{d + \log m}{n}  }$ for some large constant $C$. By \Cref{thm-lowrank}, low-rank ARMUL satisfies
\begin{align}
\max_{j \in [m] } \| \widehat{\btheta}_j  - \btheta^{\star}_j \|_2  
\lesssim  
\min \bigg\{
\sigma \sqrt{ \frac{d}{mn} } + 
\sigma \sqrt{ \frac{1 }{n} } + \delta ,~
\sigma \sqrt{ \frac{d}{n} }
\bigg\}.
\label{eqn-lowrank-upper}
\end{align}
with high probability. 
We provide a matching minimax lower bound in Appendix E.4.
Again, the error rate never exceeds that for independent task learning. When $\delta$ is small, low-rank ARMUL adapts to the task relatedness. Note that $O( \sigma \sqrt{ \frac{d}{mn} } + \frac{\sigma }{\sqrt{n}} )$ is the best rate one can achieve when the low-rank model is true ($\delta = 0$). In that case, the unknown matrix $\bTheta^{\star} = \bB^{\star} \bZ^{\star}$ has $O( d+m )$ unknown parameters. The $mn$ samples imply an error bound $O(  \sigma\sqrt{\frac{d+m}{mn}} ) = O( \sigma \sqrt{ \frac{d}{mn} } + \frac{\sigma }{\sqrt{n}} )$. The two terms can be viewed as estimation errors of bases $\bB^{\star}$ and coefficients $\bZ^{\star}$, respectively.

\section{Numerical experiments}\label{sec-numerical}

We conduct simulations to verify our theories and real data experiments to test the efficacy our proposed approaches. Our implementations of ARMUL follow the description in \Cref{sec-implementations}.
The code and all numerical results are available at \texttt{https://github.com/kw2934/ARMUL/}.

\subsection{Simulations}

We generate synthetic data for multi-task linear regression. Throughout our simulations, the number of tasks is $m = 30$. For any $j \in [m]$, the dataset $\cD_j$ consists of $n = 200$ samples $\{ (\bx_{ji}, y_{ji}) \}_{i=1}^n$. The covariate vectors $\{ \bx_{ji} \}_{(i, j) \in [n] \times [m]}$ are i.i.d.~$N(\bm{0}, \bI_d)$ with $d = 50$, given which we sample each response $y_{ji} = \bx_{ji}^{\top} \btheta^{\star}_j + \varepsilon_{ji}$ from a linear model with noise term $\varepsilon_{ji} \sim N(0,1)$ being independent of the covariates. To study vanilla, clustered and low-rank ARMUL, we determine the coefficient vectors $\{ \btheta^{\star}_{j} \}_{j=1}^m$ in three different ways. The parameters $\varepsilon$ and $\delta$ below characterize task relatedness, similar to those in Assumptions \ref{as-armul-vanilla-00}, \ref{as-armul-clustered} and \ref{as-armul-lowrank}.
Below we write $r \SSS^{d-1} $ as a shorthand notation for the sphere $\{ \bx \in \RR^d:~ \| \bx \|_2 = r \}$. 

\begin{enumerate}
\item Vanilla case 
\begin{itemize}
\item Data generation: Set $\bbeta^{\star} = 2 \be_1$, sample i.i.d.~random vectors $\{ \bdelta_j \}_{j=1}^m$ uniformly from the sphere $\delta \SSS^{d-1}$ and set $\btheta_j^{\star} = \bbeta^{\star} + \bdelta_j$ for all $j \in [m]$. Next, draw $\lceil \varepsilon m \rceil$ elements $\{ j_{s} \}_{s = 1}^{ \lceil \varepsilon m \rceil }$ uniformly at random from $[m]$ without replacement. Replace $\{ \btheta_{j_{s}}^{\star} \}_{s = 1}^{ \lceil \varepsilon m \rceil }$ with i.i.d.~random vectors from $2 \SSS^{d-1}$. Denote by $S = [m] \backslash \{ j_{1}, \cdots, j_{ \lceil \varepsilon m \rceil } \}$.
\item Methods for comparison: vanilla ARMUL \eqref{eqn-armul-vanilla-00}, independent task learning (\Cref{eg-stl}) and data pooling (\Cref{eg-pooling}).
\end{itemize}

\item Clustered case
\begin{itemize}
\item Set $K = 3$, $\bbeta_k^{\star} = 2 \be_k$ for $k \in [K]$ and $\bz_j^{\star} = (j \mod K) + 1$ for $j \in [m]$. 
Sample i.i.d.~random vectors $\{ \bdelta_j \}_{j=1}^m$ uniformly from the sphere $\delta \SSS^{d-1}$ and set $\btheta_j^{\star} = \bbeta_{z^{\star}_j}^{\star} \!+ \bdelta_j$ for all $j \in [m]$. Replace an $\varepsilon$-fraction of the coefficient vectors by the corresponding procedure in the vanilla case.
\item Methods for comparison: clustered ARMUL \eqref{eqn-armul-clustered-1}, clustered MTL (\Cref{eg-clustered}), independent task learning (\Cref{eg-stl}) and data pooling (\Cref{eg-pooling}).
\end{itemize}

\item Low-rank case 
\begin{itemize}
\item Set $K = 3$ and $\bB^{\star} = (\be_1,\be_2,\be_3) \in \RR^{m\times K}$. Samples $\{ \bz_j^{\star} \}_{j=1}^m$ independently~from $N(\bm{0}, \bI_K)$ and another set of i.i.d.~vectors $\{ \bdelta_j \}_{j=1}^m$ uniformly from the sphere $\delta \SSS^{d-1}$. Let $\btheta_j^{\star} = \bB^{\star} \bz_j^{\star} + \bdelta_j$ for all $j \in [m]$. Replace an $\varepsilon$-fraction of the coefficient vectors by the corresponding procedure in the vanilla case.
\item Methods for comparison: low-rank ARMUL \eqref{eqn-armul-lowrank-1}, low-rank MTL (\Cref{eg-lowrank}), independent task learning (\Cref{eg-stl}) and data pooling (\Cref{eg-pooling}).
\end{itemize}
\end{enumerate}

Guided by the theories in \Cref{sec-theory}, we set the regularization parameter $\lambda$ 
in ARMUL algorithms \eqref{eqn-armul-vanilla-00}, \eqref{eqn-armul-clustered-1}, \eqref{eqn-armul-lowrank-1} to be $c \sqrt{d / n} $ and select the optimal pre-constant $c$ from $\{ 0.2, 0.4, 0.6, \cdots, 2 \}$ by 5-fold cross-validation. Below is how we evaluate the quality of each $c$: 
\begin{itemize}
\item Step 1: Randomly partition each dataset $\cD_j\!$ into 5 (approximately) equally-sized subsets $\{ \cD_{j\ell} \}_{\ell=1}^5$.\!\!\!\!\!\!\!
\item Step 2: For $\ell = 1,\cdots, 5$, define $\widetilde{\cD}_j^{(\ell)} = \cup_{s \neq \ell} \cD_s$, conduct ARMUL on $\{ \widetilde{\cD}_j^{(\ell)}  \}_{j=1}^m$ with $\lambda = c \sqrt{d / n}$, test the obtained models on $\{ \widetilde{\cD}_{j \ell}  \}_{j=1}^m$.
\item Step 3: Get the average of mean squared prediction errors over all tasks.
\end{itemize}

We vary $\varepsilon$ in $\{ 0,  0.2\}$ and $\delta$ in $\{ 0, 0.1,  0.2,\cdots, 1 \}$ to obtain tasks with different degrees of relatedness. When $\varepsilon = 0$, we measure the maximum estimation error $\max_{j \in [m]} \| \widehat{\btheta}_j - \btheta^{\star}_j \|_2$. For $\varepsilon = 0$, we measure the maximum estimation error $\max_{j \in [m]} \| \widehat{\btheta}_j - \btheta^{\star}_j \|_2$ and its restricted version $\max_{j \in S} \| \widehat{\btheta}_j - \btheta^{\star}_j \|_2$ on the set $S$ of similar tasks. Figures \ref{fig-epsilon_0} and \ref{fig-epsilon_2} demonstrate how the estimation errors grow with the heterogeneity parameter $\delta$. The curves and error bands show the means and standard deviations over 100 independent runs, respectively. 

\begin{figure}[h]
	\centering
	\includegraphics[width=0.8\linewidth]{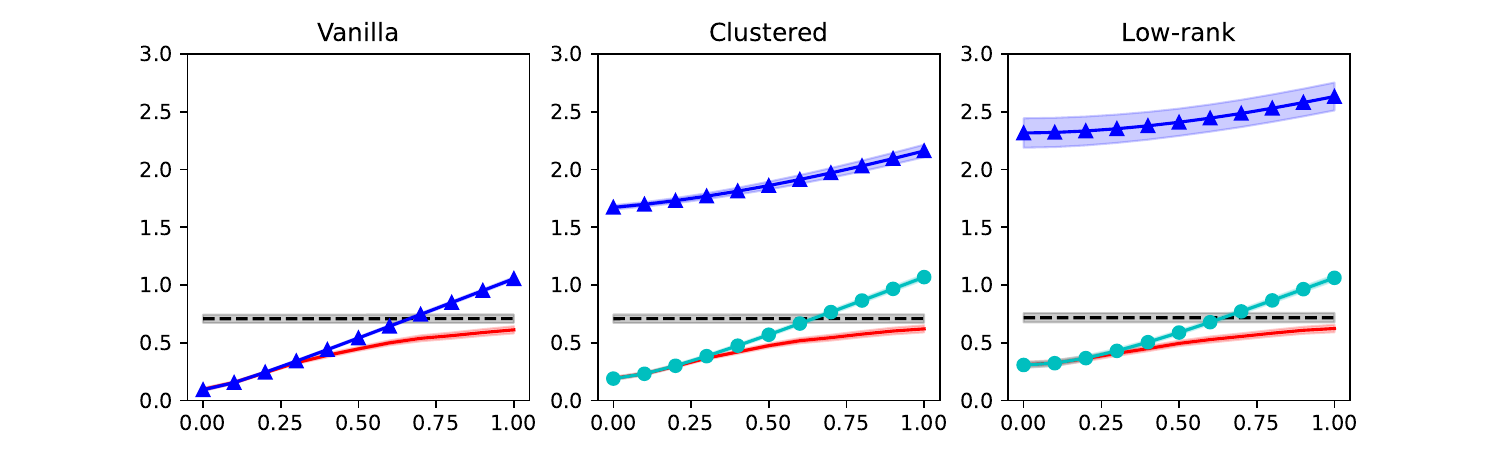}   
	\caption{Impact of task relatedness when $\varepsilon = 0$. From left to right: vanilla, clustered and low-rank cases. $x$-axis: $\delta$. $y$-axis: $\max_{j \in [m]} \| \widehat{\btheta}_j - \btheta^{\star}_j \|_2 $. Red solid lines: ARMUL. Blue triangles: data pooling. Black dashed lines: independent task learning. Cyan circles: clustered MTL (middle) or low-rank MTL (right).}\label{fig-epsilon_0}
\end{figure}

\begin{figure}[h]
	\centering
	\subfigure{
		\includegraphics[width=0.8\linewidth]{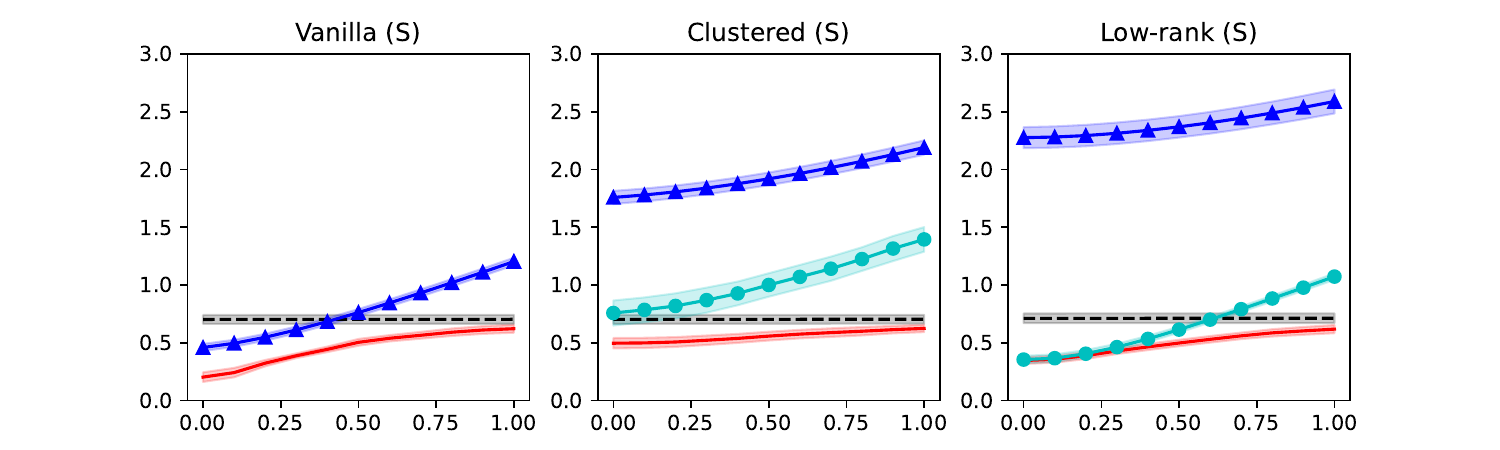}   
	}
	\\
	\subfigure{
		\includegraphics[width=0.8\linewidth]{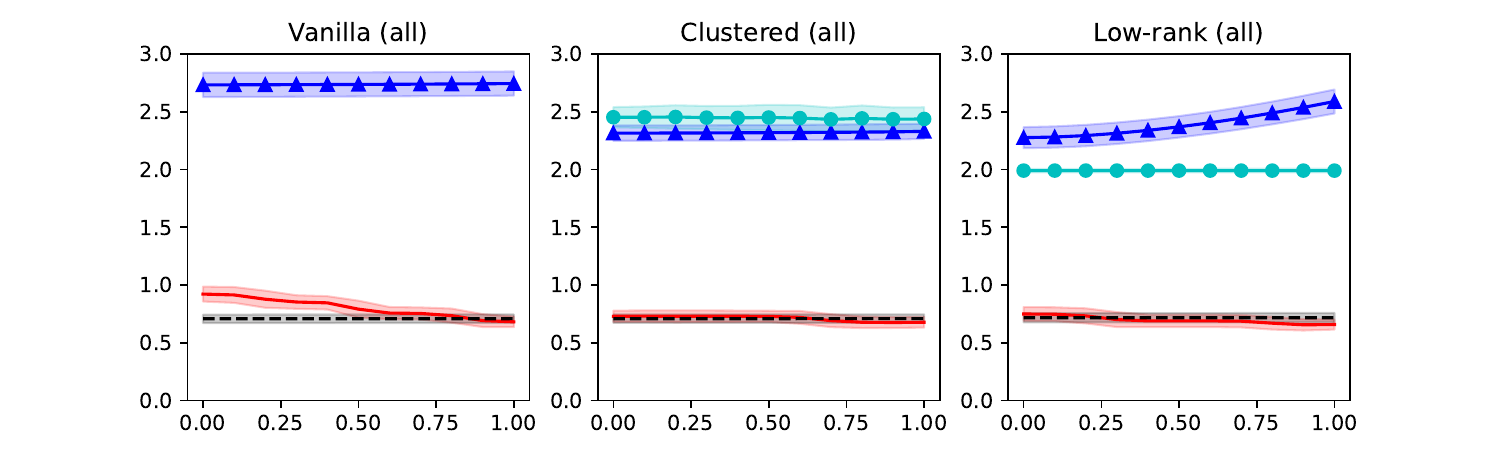}   
	}
	\caption{Impact of task relatedness when $\varepsilon = 0.2$.  From left to right: vanilla, clustered and low-rank cases. $x$-axis: $\delta$. $y$-axis: $\max_{j \in S} \| \widehat{\btheta}_j - \btheta^{\star}_j \|_2 $ (top) or $\max_{j \in [m]} \| \widehat{\btheta}_j - \btheta^{\star}_j \|_2 $ (bottom). Red solid lines: ARMUL. Blue triangles: data pooling. Black dashed lines: independent task learning. Cyan circles: clustered MTL (middle) or low-rank MTL (right).}\label{fig-epsilon_2}
\end{figure}

The simulations confirm the adaptivity and robustness of ARMUL methods, as stated in \Cref{thm-personalization-00,thm-vanilla-00,thm-clustered,thm-lowrank}. When $\varepsilon = 0$ and $\delta$ is small, the vanilla, clustered and low-rank ARMUL coincide with data pooling, clustered MTL and low-rank MTL, respectively. However, the latter are too rigid and therefore deteriorate quickly as $\delta$ grows. ARMUL methods, on the other hand, nicely handle model misspecifications and never underperform independent task learning. When $\varepsilon$ becomes $0.2$, ARMUL methods continue to work well on the set $S$ of similar tasks while data pooling and clustered MTL are badly affected. For the exceptional tasks in $S^c$, the error curves for $\max_{j \in [m]} \| \widehat{\btheta}_j - \btheta^{\star}_j \|_2$ in Figure \ref{fig-epsilon_2} implies that ARMUL methods are still comparable to independent task learning. As we have studied in \Cref{thm-personalization-00}, ARMUL estimators always stay close to the loss minimizers associated to individual tasks. They are generalizations of limited translation estimators \citep{EMo72,Ste81} to multivariate $M$-estimation. In contrast, data pooling, clustered MTL and low-rank MTL perform poorly on $S^c$.


\subsection{Real data}

We evaluate the proposed ARMUL methods on a real-world dataset. The Human Activity Recognition (HAR) database is built by \cite{AGO13} from the recordings of $30$ volunteers performing activities of daily living while carrying a waist-mounted smartphone with embedded inertial sensors. On average, each volunteer has 343.3 samples (min: 281, max: 409). Each sample corresponds to one of six activities (walking, walking upstairs, walking downstairs, sitting, standing, and laying) and has a 561-dimensional feature vector with time and frequency domain variables. 

We model each volunteer as a task and aim to distinguish between sitting and the other activities. The problem is therefore formulated as multi-task logistic regression with $m = 30$ tasks. We conduct Principal Component Analysis to reduce the dimension to 100. Together with the intercept term, the preprocessed data have $d = 101$ variables in total. We randomly select 20\% of the data from each task for testing, and train logistic models on the rest of the data. The sample sizes $\{ n_j \}_{j=1}^m$ for training range from 225 to 328. 
We apply three ARMUL methods (vanilla, clustered and low-rank) and four benchmark approaches (independent task learning, data pooling, clustered MTL and low-rank MTL) to standardized data. For each ARMUL method, we set $w_j = n_j$ and $\lambda_j = c \sqrt{ d / n_j }$ in \eqref{eqn-armul}, as is suggested by our results for general sample sizes (Theorems D.1 and D.2). The constant factor $c$ is chosen from $\{ 0.05, 0.1, 0.15,\cdots, 0.5 \}$ using 5-fold cross-validation. We use the same procedure to select the number of clusters $K$ in clustered methods from $\{2, 3, 4, 5 \}$ and the rank $K$ in low-rank methods from $\{1, 2, 3, 4, 5 \}$. Finally, we compute the misclassification error on testing data for each method.

Table \ref{table-real} summarizes the means and standard deviations (in parentheses) of test error rates (in percentage) over 100 independent runs, where ITL stands for independent task learning. The randomness comes from train/test splits and cross-validation.
We see that ARMUL methods significantly outperform benchmarks. In addition, we observe several interesting phenomena.
\begin{itemize}
\item The tasks are rather heterogeneous, since data pooling and clustered MTL are even worse than independent task learning. As the method becomes more flexible (from data pooling to clustered MTL and then low-rank MTL), the performance gets better. The same trend appears in ARMUL methods as well.

\item An ARMUL method augments a basic multi-task learning method with models for individual tasks. Such augmentation brings great benefits: even the augmented version of data pooling (i.e. vanilla ARMUL) works better than the raw version of low-rank MTL.
\end{itemize}


\section{Discussions}\label{sec-discussions}

We introduced a framework for multi-task learning named ARMUL that can be used as a wrapper around any multi-task learning algorithm of the form \eqref{eqn-mtl-obj}. 
We analyzed its adaptivity to unknown task relatedness, where the unsquared $\ell_2$ penalty function plays a crucial role. We also verified the theories by extensive numerical experiments.
We hope that our framework can spur further research in related fields.
It would be interesting to develop methods for high-dimensional problems with sparsity or other structures, and build inferential tools for uncertainty quantification. Since heterogeneous datasets are often collected and stored at multiple sites, communication-efficient procedures for distributed statistical inference are desirable. 
Another direction is to extend our methods to meta-learning, also known as learning to learn \citep{TPr12}. The goal is to extract from existing tasks useful knowledge (e.g., common representation) that facilitates learning future tasks of similar type. Our framework could provide a principled way of dealing with misspecified similarity structure.

\begin{table}
	\caption{\label{table-real}Test error rates (in percentage) on the HAR dataset.} 
	\begin{tabular}{|c|c|c|c|c|c|c|}
	\hline
	\multicolumn{3}{|c|}{ARMUL}   &  \multicolumn{4}{c|}{Benchmarks}  \\  \hline
	Vanilla  &  Clustered &  Low-rank & ITL & Data pooling & Clustered & Low-rank \\ \hline
	\!{\bf 1.12 \!(0.25)}\! & \!{\bf 0.84 \!(0.22)}\! & \!{\bf 0.80 \!(0.19)}\! & \!1.95 \!(0.32)\! & \!3.48 \!(0.39)\! & \!2.15 \!(0.33)\! & \!1.30 \!(0.23)\! \\
	\hline
\end{tabular}
\end{table}

\section*{Acknowledgement}
We are grateful to two anonymous referees and the associate editor for their helpful comments.
We thank Chen Dan, Dongming Huang, Yuhang Wu and Yichen Zhang for discussions. 
Kaizheng Wang's research is supported by an NSF grant DMS-2210907 and a startup grant at Columbia University. 
We acknowledge computing resources from Columbia University's Shared Research Computing Facility project, which is supported by NIH Research Facility Improvement Grant 1G20RR030893-01, and associated funds from the New York State Empire State Development, Division of Science Technology and Innovation (NYSTAR) Contract C090171, both awarded April 15, 2010. 
Part of the research was conducted when Yaqi Duan was affiliated with the Laboratory for Information and Decision Systems at Massachusetts Institute of Technology and the Department of Operations Research and Financial Engineering at Princeton University.

\newpage 

\appendix

\section{Deterministic analysis of ARMUL}\label{sec-deterministic}

In this subsection, we present deterministic results for ARMUL
\begin{align*}
(\widehat{\bTheta}, \widehat{\bGamma})
\in \argmin_{
	\bTheta \in \RR^{d\times m},~ \bGamma \in \Omega} \bigg\{ \sum_{j=1}^{m} w_j [ f_j (\btheta_j) + \lambda_j \| \btheta_j - \bgamma_j \|_2 ] \bigg\},
\end{align*}
with loss functions $\{ f_j \}_{j=1}^m$, weights $\{ w_j \}_{j=1}^m$ and regularization parameters $\{ \lambda_j \}_{j=1}^m$. Denote by $\{ \btheta^{\star}_j \}_{j=1}^m \subseteq \RR^d$ the target parameters. We estimate them by $\{ \widehat{\btheta}_j \}_{j=1}^m$.
We will first study the general case ($\Omega$ can be any non-empty subset of $\RR^{d\times m}$) and then come down to vanilla, clustered and low-rank versions.

\subsection{Personalization}

\begin{definition}[Regularity]\label{defn-armul-regularity}
	Let $ \btheta^{\star} \in \RR^d$, $0 < M \leq +\infty$ and $0 < \rho \leq L \leq +\infty$.
	A function $f:~\RR^d \to \RR$ is said to be $(\btheta^{\star}, M, \rho, L)$-regular if
	\begin{itemize}
		\item $f$ is convex and twice differentiable;
		\item $\rho \bI \preceq \nabla^2 f (\btheta) \preceq L \bI$ holds for all $\btheta \in B(\btheta^{\star}, M)$;
		\item $ \| \nabla f (\btheta^{\star}) \|_2  \leq \rho M / 2$.
	\end{itemize}
\end{definition}

When the loss functions satisfy the regularity condition above, we can control the difference between the ARMUL estimate $\widehat{\btheta}_j$ and its target $\btheta^{\star}_j$.

\begin{theorem}[Personalization]\label{thm-armul-personalization}
	If $f_j$ is $(\btheta_j^{\star}, M, \rho, +\infty)$-regular and $0\leq \lambda_j < \rho M /2$, then 
	\[
	\| \widetilde{\btheta}_j - \btheta^{\star}_j \|_2 \leq \frac{ \| \nabla f_j (\btheta_j^{\star}) \|_2 }{\rho}
	\qquad\text{and}\qquad
	\| \widehat{\btheta}_j - \widetilde{\btheta}_j \|_2 \leq \frac{ \lambda_j }{\rho}.
	\]
\end{theorem}

\begin{proof}[\bf Proof of \Cref{thm-armul-personalization}]
	By \Cref{lem-consensus-min} and $\| \nabla f_j (\btheta^{\star}_j) \|_2 < \rho M / 2$, $f_j$ has a unique minimizer $\widetilde\btheta_j$ and $\| \widetilde\btheta_j - \btheta^{\star}_j \|_2 \leq \| \nabla f_j (\btheta^{\star}_j) \|_2 / \rho < M / 2$. Hence $ \nabla^2 f_j(\btheta) \succeq \rho \bI $ holds for $\btheta \in B(\widetilde\btheta_j,  M / 2  )$. By \Cref{lem-cvx-reg} and $\lambda_j < \rho M /2 $, $\| \widehat{\btheta}_j - \widetilde\btheta_j \|_2 \leq \lambda_j / \rho$. 
\end{proof}

\subsection{Vanilla ARMUL}

We study the vanilla ARMUL estimators computed from
\begin{align}
(\widehat{\bTheta},\widehat{\bbeta}) \in \argmin_{ \bTheta \in \RR^{d\times m},~\bbeta \in \RR^d } \bigg\{ 
\sum_{j=1}^{m} w_j \bigg(  f_j ( \btheta_j ) + \frac{ \lambda }{\sqrt{w_j}} \|  \bbeta - \btheta_j  \|_2 \bigg)
\bigg\} .
\label{eqn-armul-deterministic}
\end{align}

%
%
%

\begin{definition}[Task relatedness]\label{defn-armul-relatedness}
	Let $\varepsilon ,\delta \geq 0$, $\{ \btheta^{\star}_j \}_{j=1}^m \subseteq \RR^d$, $0 < M \leq +\infty$ and $0 < \rho \leq L < +\infty$. $\{ f_j, w_j \}_{j=1}^m$ are said to be $(\varepsilon,\delta )$-related with regularity parameters $(\{ \btheta_j^{\star} \}_{j=1}^m, M, \rho, L)$ if there exists $S \subseteq [m]$ such that
	\begin{itemize}
		\item for any $j \in S$, $f_j$ is $(\btheta_j^{\star}, M, \rho, L)$-regular (\Cref{defn-armul-regularity});
		\item $\min_{\btheta \in \RR^d}  \max_{j \in S} \{ \sqrt{w_j}\| \btheta^{\star}_j - \btheta  \|_2 \} \leq \delta$ and $  \sum_{ j \in S^c } \sqrt{w_j} \leq \varepsilon \sum_{j \in S} w_j / (\max_{j \in S } \sqrt{w_j} ) $.
	\end{itemize}
\end{definition}

\begin{theorem}[Adaptivity and robustness]\label{thm-armul-deterministic}
Let $\{ f_j , w_j \}_{j=1}^m$ be $(\varepsilon,\delta )$-related with regularity parameters $(\{ \btheta_j^{\star} \}_{j=1}^m, M, \rho, L)$. Define $\kappa = L / \rho$ and
\[
\kappa_w = \max_{j \in S } \sqrt{w_j} \cdot \sum_{ j \in S } \sqrt{w_j} /\sum_{j \in S} w_j .
\]
Suppose that $\kappa \varepsilon < 1 $ and
\[
\frac{5 \kappa \kappa_w }{ 1 -  \kappa \varepsilon } \cdot  \max_{j \in S} \{ \sqrt{w_j}  \| \nabla f_j (\btheta_j^{\star}) \|_2 \} < \lambda < \frac{ \rho M}{2}  \cdot \min_{j \in S } \sqrt{w_j}  
.
\]
Then, the estimator $\widehat{\bTheta}$ in \eqref{eqn-armul-deterministic} satisfies
\begin{align*}
 \| \widehat{\btheta}_j  - \btheta^{\star}_j \|_2  
& \leq \frac{ \| \sum_{k \in S} w_k \nabla f_k (\btheta^{\star}_k) \|_2 }{ \rho \sum_{k \in S} w_k }
+     \frac{6 }{ ( 1 -   \kappa \varepsilon ) \sqrt{w_j} } \min\bigg\{ 3 \kappa^2 \kappa_w \delta ,~    \frac{ 2 \lambda }{ 5 \rho }   \bigg\} \\
& ~~~~ +    \frac{  \varepsilon \lambda  }{ \rho \max_{k \in S} \sqrt{w_k } }
, \qquad\forall j \in S.
\end{align*}
Moreover, there exists a constant $C$ such that under the conditions $\varepsilon = 0$ and $C \kappa \kappa_w  L \delta < \lambda $, we have $\widehat{\btheta}_1 = \cdots = \widehat{\btheta}_m = \argmin_{\btheta \in \RR^d} \{ \sum_{j=1}^m w_j f_j (\btheta) \}$.
\end{theorem}

\begin{proof}[\bf Proof of \Cref{thm-armul-deterministic}]
See \Cref{sec-thm-armul-deterministic-proof}.
\end{proof}

\subsection{Clustered ARMUL}

In this subsection, we analyze the clustered ARMUL estimators $\{ \widehat{\btheta}_j \}_{j=1}^m$ returned by
\begin{align}
(\widehat{\bTheta}, \widehat{\bB},\widehat{\bz}) \in
\argmin_{ 
	\bTheta \in \RR^{d\times m},~ \bB\in \RR^{d\times K}, \bz \in [K]^m
} \bigg\{ 
\sum_{j=1}^{m} [ f_j ( \btheta_{j}  ) + \lambda \| \bbeta_{z_j} - \btheta_j  \|_2 ]
\bigg\} 
\label{eqn-obj-clustered-0}
\end{align}
and its variant
\begin{align}
(\widehat{\bTheta}, \widehat{\bB},\widehat{\bz}) \in
\argmin_{ 
\substack{ \bTheta \in \RR^{d\times m},~ \bB\in \RR^{d\times K}, \bz \in [K]^m \\ 
\min_{k \in [K]} |\{ j \in [m] :~ z_j = k \}| \geq \alpha m / K
 }
} \bigg\{ 
\sum_{j=1}^{m} [ f_j ( \btheta_{j}  ) + \lambda \| \bbeta_{z_j} - \btheta_j  \|_2 ]
\bigg\} 
\label{eqn-obj-clustered}
\end{align}
with an additional cardinality constraint on the cluster labels. Here $\alpha \in (0, 1]$ is a tuning parameter.
Note that $\{ f_j \}_{j=1}^m$ can be general loss functions and not necessarily sample averages. 
While the penalty parameters are rescaled by $1/\sqrt{n}$ in \Cref{eqn-armul-clustered-1}, we do not do that here for notational simplicity. The results immediately translate to the rescaled case.

\begin{definition}[Task relatedness]\label{defn-armul-clustered-relatedness}
Let $\varepsilon ,\delta \geq 0$, $\{ \btheta^{\star}_j \}_{j=1}^m \subseteq \RR^d$, $0 < M \leq +\infty$, $0 < \rho \leq L < +\infty$, $K \geq 2$, $\{ \bbeta^{\star}_k \}_{k=1}^K \subseteq \RR^d$ and $\{ z_j^{\star} \}_{j=1}^m \subseteq [K]$. $\{ f_j \}_{j=1}^m$ are said to be $(\varepsilon,\delta )$-related with regularity parameters $(\{ \btheta_j^{\star} \}_{j=1}^m, M, \rho, L, \{ \bbeta^{\star}_k \}_{k=1}^K, \{ z_j^{\star} \}_{j=1}^m )$ if there exists $S \subseteq [m]$ such that
\begin{itemize}
\item for any $j \in S$, $f_j$ is $(\btheta_j^{\star}, M, \rho, L)$-regular (\Cref{defn-armul-regularity});
\item $\max_{j \in S } \| \btheta^{\star}_j - \bbeta^{\star}_{z_j^{\star}} \|_2 \leq \delta $ and $|S^c| \leq \varepsilon m$;
\end{itemize}
\end{definition}

\begin{theorem}[Adaptivity of the estimator \eqref{eqn-obj-clustered-0}]\label{thm-armul-clustered-deterministic-0}
Let $\{ f_j \}_{j=1}^m$ be $(0,\delta )$-related with regularity parameters
\[
(\{ \btheta_j^{\star} \}_{j=1}^m, M, \rho, L, \{ \bbeta^{\star}_k \}_{k=1}^K, \{ z_j^{\star} \}_{j=1}^m ).
\]
Define $\kappa = L/\rho$, $T(k) = \{ j \in [m]:~z_j^{\star} = k \}$ for $k \in [K]$, and $\kappa_m =  m / ( K \min_{k \in [K]} |T(k)| ) $.
Suppose that
\begin{align*}
10 \sqrt{\kappa  \kappa_m} K \max_{j \in [m]}   \| \nabla f_j (\btheta_j^{\star}) \|_2   < \lambda < \frac{\rho}{2}   \min \Big\{ M,~ \min_{k \neq l}
\| \bbeta^{\star}_k - \bbeta^{\star}_l \|_2 \Big\}.
\end{align*}
Then, the estimator $\widehat{\bTheta}$ in \eqref{eqn-obj-clustered-0} satisfies
\begin{align*}
& \| \widehat{\btheta}_j  - \btheta^{\star}_j \|_2    \leq \frac{ \| \sum_{j \in T( z^{\star}_j ) }   \nabla f_j (  \btheta^{\star}_{ j }  ) \|_2 }{  \rho |T( z^{\star}_j )| }
+     \min\bigg\{ 77 \sqrt{\kappa^3 \kappa_m} K \delta ,~    \frac{ 11 \lambda }{ 5 \rho }   \bigg\}  ,\qquad\forall j \in [m].
\end{align*}
Moreover, if $\delta \leq  \lambda / ( 70 \sqrt{\kappa \kappa_m} KL )$, there exists a permutation $\tau$ of $[K]$ such that
\begin{align*}
&  \widehat{z}_j = \tau (z_j^{\star}) \qquad\text{and}\qquad \widehat{\btheta}_j = \widehat\bbeta_{\widehat{z}_j },\qquad\forall  j \in [m] ,\\
& \widehat{\bbeta}_{\tau(k) } = \argmin_{\bbeta \in \RR^d}  \sum_{j \in T(k)}  f_j ( \bbeta ) 
, \qquad \forall k \in [K] .
\end{align*}
\end{theorem}

\begin{proof}[\bf Proof of \Cref{thm-armul-clustered-deterministic-0}]
See \Cref{sec-thm-armul-clustered-deterministic-0-proof}.
\end{proof}

\begin{theorem}[Adaptivity and robustness of the estimator  \eqref{eqn-obj-clustered}]\label{thm-armul-clustered-deterministic}
	Let $\{ f_j \}_{j=1}^m$ be $(\varepsilon,\delta )$-related with regularity parameters
	\[
	(\{ \btheta_j^{\star} \}_{j=1}^m, M, \rho, L, \{ \bbeta^{\star}_k \}_{k=1}^K, \{ z_j^{\star} \}_{j=1}^m ).
	\]
Define $\kappa = L/\rho$ and $T (k) = \{ j \in [m]:~z_j^{\star} = k \}$ for $k \in [K]$. 
Suppose that $\varepsilon \leq \frac{\alpha}{6 \kappa K^2}$, $  \min_{k \in [K]} |T(k)| \geq \frac{7 \alpha m}{6 K}$ and
\begin{align*}
	\frac{ 12 \kappa K }{\sqrt{\alpha}} \max_{j \in [m]}   \| \nabla f_j (\btheta_j^{\star}) \|_2   < \lambda < \frac{ \rho }{6}   \min \Big\{ M,~ \min_{k \neq l}
\| \bbeta^{\star}_k - \bbeta^{\star}_l \|_2 \Big\}.
\end{align*}
Then, the estimator $\widehat{\bTheta}$ in \eqref{eqn-obj-clustered-0} satisfies
\begin{align*}
 \| \widehat{\btheta}_j  - \btheta^{\star}_j \|_2   
& \leq 
\frac{
	\|  \sum_{j \in T(z_j^{\star}) \cap S } \nabla f_j ( \btheta^{\star}_{ j } )  \|_2
}{\rho |T(z_j^{\star}) \cap S| }
+     \min\bigg\{ \frac{ 55 \kappa^2 K \delta }{  \sqrt{\alpha} } ,~    \frac{ 11 \lambda }{ 5 \rho }   \bigg\}  \\
&~~~~ + \frac{ \lambda |S^c|   }{  \rho |T(z_j^{\star}) \cap S| }  
,\qquad\forall j \in S.
\end{align*}
Moreover, if $\delta \leq \sqrt{\alpha} \lambda / ( 100 \kappa  KL )$, there exists a permutation $\tau$ of $[K]$ such that $ \widehat{z}_j = \tau (z_j^{\star})$ and $\widehat{\btheta}_j = \widehat\bbeta_{\widehat{z}_j }$ hold for all $j \in S$.
\end{theorem}

\begin{proof}[\bf Proof of \Cref{thm-armul-clustered-deterministic}]
See \Cref{sec-thm-armul-clustered-deterministic-proof}.
\end{proof}

\subsection{Low-rank ARMUL}

In this subsection, we analyze the low-rank ARMUL estimators $\{ \widehat{\btheta}_j \}_{j=1}^m$ returned by
\begin{align}
(\widehat{\bTheta}, \widehat{\bB},\widehat{\bZ}) \in
\argmin_{ 
	\bTheta \in \RR^{d\times m},~ \bB\in \RR^{d\times K}, \bZ \in \RR^{K\times m} 
} \bigg\{ 
\sum_{j=1}^{m} [ f_j ( \btheta_{j}  ) + \lambda \| \bB \bz_j - \btheta_j  \|_2 ]
\bigg\} .
\label{eqn-obj-lowrank}
\end{align}

\begin{definition}[Task relatedness]\label{defn-armul-lowrank-relatedness}
Let $\delta \geq 0$, $\{ \btheta^{\star}_j \}_{j=1}^m \subseteq \RR^d$, $0 < M \leq +\infty$, $0 < \rho \leq L < +\infty$, $K \in \ZZ_+$, $\bB^{\star} \in \RR^{d \times K}$ and $\bZ^{\star} \in \RR^{K \times m}$. $\{ f_j \}_{j=1}^m$ are said to be $ \delta$-related with regularity parameters $(\{ \btheta_j^{\star} \}_{j=1}^m, M, \rho, L, \bB^{\star} , \bZ^{\star}  )$ if for any $j \in [m]$, $f_j$ is $(\btheta_j^{\star}, M, \rho, L)$-regular (\Cref{defn-armul-regularity}) and $ \| \btheta^{\star}_j - \bB^{\star} \bz_j^{\star} \|_2 \leq \delta $.
\end{definition}

\begin{theorem}[Adaptivity of the estimator \eqref{eqn-obj-lowrank}]\label{thm-armul-lowrank-deterministic}
	Let $\{ f_j \}_{j=1}^m$ be $\delta$-related with regularity parameters $(\{ \btheta_j^{\star} \}_{j=1}^m, M, \rho, L, \bB^{\star} , \bZ^{\star}  )$.
	Define $\bG \in \RR^{d\times m} $ with $\bg_j = \nabla f_j (\btheta^{\star}_j) $, $\kappa = L/\rho$, $\alpha_1 = \max_{j \in [m] } \| \bz^{\star}_j \|_2 $, $\alpha_2 = \sqrt{K/m} \sigma_{\min} (\bZ^{\star})$ and $\mu = \alpha_1/\alpha_2$. Here $\sigma_{\min}(\cdot)$ refers to the smallest positive singular value of a non-zero matrix. 
Denote by $\bP^{\star} \in \RR^{d\times d}$ the projection onto $\Range(\bB^{\star})$. There exist positive constants $C_0$ and $C$ such that when
	\begin{align*}
2 C_0   \kappa^{3/2}  \mu^2 K  \max_{j \in [m]} \| \bg_j \|_2 
< \lambda \leq \rho \min \{ M / 2,~\alpha_1 \},
\end{align*}
 the estimator $\widehat{\bTheta}$ in \eqref{eqn-obj-lowrank} satisfies
	\begin{align*}
	& \max_{j \in [m] } \| \widehat{\btheta}_j  - \btheta^{\star}_j \|_2    \leq 
C \bigg( \frac{ \max_{j \in [m] }  \| \bP^{\star}  \bg_j \|_2  }{\rho} + \frac{3 \kappa^{3/2} \mu K  \| \bG \|_2 }{\sqrt{m}} +  \min\bigg\{  6C_0^2 \kappa^{5/2} \mu^2 K  \delta,~ 
\frac{3\lambda }{ \rho }  
\bigg\} \bigg).
	\end{align*}
In addition, if $ \delta \leq  \frac{\lambda }{ 4 C_0^2 \kappa^{3/2} \mu^2 K L } $, then $\widehat{\bTheta} = \widehat{\bB} \widehat\bZ$.
\end{theorem}

\begin{proof}[\bf Proof of \Cref{thm-armul-lowrank-deterministic}]
	See \Cref{sec-thm-armul-lowrank-deterministic-proof}.
\end{proof}

\section{Proofs of deterministic results}

Throughout the proofs, we use $f\square g$ to refer to the infimal convolution of two convex functions $f$ and $g$ over $\RR^d$: $f\square g (\bx) = \inf_{\by \in \RR^d} \{ f(\by) + g(\bx - \by) \}$.

\subsection{Proof of \Cref{thm-armul-deterministic}}\label{sec-thm-armul-deterministic-proof}

We invoke the lemma below, whose proof is in \Cref{sec-thm-robustness-proof}.

\begin{lemma}\label{thm-robustness}
	Let $\{ f_j \}_{j=1}^m$ be convex and differentiable. Suppose there are $S \subseteq [m]$, $\btheta^{\star} \in \RR^d$, $0 < M \leq +\infty$ and $0 < \rho \leq L < +\infty$ such that
	\[
	\rho \bI \preceq \nabla^2 f_j (\btheta) \preceq L \bI , \qquad \forall \btheta \in B(\btheta^{\star}, M),~~\forall  j \in S.
	\]
	Take $\lambda_j = \lambda / \sqrt{w_j}$ for all $j \in [m]$ and some $\lambda > 0$.
	Define $\widetilde{\btheta}^{\mathrm{oracle}}  \in \argmin_{\btheta \in \RR^d}  \{  \sum_{j \in S} w_j f_j (\btheta)  \}$, $\kappa = L/\rho$, 
	\[
	\kappa_w = \frac{ (\max_{j \in S } \sqrt{w_j} ) \sum_{ j \in S } \sqrt{w_j} }{  \sum_{j \in S} w_j }
	\qquad\text{and}\qquad
	\varepsilon = \frac{ ( \max_{j \in S } \sqrt{w_j} ) \sum_{j \in S^c} \sqrt{w_j} }{   \sum_{j \in S} w_j }.
	\]
	When $\kappa \varepsilon < 1$ and
	\[
	\frac{ 3\kappa \kappa_w \max_{j \in S} \{ \sqrt{w_j}  \| \nabla f_j (\btheta^{\star}) \|_2 \} }{ 1 -  \kappa \varepsilon   } < \lambda < L M \min_{j \in S } \sqrt{w_j} ,
	\]
	we have $\widehat{\btheta}_j = \widehat\bbeta$ for all $j \in S$,
\begin{align*}
& \| \widehat{\bbeta} - \widetilde{\btheta}^{\mathrm{oracle}} \|_2 \leq \frac{ \sum_{j \in S^c} \sqrt{w_j} \lambda }{
	\sum_{j \in S} w_j \rho}
\leq \frac{\varepsilon \lambda}{ \rho \max_{j \in S } \sqrt{w_j} } , \\
& \| \widetilde{\btheta}^{\mathrm{oracle}} - \btheta^{\star} \|_2 \leq \frac{ \| \sum_{j \in S} w_j \nabla f_j (\btheta^{\star} ) \|_2 }{   \rho \sum_{j \in S} w_j }.
\end{align*}
\end{lemma}

We first use \Cref{thm-robustness} to derive the following intermediate result.

\begin{claim}\label{thm-armul-deterministic-claim}
	Define $g = \max_{j \in S} \{ \sqrt{w_j}  \| \nabla f_j (\btheta_j^{\star}) \|_2 \}$ and $\btheta^{\star} = \argmin_{\btheta \in \RR^d} \max_{j \in S}\| \btheta^{\star}_j - \btheta  \|_2$. Under the assumptions in \Cref{thm-armul-deterministic}, if 
	\[
	3 L \delta \leq g + \frac{1 - \kappa \varepsilon }{ 5 \kappa \kappa_w } \lambda,
	\]
	then
	\[
	\| \widehat{\btheta}_j  - \btheta^{\star}_j \|_2  
	\leq   
	\frac{ \| \sum_{k \in S} w_k \nabla f_k (\btheta^{\star}_k) \|_2 }{ \rho \sum_{k \in S} w_k } +    \frac{ 2 \kappa \kappa_w \delta }{ \sqrt{w_j} } +  \frac{   \lambda \sum_{k \in S^c} \sqrt{w_k} }{ \rho \sum_{k \in S} w_k  } , \qquad\forall j \in S.
	\]
\end{claim}
\begin{proof}[\bf Proof of Claim \ref{thm-armul-deterministic-claim}]
	We obtain from the assumption
	\begin{align}
	\frac{5 \kappa \kappa_w }{ 1 -  \kappa \varepsilon } \cdot  \max_{j \in S} \{ \sqrt{w_j} \| \nabla f_j (\btheta_j^{\star}) \|_2 \} < \lambda < \frac{ \rho M }{2} \cdot \min_{j \in S} \sqrt{w_j}
	\label{eqn-armul-deterministic-lambda}
	\end{align}
	that $g \leq \rho M \min_{j \in S} \sqrt{w_j} / 10 $, $\lambda < \rho M \min_{j \in S} \sqrt{w_j}/ 2$ and $g + \lambda < \frac{3}{5} LM \min_{j \in S} \sqrt{w_j}$. When
	\[
	3 L \delta \leq g + \frac{1 - \kappa \varepsilon }{ 5 \kappa \kappa_w } \lambda \leq g + \lambda,
	\]
	we have $ \delta < M \min_{j \in S} \sqrt{w_j} / 5$. Thus $\max_{j \in S}\| \btheta^{\star}_j  - \btheta^{\star} \|_2 \leq M/5$; $\rho \bI \preceq \nabla^2 f_j(\btheta) \preceq L \bI$ holds for all $j \in S$ and $\btheta \in B(\btheta^{\star}, 4M/5)$.

	For any $j \in S$, the regularity condition $ \nabla^2 f_j (\btheta) \preceq L \bI$, $\forall \btheta \in B(\btheta_j^{\star}, M)$ leads to $ \| \nabla f_j (\btheta^{\star}_j)  - \nabla f_j (\btheta^{\star})  \|_2 \leq L \delta / \sqrt{w_j}$. By triangle's inequality,
	\begin{align*}
	\max_{j \in S} \{ \sqrt{w_j} \| \nabla f_j (\btheta^{\star}) \|_2 \}
	&\leq \max_{j \in S} \{ \sqrt{w_j}  \| \nabla f_j (\btheta^{\star}_j) \|_2 \} + \max_{j \in S} \{ \sqrt{w_j} \| \nabla f_j (\btheta^{\star}_j)  - \nabla f_j (\btheta^{\star})  \|_2 \} \\
	&	\leq g + L \delta
	<  \frac{4g}{3} + \frac{1 - \kappa \varepsilon }{ 15 \kappa \kappa_w } \lambda.
	\end{align*}
	Consequently,
	\begin{align*}
	\frac{ 3\kappa \kappa_w \max_{j \in S}  \{ \sqrt{w_j} \| \nabla f_j (\btheta^{\star}) \|_2 \}
	}{ 1 -  \kappa \varepsilon   } 
	<  \frac{ 4 \kappa \kappa_w g }{ 1 -  \kappa \varepsilon  } + \frac{\lambda }{ 5 } 
	\overset{\mathrm{(i)}}{<} \lambda \overset{\mathrm{(ii)}}{<} \frac{ \rho M }{2} \min_{j \in S} \sqrt{w_j} < L \cdot \frac{4M}{5} \min_{j \in S} \sqrt{w_j}.
	\end{align*}
	The inequalities $\mathrm{(i)}$ and $\mathrm{(ii)}$ follow from \Cref{eqn-armul-deterministic-lambda}.

	Based on the above, \Cref{thm-robustness} asserts that $\widehat{\btheta}_j = \widehat\bbeta$ for all $j \in S$ and
	\[
	\| \widehat{\bbeta}  - \btheta^{\star} \|_2 \leq \frac{ \| \sum_{j \in S} w_j \nabla f_j (\btheta^{\star}) \|_2 }{  \rho \sum_{j \in S} w_j } +  \frac{   \lambda \sum_{j \in S^c} \sqrt{w_j} }{ \rho \sum_{j \in S} w_j  } .
	\]
	For $j \in S$, $ \| \widehat{\btheta}_j  - \btheta^{\star}_j \|_2 \leq \| \widehat{\bbeta}  - \btheta^{\star} \|_2 +  \| \btheta^{\star}  - \btheta^{\star}_j \|_2 \leq \| \widehat{\bbeta}  - \btheta^{\star} \|_2 + \delta / \sqrt{w_j} $. Also,
	\begin{align*}
	& \frac{ \| \sum_{j \in S} w_j \nabla f_j (\btheta^{\star}) \|_2 }{  \sum_{j \in S} w_j }
	\leq \frac{ \| \sum_{j \in S} w_j \nabla f_j (\btheta^{\star}_j) \|_2 }{  \sum_{j \in S} w_j } + L \delta \frac{ \sum_{j \in S} \sqrt{w_j}  }{  \sum_{j \in S} w_j } .
	\end{align*}
	Based on the above estimates,
	\begin{align*}
	\| \widehat{\btheta}_j  - \btheta^{\star}_j \|_2 & \leq \frac{ 1 }{  \rho }
	\bigg(
	\frac{ \| \sum_{k \in S} w_k \nabla f_k (\btheta^{\star}_k) \|_2 }{  \sum_{k \in S} w_k } + L \delta \frac{ \sum_{k \in S} \sqrt{w_k}  }{  \sum_{k \in S} w_k } 
	\bigg)
	+  \frac{   \lambda \sum_{k \in S^c} \sqrt{w_k} }{ \rho \sum_{k \in S} w_k  } + \frac{ \delta }{\sqrt{w_j}}
	\\&
	\leq  
	\frac{ \| \sum_{k \in S} w_k \nabla f_k (\btheta^{\star}_k) \|_2 }{ \rho \sum_{k \in S} w_k } +    \frac{ 2 \kappa \kappa_w \delta }{ \sqrt{w_j} } +  \frac{   \lambda \sum_{k \in S^c} \sqrt{w_k} }{ \rho \sum_{k \in S} w_k  } , \qquad\forall j \in S
	.
	\end{align*}
\end{proof}

We now come back to \Cref{thm-armul-deterministic}. The condition \eqref{eqn-armul-deterministic-lambda} forces $\lambda > 5 \kappa \kappa_w g$.
Claim \ref{thm-armul-deterministic-claim} implies that when 
$3 L \delta \leq g + \frac{1 - \kappa \varepsilon }{ 5 \kappa \kappa_w } \lambda$,
\begin{align*}
\| \widehat{\btheta}_j  - \btheta^{\star}_j \|_2  
& \leq   
\frac{ \| \sum_{k \in S} w_k \nabla f_k (\btheta^{\star}_k) \|_2 }{ \rho \sum_{k \in S} w_k } +  
\frac{\kappa_w }{ \rho \sqrt{w_j} }
\min \bigg\{
3 L \delta,~ g + \frac{1 - \kappa \varepsilon }{ 5 \kappa \kappa_w } \lambda
\bigg\} \\
&~~~~ +  \frac{   \lambda \sum_{k \in S^c} \sqrt{w_k} }{ \rho \sum_{k \in S} w_k  } , \qquad\forall j \in S.
\end{align*}
On the other hand, when $3 L \delta > g + \frac{1 - \kappa \varepsilon }{ 5 \kappa \kappa_w } \lambda$, we use \Cref{thm-armul-personalization} to get 
\[
\max_{j \in S} \| \widehat{\btheta}_j  - \btheta^{\star}_j \|_2  \leq \frac{g + \lambda}{\rho \sqrt{w_j} } \leq \bigg( \frac{1}{5\kappa \kappa_w } + 1 \bigg)\frac{\lambda}{\rho \sqrt{w_j} } \leq \frac{6 \lambda}{5 \rho \sqrt{w_j}} .
\]
On top of the above, for any $j \in S$ we have
\begin{align*}
\| \widehat{\btheta}_j  - \btheta^{\star}_j \|_2  
& \leq \frac{ \| \sum_{k \in S} w_k \nabla f_k (\btheta^{\star}_k) \|_2 }{ \rho \sum_{k \in S} w_k }
+  \frac{\kappa_w}{\rho \sqrt{w_j} } \cdot \frac{6 \kappa  }{1 - \kappa \varepsilon  } \min\bigg\{ 3 L \delta ,~ g +  \frac{1 - \kappa \varepsilon }{ 5 \kappa \kappa_w } \lambda  \bigg\}  +    \frac{   \lambda \sum_{k \in S^c} \sqrt{w_k} }{ \rho \sum_{k \in S} w_k  } \\
& \leq \frac{ \| \sum_{k \in S} w_k \nabla f_k (\btheta^{\star}_k) \|_2 }{ \rho \sum_{k \in S} w_k }
+  \frac{\kappa_w}{\rho \sqrt{w_j} } \cdot \frac{6 \kappa}{1 - \kappa \varepsilon } \min\bigg\{ 3 L \delta ,~    \frac{ 2 \lambda }{ 5 \kappa \kappa_w }   \bigg\}  +   \frac{   \lambda \sum_{k \in S^c} \sqrt{w_k} }{ \rho \sum_{k \in S} w_k  } \\
& \leq \frac{ \| \sum_{k \in S} w_k \nabla f_k (\btheta^{\star}_k) \|_2 }{ \rho \sum_{k \in S} w_k }
+     \frac{6  }{ ( 1 -   \kappa \varepsilon ) \sqrt{w_j} } \min\bigg\{ 3  \kappa^2 \kappa_w \delta ,~    \frac{ 2 \lambda }{ 5 \rho }   \bigg\}  +    \frac{  \varepsilon \lambda  }{ \rho \max_{k \in S} \sqrt{w_k } }.
\end{align*}

The relation between $\widehat{\btheta}_j$ and $\argmin_{\btheta \in \RR^d} \{ \sum_{j \in [m]} w_j f_j (\btheta) \}$ can be derived from \Cref{thm-robustness}.

\subsection{Proof of \Cref{thm-robustness}}\label{sec-thm-robustness-proof}

We will prove stronger results for a weighted version
\begin{align}
(\widehat{\btheta}_1,\cdots,\widehat{\btheta}_m,\widehat{\bbeta}) \in \argmin_{ \btheta_1,\cdots,\btheta_m,\bbeta \in \RR^d } \bigg\{ 
\sum_{j=1}^{m} w_j [ f_j ( \btheta_j ) + \lambda_j \|  \bbeta - \btheta_j  \|_2 ]
\bigg\} 
\label{eqn-armul-weighted-appendix}
\end{align}
with general $w_j \geq 0$ and $\lambda_j \geq 0$, and then get \Cref{thm-robustness} as a corollary.

\noindent{\bf Step 1.} We first work on the no-outlier case $S = [m]$. Let $\widetilde{\btheta} \in \argmin_{\btheta \in \RR^d} \sum_{j=1}^{m} w_j f_j(\btheta)$. 

\begin{lemma}\label{thm-consensus}
	Let $\{ f_j \}_{j=1}^m$ be convex and differentiable. Define $F = \sum_{j=1}^{m} w_j f_j \square (\lambda_j \| \cdot \|_2) $ and $G = \sum_{j=1}^{m} w_j f_j$. Suppose there exist $\btheta^{\star} \in \RR^d$, $0 < M \leq +\infty$ and $0 < \rho_0, L_1,\cdots, L_m < + \infty$ such that
	\begin{align*}
	\nabla^2 f_j (\btheta) \preceq L_j \bI,~~ \forall j \in [m] \qquad\text{and}\qquad
	\nabla^2 G (\btheta) \succeq \rho_0 \bI
	\end{align*}
	hold for all $ \btheta \in B(\btheta^{\star}, M )$. If
	\begin{align*}
	&\| \nabla f_j (\btheta^{\star}) \|_{2} + \frac{2 L_j \| \sum_{k = 1}^{m} w_k \nabla f_k (\btheta^{\star}) \|_2 }{  \rho_0 } < \lambda_j < \| \nabla f_j (\btheta^{\star}) \|_{2} + L_j M, \qquad \forall j \in [m] ,
	\end{align*}
	then $\widehat{\btheta}_1 = \cdots = \widehat{\btheta}_m = \widehat{\bbeta} = \widetilde{\btheta} $,
	\begin{align*}
	&  F (\btheta) =  G (\btheta) , \qquad \forall  \btheta  \in B \bigg( \btheta^{\star} , \min_{j \in [m]} \{ ( \lambda_j - \| \nabla f_j (\btheta^{\star}) \|_{2}  ) / L_j \} \bigg) ,\\
	&\| \widehat{\bbeta} - \btheta^{\star} \|_2 \leq  \frac{  \| \sum_{j = 1}^{m} w_j \nabla f_j  (\btheta^{\star}) \|_2 }{  \rho_0 }.
	\end{align*}
\end{lemma}

\begin{proof}[\bf Proof of \Cref{thm-consensus}]
	By \Cref{lem-inf-conv}, we have $f_j = f_j \square (\lambda_j \| \cdot \|_2)$ in $B(\btheta^{\star},    ( \lambda_j - \| \nabla f_j (\btheta^{\star}) \|_{2} ) / L_j )$. Then $F = G$ in $B(\btheta^{\star}, R)$ with $R = \min_{j \in [m]} \{ ( \lambda_j - \| \nabla f_j (\btheta^{\star}) \|_{2} ) / L_j \}$.
	
	Note that $\nabla^2 G (\btheta) \succeq \rho_0  \bI$, $\forall  \btheta \in B(\btheta^{\star}, M )$ and $M \geq R$. Then, $\nabla^2 F = \nabla^2 G   \succeq  \rho_0 \bI$ in $B(\btheta^{\star}, R )$. By the assumption $ \|  \nabla G  (\btheta^{\star}) \|_2 < \frac{1}{2} R \rho_0$, the fact $\|  \nabla F  (\btheta^{\star}) \|_2 = \|  \nabla G  (\btheta^{\star}) \|_2$ and the first part of \Cref{lem-consensus-min}, $\argmin_{\btheta} F (\btheta) \in B ( \btheta^{\star},  R )$. This bound forces $\widehat{\bbeta} = \argmin_{\bbeta \in \RR^d} F (\bbeta) = \argmin_{\btheta} G(\btheta) = \widetilde{\btheta}$. We can easily control $\| \widetilde{\btheta} - \btheta^{\star} \|_2$ using the second part of \Cref{lem-consensus-min} to $G$.

	Finally, $\widehat\btheta_j = \widehat\bbeta$ follows from $\widehat\btheta_j \in \argmin_{\btheta \in \RR^d} \{ f_j(\btheta) + \lambda_j \| \widehat{\bbeta} - \btheta \|_2 \}$, $\widehat{\btheta}_j \in B(\btheta^{\star}, R)$ and \Cref{lem-inf-conv}.
\end{proof}

\noindent{\bf Step 2.} We are now ready to include outliers and prove a stronger version of \Cref{thm-robustness}.

\begin{lemma}[Robustness]\label{thm-robustness-general}
	Let $\{ f_j \}_{j=1}^m$ be convex functions from $\RR^d$ to $\RR$. Suppose there exist $S \subseteq [m]$, $\btheta^{\star} \in \RR^d$ and $0 < M \leq +\infty$ such that for all $j \in S$, we have 
	\[
	\rho_j \bI \preceq \nabla^2 f_j (\btheta) \preceq L_j \bI , \qquad \forall \btheta \in B(\btheta^{\star}, M)
	\]
	with some $0 < \rho_j \leq L_j < +\infty$, and
	\begin{align*}
	\| \nabla f_j (\btheta^{\star}) \|_{2} + \frac{  L_j ( 2 \| \sum_{k \in S} w_k \nabla f_k  (\btheta^{\star}) \|_2 
		+ \sum_{k \in S^c} w_k\lambda_k
		)
	}{ \sum_{k \in S } w_k\rho_k }
	& < \lambda_j \\
	& \leq \| \nabla f_j (\btheta^{\star}) \|_{2} + L_j M
	, ~~ \forall j \in S .
	\end{align*}
	Define $G  =  \sum_{j \in S} f_j \square (\lambda_j \| \cdot \|_2)$. Then, the function $G$ has a unique minimizer $\widehat{\btheta}_S$ and it satisfies
	\[
	\| \widehat{\btheta}_S - \btheta^{\star} \|_2 \leq \frac{ \| \sum_{j \in S} w_j \nabla f_j (\btheta^{\star}) \|_2 }{
		\sum_{j \in S} w_j \rho_j
	} .
	\]
	Moreover, we have $\widehat{\btheta}_j = \widehat\bbeta$ for $j \in S$ and
	\[
	\| \widehat{\bbeta} - \widehat{\btheta}_S \|_2 \leq  \frac{ \sum_{j \in S^c} w_j \lambda_j }{
		\sum_{j \in S} w_j \rho_j
	} .
	\]
\end{lemma}

\begin{proof}[\bf Proof of \Cref{thm-robustness-general}]
	
	Let $r_j = ( \lambda_j - \| \nabla f_j (\btheta^{\star}) \|_{2} ) / L_j $ and $\rho_0 = \sum_{j \in S } w_j \rho_j$. We have
	\begin{align}
	r_j > \frac{2  \| \sum_{k \in S} w_k \nabla f_k  (\btheta^{\star}) \|_2 
		+ \sum_{k \in S^c} w_k \lambda_k 
	}{ \rho_0 }, \qquad\forall j \in S.
	\label{eqn-thm-robustness-general-0}
	\end{align}
	
	Let $F  = \sum_{j=1}^{m} w_j f_j \square (\lambda_j \| \cdot \|_2) $. \Cref{thm-consensus} applied to $\{ f_j \}_{j \in S}$ yields
	\begin{align}
	&  G   (\btheta) =  \sum_{j \in S} w_j f_j (\btheta) , \qquad \forall \btheta \in B \bigg(\btheta^{\star},  \min_{j \in S} r_j \bigg) ; \notag \\
	& \widehat{\btheta}_S = \argmin_{\btheta \in \RR^d} G   (\btheta) = \argmin_{\btheta \in \RR^d}  \bigg \{  \sum_{j\in S} w_j f_j (\btheta) \bigg \}\in B \bigg( \btheta^{\star},   \frac{  \| \sum_{k \in S} w_k \nabla f_k  (\btheta^{\star}) \|_2 }{ \rho_0 } \bigg).
	\label{eqn-thm-robustness-general-1}
	\end{align}
	As a result, 
	\begin{align}
	\nabla^2 G (\btheta) = \sum_{j \in S} w_j \nabla^2 f_j (\btheta) \succeq \rho_0 \bI, \qquad \forall \btheta \in B (\btheta^{\star},  \min_{j \in S}  r_j ).
	\label{eqn-thm-robustness-general}
	\end{align}
	By \Cref{lem-consensus-min}, $\| \widehat{\btheta}_S - \btheta^{\star}\|_2 \leq   \| \sum_{j \in S} w_j \nabla f_j  (\btheta^{\star}) \|_2 / \rho_0$. 
	
	By \Cref{eqn-thm-robustness-general-0}, 
	\[
	\frac{ \sum_{k \in S^c} w_k \lambda_k
	}{ \rho_0 } 
	< \min_{j \in S} r_j -  \frac{ 2 \| \sum_{k \in S} w_k \nabla f_k  (\btheta^{\star}) \|_2 }{ \rho_0 }.
	\]
	Let $r$ denote the left-hand side above. In light of \Cref{eqn-thm-robustness-general-1} and \Cref{eqn-thm-robustness-general}, $G$ is $\rho_0$-strongly convex in $B( \widehat{\btheta}_S, r ) $. Note that $H = F - G =  \sum_{j \in S^c} w_j f_j \square (\lambda_j \| \cdot \|_2)$ is convex and $\sum_{j \in S^c} w_j \lambda_j$-Lipschitz according to \Cref{lem-inf-conv-lip}. Applying \Cref{lem-cvx-reg} to $F = G + H$ yields $\| \widehat{\bbeta} - \widehat{\btheta}_S \|_2 \leq   \sum_{j \in S^c} w_j \lambda_j  /
	\rho_0 $.

	As a result,
	\[
	\| \widehat{\bbeta} - \btheta^{\star}\|_2 \leq \| \widehat{\bbeta} - \widehat{\btheta}_S \|_2 + \| \widehat{\btheta}_S - \btheta^{\star}\|_2 \leq \frac{  \| \sum_{j \in S} w_j \nabla f_j  (\btheta^{\star}) \|_2 + \sum_{j \in S^c} w_j \lambda_j  }{ \rho_0 }.
	\]
	For any $j \in S$, we have $\widehat\btheta_j \in \argmin_{\btheta \in \RR^d} \{ f_j(\btheta) + \lambda_j \| \widehat{\bbeta} - \btheta \|_2 \}$ and
	\[
	M \geq  ( \lambda_j - \| \nabla f_j (\btheta^{\star}) \|_{2} ) / L_j   
	> \frac{2  \| \sum_{j \in S} w_j \nabla f_j  (\btheta^{\star}) \|_2 
		+ \sum_{j \in S^c} w_j \lambda_j
	}{ \rho_0 } \geq \| \widehat{\bbeta} - \btheta^{\star}\|_2.
	\]
	The desired result $\widehat\btheta_j = \widehat\bbeta$ follows from \Cref{lem-inf-conv}.
\end{proof}

\noindent{\bf Step 3.} We now come back to \Cref{thm-robustness}. Define $\eta =  \max_{j \in S} \{ \sqrt{w_j} \| \nabla f_j (\btheta^{\star}) \|_2 \}$.
From the assumption $\lambda > \frac{3 \kappa \kappa_w \eta }{1 - \kappa \varepsilon}$
we get
$\lambda > 3 \kappa \kappa_w \eta +  \kappa \varepsilon \lambda   $ and for all $j \in S$,
\begin{align*}
& \| \nabla f_j (\btheta^{\star}) \|_{2} + \frac{  L ( 2 \| \sum_{k \in S} w_k \nabla f_k  (\btheta^{\star}) \|_2 
	+ \sum_{k \in S^c} w_k\lambda_k
	)
}{ \sum_{k \in S } w_k\rho } \\
& \leq \frac{\eta}{\sqrt{w_j}} + \frac{  L ( 2 \sum_{k \in S} \sqrt{w_k} \eta 
	+ \sum_{k \in S^c} \sqrt{w_k} \lambda
	)
}{ \sum_{k \in S } w_k\rho } \\
& = \frac{\eta}{\sqrt{w_j}} \bigg( 1 + 2 \kappa  \cdot 
\underbrace{
	\frac{ \sqrt{w_j} \sum_{k \in S} \sqrt{w_k} }{ \sum_{k \in S } w_k } 
}_{\leq \kappa_w}
\bigg)
+ \frac{ \kappa \lambda  }{\sqrt{w_j}} \cdot 
\underbrace{
	\frac{ \sqrt{w_j} \sum_{k \in S^c} \sqrt{w_k}  
	}{ \sum_{k \in S } w_k  } 
}_{\leq \varepsilon}
\\
& \leq 3 \kappa  \kappa_w \frac{\eta}{\sqrt{w_j}}  + \frac{ \varepsilon \kappa \lambda }{\sqrt{w_j}} 
< \frac{ \lambda }{\sqrt{w_j}} = \lambda_j.
\end{align*}
Finally, the proof is finished by \Cref{thm-robustness-general}.

\subsection{Proof of \Cref{thm-armul-clustered-deterministic-0}}\label{sec-thm-armul-clustered-deterministic-0-proof}

We invoke the following lemma, whose proof is in \Cref{sec-thm-adaptivity-clustered-proof}.

\begin{lemma}\label{thm-adaptivity-clustered}
	Define $T(k) =  \{ j \in [m] :~ z_j^{\star} = k \} $, $\widehat{\bbeta}_k^{\mathrm{oracle}} \in \argmin_{\bbeta \in \RR^d}  \sum_{j \in T(k)}  f_j ( \bbeta )  $ for $k \in [K]$ and $m_{\min} = \min_{k \in[K]} |T(k)| $. 
	Suppose there exist $0 < \rho, L ,\eta < +\infty$ and $0 < M \leq +\infty$ such that 
	\begin{align*}
	&\rho \bI \preceq \nabla^2 f_j (\btheta) \preceq L \bI , ~~ \forall \btheta \in B( \bbeta^{\star}_{z_j^{\star}}  , M)
	\qquad\text{and}\qquad
	\| \nabla f_j ( \bbeta^{\star}_{z_j^{\star}}  ) \|_2 \leq \eta
	\end{align*}
	hold for all $j \in [m] $. Additionally, assume that
	\begin{align*}
	&\min_{k \neq l}
	\| \bbeta^{\star}_k - \bbeta^{\star}_l \|_2 > \frac{2 \eta}{\rho} \bigg( 1 + 4 \sqrt{ \frac{K m L}{m_{\min} \rho } }  \bigg) 
	\qquad\text{and}\qquad
	\frac{7 \eta L}{\rho} \sqrt{ \frac{K m L}{m_{\min} \rho } } < \lambda < LM .
	\end{align*}
	Consider the solution $(\widehat{\bTheta}, \widehat{\bB},\widehat{\bz})$ in \eqref{eqn-obj-clustered-0}.
	There exists a permutation $\tau$ of $[K]$ such that
	\begin{align*}
	&  \widehat{z}_j = \tau (z_j^{\star}) \qquad\text{and}\qquad \widehat{\btheta}_j = \widehat\bbeta_{\widehat{z}_j },\qquad\forall  j \in [m] ,\\
	& \widehat{\bbeta}_{\tau(k) } = \widehat{\bbeta}_{ k }^{\mathrm{oracle}} 
	\in B \bigg( \bbeta^{\star}_{k} ,~ \frac{1}{\rho |T(k)| } \bigg\|  \sum_{j \in T(k) } \nabla f_j ( \bbeta^{\star}_{ k } ) \bigg\|_2
	\bigg)
	, \qquad \forall k \in [K].
	\end{align*}
\end{lemma}

We first use \Cref{thm-adaptivity-clustered} to derive the following intermediate result.

\begin{claim}\label{thm-armul-clustered-deterministic-0-claim}
	Define $g = \max_{j \in [m]} \| \nabla f_j (\btheta_j^{\star}) \|_2 $. Under the assumptions in \Cref{thm-armul-clustered-deterministic-0}, if 
	\[
	7 L \delta \leq g + \frac{\lambda }{ 10 \sqrt{\kappa \kappa_m} K  } ,
	\]
	then, there exists a permutation $\tau$ of $[K]$ such that
	\begin{align*}
	&  \widehat{z}_j = \tau (z_j^{\star}) \qquad\text{and}\qquad \widehat{\btheta}_j = \widehat\bbeta_{\widehat{z}_j },\qquad\forall  j \in [m] ,\\
	& \widehat{\bbeta}_{\tau(k) } = \argmin_{\bbeta \in \RR^d}  \sum_{j \in T(k)}  f_j ( \bbeta ) 
	, \qquad \forall k \in [K] , \\
	&	\| \widehat{\btheta}_j  - \btheta^{\star}_j \|_2  
	\leq  
	\frac{ 1 }{  \rho |T( z^{\star}_j )| } \bigg\| \sum_{j \in T( z^{\star}_j ) }   \nabla f_j ( \btheta^{\star}_{ j }  ) \bigg\|_2
	+  2 \kappa \delta , \qquad\forall j \in [m] .
	\end{align*}
\end{claim}
\begin{proof}[\bf Proof of Claim \ref{thm-armul-clustered-deterministic-0-claim}]
	We obtain from the assumption
	\begin{align}
	10 \sqrt{\kappa  \kappa_m} K \max_{j \in [m]}   \| \nabla f_j (\btheta_j^{\star}) \|_2   < \lambda < \frac{\rho}{2}   \min \Big\{ M,~ \min_{k \neq l}
	\| \bbeta^{\star}_k - \bbeta^{\star}_l \|_2 \Big\}
	\label{eqn-armul-deterministic-clustered-lambda}
	\end{align}
	that $g \leq \rho M   / 10 $, $\lambda < \rho M $ and $g + \lambda < \frac{11}{10} LM  $. When
	\[
	7 L \delta \leq g + \frac{\lambda }{ 10 \sqrt{\kappa \kappa_m} K  }  \leq g + \lambda,
	\]
	we have $ \delta < 11M/70 < M / 5$. Thus $\max_{j \in [m]} \| \btheta^{\star}_j  - \bbeta^{\star}_{z_j^{\star}} \|_2 \leq M/5$; $\rho \bI \preceq \nabla^2 f_j(\btheta) \preceq L \bI$ holds for all $j \in [m]$ and $\btheta \in B( \bbeta^{\star}_{z_j^{\star}}, 4M/5)$.

	For any $j$, the regularity condition $ \nabla^2 f_j (\btheta) \preceq L \bI$, $\forall \btheta \in B(\btheta_j^{\star}, M)$ leads to $ \| \nabla f_j (\btheta^{\star}_j)  - \nabla f_j (\bbeta^{\star}_{z_j^{\star}})  \|_2 \leq L \delta $. By triangle's inequality,
	\begin{align*}
	\| \nabla f_j (\bbeta^{\star}_{z_j^{\star}}) \|_2 
	&\leq   \| \nabla f_j (\btheta^{\star}_j) \|_2  +   \| \nabla f_j (\bbeta^{\star}_{z_j^{\star}})  - \nabla f_j (\btheta^{\star}_j)  \|_2   	\leq g + L \delta
	<  \frac{8g}{7} + \frac{ \lambda }{ 70 \sqrt{\kappa  \kappa_m} K  }  .
	\end{align*}
	Consequently,
	\begin{align*}
	7 \sqrt{\kappa  \kappa_m} K \max_{j \in [m]}  \| \nabla f_j (\bbeta^{\star}_{z_j^{\star}}) \|_2 
	< 8  \sqrt{\kappa \kappa_m} K g + \frac{ \lambda }{ 10 } 
	\overset{\mathrm{(i)}}{<}  \lambda  \overset{\mathrm{(ii)}}{<}
	\rho \cdot \min \Big\{ 4 M / 5,~ \min_{k \neq l}
	\| \bbeta^{\star}_k - \bbeta^{\star}_l \|_2 / 2 \Big\}.
	\end{align*}
	The inequalities $\mathrm{(i)}$ and $\mathrm{(ii)}$ follow from \Cref{eqn-armul-deterministic-clustered-lambda}.
	We have $\lambda < L \cdot \frac{4M}{5}$ and
	\[
	\min_{k \neq l}
	\| \bbeta^{\star}_k - \bbeta^{\star}_l \|_2   > \frac{2 \lambda}{ \rho} > \frac{ 14 \sqrt{\kappa  \kappa_m} K }{\rho} \max_{j \in [m]}  \| \nabla f_j (\bbeta^{\star}_{z_j^{\star}}) \|_2 .
	\]
	Based on the above estimates and $\sqrt{\kappa  \kappa_m} K = \sqrt{\frac{K m L}{m_{\min} \rho }}$, \Cref{thm-adaptivity-clustered} asserts the existence of a permutation $\tau$ of $[K]$ such that
	\begin{align*}
	&  \widehat{z}_j = \tau (z_j^{\star}) \qquad\text{and}\qquad \widehat{\btheta}_j = \widehat\bbeta_{\widehat{z}_j },\qquad\forall  j \in [m] ,\\
	& \widehat{\bbeta}_{\tau(k) } = \argmin_{\bbeta \in \RR^d}  \sum_{j \in T(k)}  f_j ( \bbeta ) 
	, \qquad \forall k \in [K] , \\
	& \| \widehat\bbeta_{\widehat{z}_j}  - \bbeta^{\star}_{ z^{\star}_j } \|_2 \leq \frac{ \| \sum_{j \in T( z^{\star}_j ) }   \nabla f_j (  \bbeta^{\star}_{ z^{\star}_j } ) \|_2 }{  \rho |T( z^{\star}_j )| } , \qquad \forall  j \in [m].
	\end{align*}
	By the triangle's inequality, $ \| \widehat{\btheta}_j  - \btheta^{\star}_j \|_2 \leq \| \widehat{\btheta}_j  - \bbeta^{\star}_{ z^{\star}_j } \|_2 +  \| \bbeta^{\star}_{ z^{\star}_j } - \btheta^{\star}_j \|_2 \leq \| \widehat\bbeta_{\widehat{z}_j}  - \bbeta^{\star}_{ z^{\star}_j } \|_2  + \delta $. Also,
	\begin{align*}
	& \frac{ \| \sum_{j \in T( z^{\star}_j ) }   \nabla f_j (  \bbeta^{\star}_{ z^{\star}_j } ) \|_2 }{ |T( z^{\star}_j )| } 
	\leq \frac{ \| \sum_{j \in T( z^{\star}_j ) }   \nabla f_j (  \btheta^{\star}_{ j } ) \|_2 }{ |T( z^{\star}_j )| }  + L \delta  .
	\end{align*}
	Based on the above estimates,
	\begin{align*}
	\| \widehat{\btheta}_j  - \btheta^{\star}_j \|_2 & \leq \frac{ 1 }{  \rho }
	\bigg(
	\frac{ \| \sum_{j \in T( z^{\star}_j ) }   \nabla f_j (  \btheta^{\star}_{ j } ) \|_2 }{ |T( z^{\star}_j )| }  + L \delta 
	\bigg) + \delta \\&
	\leq  
	\frac{ \| \sum_{j \in T( z^{\star}_j ) }   \nabla f_j (  \btheta^{\star}_{ j }  ) \|_2 }{  \rho |T( z^{\star}_j )| } +  2 \kappa \delta , \qquad\forall j \in [m] .
	\end{align*}
\end{proof}

We now come back to \Cref{thm-armul-clustered-deterministic-0}. The condition \eqref{eqn-armul-deterministic-clustered-lambda} forces $\lambda > 10 \sqrt{\kappa  \kappa_m} K  g$.
Claim \ref{thm-armul-clustered-deterministic-0-claim} implies that when 
$7 L \delta \leq g + \frac{\lambda }{ 10 \sqrt{\kappa  \kappa_m} K  } $,
\[
\| \widehat{\btheta}_j  - \btheta^{\star}_j \|_2 \leq \frac{ \| \sum_{j \in T( z^{\star}_j ) }   \nabla f_j (  \btheta^{\star}_{ j }  ) \|_2 }{  \rho |T( z^{\star}_j )| } +  \frac{2}{\rho} \min\bigg\{ L \delta ,~ \frac{g}{7} + \frac{\lambda }{ 70 \sqrt{\kappa  \kappa_m} K  } \bigg\} , \qquad \forall j \in [m].
\]
On the other hand, when $7 L \delta > g + \frac{\lambda }{ 10 \sqrt{\kappa \kappa_m} K  } $, we use \Cref{thm-armul-personalization} to get 
\[
\max_{j \in [m]} \| \widehat{\btheta}_j  - \btheta^{\star}_j \|_2  \leq \frac{g + \lambda}{\rho } \leq \bigg( \frac{1}{ 10 \sqrt{\kappa \kappa_m} K  } + 1 \bigg)\frac{\lambda}{\rho  } \leq \frac{11 \lambda}{10  \rho  } .
\]
On top of the above, for any $j \in [m]$ we have
\begin{align*}
\| \widehat{\btheta}_j  - \btheta^{\star}_j \|_2  
& \leq \frac{ \| \sum_{j \in T( z^{\star}_j ) }   \nabla f_j (  \btheta^{\star}_{ j }  ) \|_2 }{  \rho |T( z^{\star}_j )| }
+   \frac{77\sqrt{\kappa \kappa_m} K}{\rho} \min\bigg\{ L \delta ,~ \frac{g}{7} + \frac{\lambda }{ 70 \sqrt{\kappa^3 \kappa_m} K  } \bigg\}   \\
& \leq \frac{ \| \sum_{j \in T( z^{\star}_j ) }   \nabla f_j (  \btheta^{\star}_{ j }  ) \|_2 }{  \rho |T( z^{\star}_j )| }
+   \frac{77\sqrt{\kappa \kappa_m} K}{\rho} \min\bigg\{   L \delta ,~   \frac{\lambda }{ 35 \sqrt{\kappa \kappa_m} K  }   \bigg\}   \\
& \leq \frac{ \| \sum_{j \in T( z^{\star}_j ) }   \nabla f_j (  \btheta^{\star}_{ j }  ) \|_2 }{  \rho |T( z^{\star}_j )| }
+     \min\bigg\{ 77 \sqrt{\kappa^3 \kappa_m} K \delta ,~    \frac{ 11 \lambda }{ 5 \rho }   \bigg\}  .
\end{align*}

Finally, if $\delta \leq \lambda / ( 70 \sqrt{\kappa \kappa_m} K L )$, then we use Claim \ref{thm-armul-clustered-deterministic-0-claim} to get a permutation $\tau$ of $[K]$ such that
\begin{align*}
&  \widehat{z}_j = \tau (z_j^{\star}) \qquad\text{and}\qquad \widehat{\btheta}_j = \widehat\bbeta_{\widehat{z}_j },\qquad\forall  j \in [m] ,\\
& \widehat{\bbeta}_{\tau(k) } = \argmin_{\bbeta \in \RR^d}  \sum_{j \in T(k)}  f_j ( \bbeta ) 
, \qquad \forall k \in [K] .
\end{align*}

\subsection{Proof of \Cref{thm-adaptivity-clustered}}\label{sec-thm-adaptivity-clustered-proof}

Define
\begin{align*}
F(\bTheta, \bB, \bz ) = \sum_{j=1}^{m} [ f_j ( \btheta_j ) + \lambda \| \bbeta_{z_j} - \btheta_j \|_2 ],\qquad
\forall \bTheta \in \RR^{d \times m},~~\bB\in \RR^{d\times K}, ~~ \bz \in [K]^m.
\end{align*}
Then, $(\widehat{\bTheta} , \widehat{\bB}, \widehat{\bz}) $ is a minimizer of $F$. Define $\bB^{\star} = (\bbeta^{\star}_1,\cdots,\bbeta^{\star}_K)$, $\widehat{\bB}^{\mathrm{oracle}} = (\widehat{\bbeta}_1^{\mathrm{oracle}} , \cdots, \widehat{\bbeta}_K^{\mathrm{oracle}})$, $\widetilde{f}_j = f_j \square (\lambda \| \cdot \|_2)$ and $\widetilde{\bbeta}_j = \argmin_{\bbeta} \widetilde{f}_j(\bbeta)$. With slight abuse of notation, let
\[
F(\bB, \bz) = \inf_{\bTheta \in \RR^{d\times m}} F(\bTheta, \bB, \bz) = \sum_{j=1}^{m} \widetilde{f}_j ( \bbeta_{z_j} ) , \qquad \forall \bB\in \RR^{d \times K},~~ \bz \in [K]^m.
\]
A key fact is
\begin{align}
F(\widehat{\bB} , \widehat{\bz}) = F(\widehat{\bTheta}, \widehat\bB, \widehat\bz) \leq
\inf_{\bTheta \in \RR^{d\times m}}
F(\bTheta, \bB^{\star}, \bz^{\star})
= F(\bB^{\star}, \bz^{\star}).
\label{eqn-clustered-0}
\end{align}

We will invoke \Cref{lem-clustered-lower-bound} to analyze $(\widehat{\bTheta}, \widehat{\bB},  \widehat{\bz})$. Let $r = 4 \sqrt{\frac{KmL}{m_{\min} \rho}} \frac{\eta}{\rho}$. It is easily seen that
\[
3 \eta L / \rho + Lr <  \lambda < LM.
\]
Define $S_{kl} = | \{ j \in [m] :~ z_j^{\star} = k,~ \widehat{z}_j = l \}|$, $\tau(k) = \argmax_{l \in [K]} S_{k l}$ for $k, l \in [K]$,
\begin{align*}
& H(x) = \begin{cases}
x^2/2, &\mbox{ if } 0 \leq x\leq r \\
r(x-r/2), &\mbox{ if } x > r 
\end{cases}, \\
&E  = \sum_{j =1}^m H \Big(
(\| \widehat\bbeta_{\widehat z_j} - \bbeta^{\star}_{z^{\star}_j} \|_2 -  \eta/\rho
)_{+} 
\Big) .
\end{align*}
By \Cref{lem-clustered-lower-bound}, we have $ \min_{k \in [K]} S_{k \tau(k)} \geq m_{\min} / K$,
\begin{align}
&  \widetilde{f}_j(\btheta) =  f_j (\btheta) , \qquad \forall \btheta \in B ( \widetilde{\bbeta}_j, r) ~~\text{and}~~j \in [m] , \label{eqn-clustered-20}
 \\
& \| \widetilde{\bbeta}_j - \bbeta^{\star}_{z^{\star}_j} \|_2 \leq \eta / \rho , \qquad\forall j \in [m] , 
\label{eqn-clustered-10}
\\
& 0 \geq F (\widehat\bB, \widehat\bz) - F (\bB^{\star}, \bz^{\star}) \geq \rho  E
- \frac{ m \eta^2}{\rho},  
\label{eqn-clustered-6.5}
\\
& E  \geq  \sum_{k, l \in [K]} S_{kl} H \Big(
(\| \widehat{\bbeta}_{l} - \bbeta^{\star}_{k} \|_2 -  \eta/\rho
)_{+} 
\Big)  \geq \frac{m_{\min}}{K} \sum_{k=1}^K H \Big(
(\| \widehat{\bbeta}_{\tau(k)} - \bbeta^{\star}_{k} \|_2 -  \eta/\rho
)_{+} 
\Big) .
\label{eqn-clustered-7}
\end{align}

To study $\{\widehat{\bbeta}_{k}\}_{k=1}^K$, we first make some crude estimates and then refine them.

\begin{claim}\label{claim-clustered-w}
	$\max_{k \in [K]}\| \widehat{\bbeta}_{\tau(k)} - \bbeta^{\star}_k \|_2 \leq \frac{ \eta}{\rho} (
	1 + \sqrt{ \frac{2 K m}{m_{\min}} } ) $.
\end{claim}

\begin{proof}[\bf Proof of Claim \ref{claim-clustered-w}]
	By \eqref{eqn-clustered-6.5} and \eqref{eqn-clustered-7}, for any $k \in [K]$ we have
	\begin{align*}
	&0 \geq F (\widehat\bB, \widehat\bz) - F (\bB^{\star}, \bz^{\star}) \geq \rho \frac{m_{\min}}{K}  H \Big(
	(\| \widehat{\bbeta}_{\tau(k)} - \bbeta^{\star}_{k} \|_2 -  \eta/\rho
	)_{+} 
	\Big) 
	- \frac{ m \eta^2}{\rho},
	\end{align*}
	which leads to
	\[
	H \Big(
	(\| \widehat{\bbeta}_{\tau(k)} - \bbeta^{\star}_{k} \|_2 -  \eta/\rho
	)_{+} 
	\Big) \leq \frac{ K m}{m_{\min}} \bigg( \frac{\eta}{\rho} \bigg)^2.
	\]
	On the other hand, the condition
\begin{align}
r = 4 \sqrt{ \frac{K m L}{m_{\min} \rho } } \cdot \frac{ \eta }{\rho} \geq \frac{ \eta }{\rho} \sqrt{ \frac{2 K m}{m_{\min}} },
\label{eqn-clustered-11}
\end{align}
 forces
	\[
	H\bigg(
	\frac{ \eta }{\rho} \sqrt{ \frac{2 K m}{m_{\min}} }
	\bigg)
	= \frac{1}{2} \bigg( \frac{ \eta }{\rho} \sqrt{ \frac{2 K m}{m_{\min}} } \bigg)^2 = \frac{ K m}{m_{\min}} \bigg( \frac{\eta}{\rho} \bigg)^2.
	\]
	By the monotonicity of $H(\cdot)$ on $[0, +\infty)$,
	\begin{align*}
	& (\| \widehat{\bbeta}_{\tau(k)} - \bbeta^{\star}_{k} \|_2 -  \eta/\rho
	)_{+} 
	\leq \frac{ \eta }{\rho} \sqrt{ \frac{2 K m}{m_{\min}} } , \\
	&\| \widehat{\bbeta}_{\tau(k)} - \bbeta^{\star}_{k} \|_2 \leq \frac{ \eta}{\rho} + \frac{ \eta }{\rho} \sqrt{ \frac{2 K m}{m_{\min}} } = \frac{ \eta}{\rho} \bigg(
	1 + \sqrt{ \frac{2 K m}{m_{\min}} } 
	\bigg).
	\end{align*}
\end{proof}

\begin{claim}\label{claim-2}
	$\tau:~[K] \to [K]$ is a permutation.
\end{claim}

\begin{proof}[\bf Proof of Claim \ref{claim-2}]
It suffices to show that $\tau$ is a bijection and below we prove it by contradiction. Let
\begin{align}
\Delta = \min_{k \neq l}\| \bbeta^{\star}_k - \bbeta^{\star}_l \|_2 .
\label{eqn-clustered-separation}
\end{align}
If there are $k \neq l$ such that $\tau(k) = \tau(l) = t$, then
\begin{align*}
& \max \{ \| \widehat{\bbeta}_{t} - \bbeta^{\star}_{l} \|_2, \| \widehat{\bbeta}_{t} - \bbeta^{\star}_{k} \|_2 \} \geq \Delta / 2 > 
\frac{ \eta}{\rho} \bigg( 1 + 4 \sqrt{ \frac{K m L}{m_{\min} \rho } }  \bigg) ,
\end{align*}
which contradicts Claim \ref{claim-clustered-w}.
\end{proof}

\begin{claim}\label{claim-clustered-z}
	$\widehat{z}_j = \tau ( z_j^{\star} )$, $\forall j \in [m]$.
\end{claim}

\begin{proof}[\bf Proof of Claim \ref{claim-clustered-z}]
According to the assumption
\[
( \widehat{\bTheta} , \widehat{\bz}) \in \argmin_{\bTheta\in \RR^{d\times m} , \bz \in [K]^m} F (\bTheta, \widehat{\bB},  \bz),
\]
we have
\begin{align*}
&\widehat{\bz} \in \argmin_{\bz \in [K]^m} \bigg\{ \inf_{\bTheta\in \RR^{d\times m}} F (\bTheta, \widehat{\bB},  \bz) \bigg\} = \argmin_{\bz \in [K]^m} F (\widehat{\bB},  \bz) ,\\
& \widehat{z}_j \in  \argmin_{z \in [K]} \widetilde{f}_j( \widehat{\bbeta}_{z} ), \qquad \forall j \in [m].
\end{align*}
Below we prove that $\min_{k \neq \tau(z_j^{\star}) }\widetilde{f}_j( \widehat{\bbeta}_k  ) > \widetilde{f}_j( \widehat{\bbeta}_{\tau(z_j^{\star})} )$.
 
On the one hand, $\widetilde{f}_j (\bbeta) = \min_{\bxi} \{ f_j(\bbeta - \bxi) + \lambda \| \bxi \|_2 \} \leq f_j(\bbeta)$ for all $\bbeta$. Hence
\[
\widetilde{f}_j (\widehat{\bbeta}_{\tau ( z_j^{\star} )})
 \leq f_j (\widehat{\bbeta}_{\tau ( z_j^{\star} )}).
\]
By Claim \ref{claim-clustered-w} and \eqref{eqn-clustered-10},
\[
 \widehat{\bbeta}_{\tau ( z_j^{\star} )} ,  \widetilde{\bbeta}_j \in B(  \bbeta^{\star}_{z^{\star}_j} , M) .
\]
Then, the assumption $ \nabla^2 f_j (\btheta) \preceq L \bI$, $\forall \btheta \in B( \bbeta^{\star}_{z_j^{\star}}  , M)$ yields
\begin{align*}
 f_j (\widehat{\bbeta}_{\tau ( z_j^{\star} )})
& \leq  f_j (\widetilde{\bbeta}_j)  + \frac{L}{2} \| \widehat{\bbeta}_{\tau ( z_j^{\star} )} - \widetilde{\bbeta}_j \|_2^2 
\notag \\ & 
\leq  f_j (\widetilde{\bbeta}_j) + \frac{L}{2} \Big( \| \widehat{\bbeta}_{\tau ( z_j^{\star} )} - \bbeta^{\star}_{ z_j^{\star} } \|_2 
+ \| \bbeta^{\star}_{ z_j^{\star} }  - \widetilde{\bbeta}_j \|_2 
\Big)^2 
\notag \\&
\leq f_j (\widetilde{\bbeta}_j) + \frac{L}{2} \bigg[
\frac{ \eta}{\rho} \bigg(
1 + \sqrt{ \frac{2 K m}{m_{\min}} } \bigg) + \frac{ \eta}{\rho}
\bigg]^2,
\end{align*}
where the last inequality follows from Claim \ref{claim-clustered-w} and \eqref{eqn-clustered-10}. Hence
\begin{align}
\widetilde{f}_j (\widehat{\bbeta}_{\tau ( z_j^{\star} )})
\leq f_j (\widetilde{\bbeta}_j) + \frac{L}{2} \bigg[
\frac{ \eta}{\rho} \bigg(
1 + \sqrt{ \frac{2 K m}{m_{\min}} } \bigg) + \frac{ \eta}{\rho}
\bigg]^2,
\label{eqn-clustered-8}
\end{align}

On the other hand, for any $k \neq \tau(z_j^{\star})$ we have $\tau^{-1}(k) \neq z_j^{\star}$ and
\begin{align*}
\| \widehat{\bbeta}_k - \widetilde{\bbeta}_j \|_2 & \geq \| \widetilde{\bbeta}_j - \bbeta^{\star}_{\tau^{-1}(k)} \|_2 - \| \bbeta^{\star}_{\tau^{-1}(k)} - \widehat\bbeta_k \|_2 \\
& \geq \|  \bbeta^{\star}_{\tau^{-1}(k)} - \bbeta^{\star}_{z_j^{\star}} \|_2 - \| \bbeta^{\star}_{z_j^{\star}} - \widetilde{\bbeta}_j  \|_2 - \| \bbeta^{\star}_{\tau^{-1}(k)} - \widehat\bbeta_k \|_2 \\
& \geq \Delta - \frac{ \eta}{\rho} - \frac{ \eta}{\rho} \bigg(
1 + \sqrt{ \frac{2 K m}{m_{\min}} } \bigg),
\end{align*}
where the last inequality follows from \eqref{eqn-clustered-separation}, \eqref{eqn-clustered-10} and Claim \ref{claim-clustered-w}. By $\Delta > \frac{2 \eta}{\rho} ( 1 + 4 \sqrt{ \frac{K m L}{m_{\min} \rho } }  ) $,
we have
\begin{align*}
\| \widehat{\bbeta}_k - \widetilde{\bbeta}_j \|_2 & \geq 
\Delta - \frac{ \eta}{\rho} - \frac{ \eta}{\rho} \bigg(
1 + \sqrt{ \frac{2K m}{m_{\min}} } \bigg) \\
& > 
\frac{2 \eta}{\rho} \bigg( 1 + 4 \sqrt{ \frac{K m L}{m_{\min} \rho } }  \bigg) 
 - \frac{ \eta}{\rho} - \frac{ \eta}{\rho} \bigg(
1 + \sqrt{ \frac{2 K m}{m_{\min}} } \bigg) \geq r.
\end{align*}
By \eqref{eqn-clustered-20} and \Cref{lem-cvx-lower},
\begin{align}
\widetilde{f}_j( \widehat{\bbeta}_k  ) > f_j( \widetilde{\bbeta}_j ) + \frac{\rho r^2}{2}.
\label{eqn-clustered-9}
\end{align}

The inequalities \eqref{eqn-clustered-8} and \eqref{eqn-clustered-9} imply that for any $k \neq \tau(z_j^{\star})$,
\begin{align*}
\widetilde{f}_j( \widehat{\bbeta}_k  ) - \widetilde{f}_j( \widehat{\bbeta}_{\tau(z_j^{\star})} ) & >
\frac{L}{2} \bigg[
\frac{ \eta}{\rho} \bigg(
1 + \sqrt{ \frac{2 K m}{m_{\min}} } \bigg) + \frac{ \eta}{\rho}
\bigg]^2 - \frac{\rho r^2}{2} \\
& = \frac{\rho}{2} \bigg\{ 
\frac{L}{\rho}
\bigg[
\frac{  \eta}{\rho} \bigg(
2 + \sqrt{ \frac{2 K m}{m_{\min}} } \bigg) 
\bigg]^2
- r^2
\bigg\} .
\end{align*}
Then $\min_{k \neq \tau(z_j^{\star}) }\widetilde{f}_j( \widehat{\bbeta}_k  ) > \widetilde{f}_j( \widehat{\bbeta}_{\tau(z_j^{\star})} )$ follows from the fact
\[
r = 4 \sqrt{ \frac{K m L}{m_{\min} \rho } } \cdot \frac{ \eta }{\rho}
\geq \sqrt{\frac{L}{\rho}} \bigg(2 + \sqrt{ \frac{2 K m}{m_{\min}} } \bigg) \frac{ \eta }{\rho}.
\]
\end{proof}

\begin{claim}\label{claim-4}
$\widehat{\bbeta}_{\tau(k)} = \widehat{\bbeta}_{ k }^{\mathrm{oracle}}  $ and $\| \widehat{\bbeta}_{ k }^{\mathrm{oracle}} - \bbeta^{\star}_k \|_2 \leq ( \rho  |T(k)| )^{-1} \| \sum_{ j \in T(k) } \nabla f_j (\bbeta^{\star}_k) \|_2$ hold for all $k \in [K]$; $\widehat{\btheta}_j = \widehat{\bbeta}_{\widehat{z}_j}$ holds for all $j \in [m]$.
\end{claim}

\begin{proof}[\bf Proof of Claim \ref{claim-4}]
By definition, $\widehat{\bbeta}_{\tau(k)} \in \argmin_{\bbeta \in \RR^d} \{ \sum_{j  :~ \widehat{z}_j = \tau(k) } \widetilde{f}_j (\bbeta) \}$ and $\widehat{\btheta}_j \in \argmin_{\btheta \in \RR^d} \{ f_j (\btheta) + \lambda \| \widehat{\bbeta}_{\widehat{z}_j} - \btheta \|_2 \}$. 
Claim \ref{claim-clustered-z} yields $\widehat{\bbeta}_{\tau(k)} \in \argmin_{\bbeta \in \RR^d} \sum_{j \in T(k) } \widetilde{f}_j (\bbeta) $. Since $\max_{j \in T(k) }  \| \nabla f_j (\bbeta^{\star}_{z_j^{\star}} ) \|_2 \leq \eta$ and $\lambda >  3 \eta  L / \rho$, \Cref{thm-consensus} applied to $\{ f_j \}_{j \in T(k)}$ proves the claim.
\end{proof}

\subsection{Proof of \Cref{thm-armul-clustered-deterministic}}\label{sec-thm-armul-clustered-deterministic-proof}

We invoke the following lemma, whose proof is in \Cref{sef-thm-robustness-clustered-proof}.

\begin{lemma}[Robustness]\label{thm-robustness-clustered}
	Let $K \geq 2$, $\alpha \in (0, 1]$, and $(\widehat{\bTheta}, \widehat{\bB},\widehat{\bz}) $ be an optimal solution to \eqref{eqn-obj-clustered}.
	Define $T(k) =  \{ j \in [m] :~ z_j^{\star} = k \} $, $\widehat{\bbeta}_k^{\mathrm{oracle}} \in \argmin_{\bbeta \in \RR^d}  \sum_{j \in T(k) \cap S}  f_j ( \bbeta )  $ for $k \in [K]$.
	Suppose there are $0 < \rho \leq L < +\infty$, $\eta > 0$, and $S \subseteq [m]$ with $|S^c| / m \leq \alpha \rho / (6 K^2 L)$ that satisfy the followings:
	\begin{itemize}
		\item (Regularity) $\| \nabla f_j ( \bbeta^{\star}_{z_j^{\star}} ) \|_2 \leq \eta$ and $\rho \bI \preceq \nabla^2 f_j (\btheta) \preceq L \bI$ hold for all $  j \in S $ and $  \btheta \in B( \bbeta^{\star}_{z_j^{\star}}  , M)$;
		\item (Balancedness) $\min_{k \in[K]} |T(k)|  \geq \frac{ \alpha m }{K} (1 + \frac{\rho}{6 L K })$.
	\end{itemize}
	If
	\[
	\frac{9 K L \eta}{\sqrt{\alpha} \rho } < \lambda < \rho \cdot \min\Big\{  M ,~ \min_{k \neq l}
	\| \bbeta^{\star}_k - \bbeta^{\star}_l \|_2 / 6 \Big\},
	\]
	there is a permutation $\tau$ of $[K]$ such that $ \widehat{z}_j = \tau (z_j^{\star})$ and $\widehat{\btheta}_j = \widehat\bbeta_{\widehat{z}_j }$ hold for all $j \in S$,
	\begin{align*}
	&\| \widehat{\bbeta}_{\tau(k)} -	\widehat{\bbeta}_{k}^{\mathrm{oracle}}
	\|_2 \leq \frac{ \lambda |S^c|   }{  \rho |T(k) \cap S| }, \qquad \forall k \in [K] .
	\end{align*}
\end{lemma}

We first use \Cref{thm-adaptivity-clustered} to derive the following intermediate result.

\begin{claim}\label{thm-armul-clustered-deterministic-claim}
	Define $g = \max_{j \in [m]} \| \nabla f_j (\btheta_j^{\star}) \|_2 $. Under the assumptions in \Cref{thm-armul-clustered-deterministic}, if 
	\[
	10 L \delta \leq g + \frac{ \sqrt{\alpha}\lambda }{ 10 \kappa  K  } ,
	\]
	then, there exists a permutation $\tau$ of $[K]$ such that
	\begin{align*}
	&  \widehat{z}_j = \tau (z_j^{\star}) \qquad\text{and}\qquad \widehat{\btheta}_j = \widehat\bbeta_{\widehat{z}_j },\qquad\forall  j \in S ,\\
	&	\| \widehat{\btheta}_j  - \btheta^{\star}_j \|_2  
	\leq  
	\frac{1}{\rho |T(z_j^{\star}) \cap S| } \bigg\|  \sum_{j \in T(z_j^{\star}) \cap S } \nabla f_j ( \btheta^{\star}_{ j } ) \bigg\|_2
	+  2 \kappa \delta 
	+ \frac{ \lambda |S^c|   }{  \rho |T(z_j^{\star}) \cap S| }
	, \qquad\forall j \in S.
	\end{align*}
\end{claim}
\begin{proof}[\bf Proof of Claim \ref{thm-armul-clustered-deterministic-claim}]
	We obtain from the assumption
	\begin{align}
	\frac{ 12 \kappa K }{\sqrt{\alpha}} \max_{j \in [m]}   \| \nabla f_j (\btheta_j^{\star}) \|_2   < \lambda < \frac{ \rho }{6} \cdot \min \Big\{ M,~ \min_{k \neq l}
	\| \bbeta^{\star}_k - \bbeta^{\star}_l \|_2   \Big\}
	\label{eqn-armul-deterministic-clustered-lambda-}
	\end{align}
	that $g \leq \rho M   / 12 $, $\lambda < \rho M $ and $g + \lambda < \frac{13}{12} LM  $. When
	\[
	10 L \delta \leq g + \frac{ \sqrt{\alpha} \lambda }{ 10 \kappa  K  }  \leq g + \lambda,
	\]
	we have $ \delta < 13M/120 < M /9$. Thus $\max_{j \in [m]} \| \btheta^{\star}_j  - \bbeta^{\star}_{z_j^{\star}} \|_2 \leq M/9$; $\rho \bI \preceq \nabla^2 f_j(\btheta) \preceq L \bI$ holds for all $j \in [m]$ and $\btheta \in B( \bbeta^{\star}_{z_j^{\star}}, 8M/9)$.

	For any $j$, the regularity condition $ \nabla^2 f_j (\btheta) \preceq L \bI$, $\forall \btheta \in B(\btheta_j^{\star}, M)$ leads to $ \| \nabla f_j (\btheta^{\star}_j)  - \nabla f_j (\bbeta^{\star}_{z_j^{\star}})  \|_2 \leq L \delta $. By triangle's inequality,
	\begin{align*}
	\| \nabla f_j (\bbeta^{\star}_{z_j^{\star}}) \|_2 
	&\leq   \| \nabla f_j (\btheta^{\star}_j) \|_2  +   \| \nabla f_j (\bbeta^{\star}_{z_j^{\star}})  - \nabla f_j (\btheta^{\star}_j)  \|_2   	\leq g + L \delta
	<  \frac{11g}{10} + \frac{ \sqrt{\alpha} \lambda }{ 100 \kappa  K  }  .
	\end{align*}
	Consequently,
	\begin{align*}
	\frac{ 9 \kappa  K }{\sqrt{\alpha}} \max_{j \in [m]}  \| \nabla f_j (\bbeta^{\star}_{z_j^{\star}}) \|_2 
	< 	\frac{ 99 \kappa  K }{10 \sqrt{\alpha}}  g + \frac{ \lambda }{ 10 } 
	\overset{\mathrm{(i)}}{<}  \lambda  \overset{\mathrm{(ii)}}{<}
	\rho \cdot \min \Big\{ \frac{8M}{9},~ \min_{k \neq l}
	\| \bbeta^{\star}_k - \bbeta^{\star}_l \|_2 / 6 \Big\}.
	\end{align*}
	The inequalities $\mathrm{(i)}$ and $\mathrm{(ii)}$ follow from \Cref{eqn-armul-deterministic-clustered-lambda}.
	Based on the above estimates, \Cref{thm-robustness-clustered} asserts the existence of a permutation 	$\tau$ of $[K]$ such that
	$ \widehat{z}_j = \tau (z_j^{\star})$ and $\widehat{\btheta}_j = \widehat\bbeta_{\widehat{z}_j }$ hold for all $j \in S$,
	\begin{align*}
	&\| \widehat{\bbeta}_{\tau(k)} -	\widehat{\bbeta}_{k}^{\mathrm{oracle}}
	\|_2 \leq \frac{ \lambda |S^c|   }{  \rho |T(k) \cap S| }, \qquad \forall k \in [K] .
	\end{align*}
	Meanwhile, \Cref{thm-consensus} applied to $\{ f_j \}_{j \in T(k) \cap S}$ shows that
	\[
	\| \widehat{\bbeta}_{k}^{\mathrm{oracle} } - \bbeta^{\star}_k \|_2 \leq \frac{1}{\rho |T(k) \cap S| } \bigg\|  \sum_{j \in T(k) \cap S } \nabla f_j ( \bbeta^{\star}_{ k } ) \bigg\|_2.
	\]
	Therefore,
	\[
	\| \widehat{\btheta}_j - \bbeta^{\star}_{z_j^{\star}} \|_2 \leq
	\frac{1}{\rho |T(z_j^{\star}) \cap S| } \bigg\|  \sum_{j \in T(z_j^{\star}) \cap S } \nabla f_j ( \bbeta^{\star}_{ z_j^{\star}} ) \bigg\|_2
	+ \frac{ \lambda |S^c|   }{  \rho |T(z_j^{\star}) \cap S| }
	, \qquad\forall j \in S.
	\]
	
	By the triangle's inequality, $ \| \widehat{\btheta}_j  - \btheta^{\star}_j \|_2 \leq \| \widehat{\btheta}_j  - \bbeta^{\star}_{ z^{\star}_j } \|_2 +  \| \bbeta^{\star}_{ z^{\star}_j } - \btheta^{\star}_j \|_2 \leq \| \widehat\bbeta_{\widehat{z}_j}  - \bbeta^{\star}_{ z^{\star}_j } \|_2  + \delta $. Also,
	\begin{align*}
	& \frac{ \| \sum_{j \in T( z^{\star}_j ) \cap S }   \nabla f_j (  \bbeta^{\star}_{ z^{\star}_j } ) \|_2 }{ |T( z^{\star}_j )  \cap S| } 
	\leq \frac{ \| \sum_{j \in T( z^{\star}_j )  \cap S }   \nabla f_j (  \btheta^{\star}_{ j } ) \|_2 }{ |T( z^{\star}_j )  \cap S| }  + L \delta  .
	\end{align*}
	Based on the above estimates,
	\begin{align*}
	\| \widehat{\btheta}_j  - \btheta^{\star}_j \|_2 & \leq \frac{ 1 }{  \rho }
	\bigg(
	\frac{ \| \sum_{j \in T( z^{\star}_j ) \cap S }   \nabla f_j (  \btheta^{\star}_{ j } ) \|_2 }{ |T( z^{\star}_j ) \cap S | }  + L \delta +  \frac{ \lambda |S^c|   }{    |T(z_j^{\star}) \cap S| }
	\bigg) + \delta \\&
	\leq  
	\frac{1}{\rho |T(z_j^{\star}) \cap S| } \bigg\|  \sum_{j \in T(z_j^{\star}) \cap S } \nabla f_j ( \btheta^{\star}_{ j } ) \bigg\|_2
	+  2 \kappa \delta 
	+ \frac{ \lambda |S^c|   }{  \rho |T(z_j^{\star}) \cap S| }
	, \qquad\forall j \in S .
	\end{align*}
\end{proof}

We now come back to \Cref{thm-armul-clustered-deterministic}. The condition \eqref{eqn-armul-deterministic-clustered-lambda-} forces $\lambda > 12  \kappa  K  g / \sqrt{\alpha}$.
Claim \ref{thm-armul-clustered-deterministic-claim} implies that when 
$	10 L \delta \leq g + \frac{ \sqrt{\alpha}\lambda }{ 10 \kappa  K  }$,
\begin{align*}
\| \widehat{\btheta}_j  - \btheta^{\star}_j \|_2  
& \leq  
\frac{1}{\rho |T(z_j^{\star}) \cap S| } \bigg\|  \sum_{j \in T(z_j^{\star}) \cap S } \nabla f_j ( \btheta^{\star}_{ j } ) \bigg\|_2
+   
\frac{2}{\rho} \min\bigg\{ L \delta ,~ \frac{g}{5} + \frac{ \sqrt{\alpha} \lambda }{ 50  \kappa  K  } \bigg\}  \\
& ~~~~ + \frac{ \lambda |S^c|   }{  \rho |T(z_j^{\star}) \cap S| } , \qquad \forall j \in S.
\end{align*}
On the other hand, when $	10 L \delta > g + \frac{ \sqrt{\alpha}\lambda }{ 10 \kappa  K  }$, we use \Cref{thm-armul-personalization} to get 
\[
\max_{j \in [m]} \| \widehat{\btheta}_j  - \btheta^{\star}_j \|_2  \leq \frac{g + \lambda}{\rho } \leq \bigg(  \frac{ \sqrt{\alpha}   }{ 12  \kappa  K  }  + 1 \bigg)\frac{\lambda}{\rho  } \leq \frac{13 \lambda}{12  \rho  } .
\]
On top of the above, for any $j \in S$ we have
\begin{align*}
\| \widehat{\btheta}_j  - \btheta^{\star}_j \|_2  
& \leq \frac{
	\|  \sum_{j \in T(z_j^{\star}) \cap S } \nabla f_j ( \btheta^{\star}_{ j } )  \|_2
}{\rho |T(z_j^{\star}) \cap S| }
+   \frac{55 \kappa  K}{\rho \sqrt{\alpha}} \min\bigg\{ L \delta ,~ \frac{g}{5} + \frac{ \sqrt{\alpha} \lambda }{ 50  \kappa  K  } \bigg\}  + \frac{ \lambda |S^c|   }{  \rho |T(z_j^{\star}) \cap S| } \\
& \leq 
\frac{
	\|  \sum_{j \in T(z_j^{\star}) \cap S } \nabla f_j ( \btheta^{\star}_{ j } )  \|_2
}{\rho |T(z_j^{\star}) \cap S| }
+   
\frac{55 \kappa  K}{\rho \sqrt{\alpha}} 
\min\bigg\{   L \delta ,~   \frac{ \sqrt{\alpha} \lambda }{ 25  \kappa  K  }   \bigg\}   + \frac{ \lambda |S^c|   }{  \rho |T(z_j^{\star}) \cap S| }  \\
& \leq
\frac{
	\|  \sum_{j \in T(z_j^{\star}) \cap S } \nabla f_j ( \btheta^{\star}_{ j } )  \|_2
}{\rho |T(z_j^{\star}) \cap S| }
+     \min\bigg\{ \frac{ 55 \kappa^2 K \delta }{  \sqrt{\alpha} } ,~    \frac{ 11 \lambda }{ 5 \rho }   \bigg\} 
+ \frac{ \lambda |S^c|   }{  \rho |T(z_j^{\star}) \cap S| }  .
\end{align*}

Finally, if $\delta \leq \sqrt{\alpha} \lambda / ( 100  \kappa K L )$, then we use Claim \ref{thm-armul-clustered-deterministic-claim} to get a permutation $\tau$ of $[K]$ such that
\begin{align*}
&  \widehat{z}_j = \tau (z_j^{\star}) \qquad\text{and}\qquad \widehat{\btheta}_j = \widehat\bbeta_{\widehat{z}_j },\qquad\forall  j \in S .
\end{align*}

\subsection{Proof of \Cref{thm-robustness-clustered}}\label{sef-thm-robustness-clustered-proof}

Define $\kappa = L / \rho$, $\varepsilon = |S^c| / m$, $C = \alpha - K \varepsilon$ and $\Delta = \min_{k \neq l}\| \bbeta^{\star}_k - \bbeta^{\star}_l \|_2 $. We have $\varepsilon \leq \frac{\alpha}{6 \kappa K^2}$ and
\begin{align}
C \geq \alpha - K \cdot \frac{\alpha}{6 \kappa K^2} = \alpha \bigg( 1 - \frac{1}{6 \kappa K} \bigg) \geq \frac{5 \alpha}{6}.
\label{eqn-C}
\end{align}
As a result,
\begin{align}
\frac{8 \kappa K \eta}{\sqrt{C}}
\leq \frac{8 \kappa K \eta}{\sqrt{5\alpha/6}}
\leq \frac{9 \kappa K \eta}{\sqrt{\alpha}} < \lambda
< \rho \cdot \min \{ M, \Delta / 6 \} \leq \min \{ LM, \rho \Delta / 6 \}.
\label{eqn-lambda}
\end{align}

Define
\begin{align*}
F(\bTheta, \bB, \bz ) = \sum_{j=1}^{m} [ f_j ( \btheta_j ) + \lambda \| \bbeta_{z_j} - \btheta_j \|_2 ],\qquad
\forall \bTheta \in \RR^{d \times m},~~\bB\in \RR^{d\times K}, ~~ \bz \in [K]^m.
\end{align*}
Then, $(\widehat{\bTheta} , \widehat{\bB}, \widehat{\bz})$ is a minimizer of $F$ under the cardinality constraint $\min_{k \in [K]} | \{ j \in [m] :~ z_j = k  \} | \geq \alpha m / K$.
Define $\bB^{\star} = (\bbeta^{\star}_1,\cdots,\bbeta^{\star}_K)$, $\widehat{\bB}^{\mathrm{oracle}} = (\widehat{\bbeta}_1^{\mathrm{oracle}} , \cdots, \widehat{\bbeta}_K^{\mathrm{oracle}})$, $\widetilde{f}_j = f_j \square (\lambda \| \cdot \|_2)$ and $\widetilde{\bbeta}_j = \argmin_{\bbeta} \widetilde{f}_j(\bbeta)$. With slight abuse of notation, let
\[
F(\bB, \bz) = \inf_{\bTheta \in \RR^{d\times m}} F(\bTheta, \bB, \bz) = \sum_{j=1}^{m} \widetilde{f}_j ( \bbeta_{z_j} ) , \qquad \forall \bB\in \RR^{d \times K},~~ \bz \in [K]^m.
\]
For $k, l \in [K]$, define $T_{kl} = |\{ j \in S:~ \widehat{z}_j = k,~ z^{\star}_j = l \}|$ and $\sigma(k) = \argmax_{l \in [K]} T_{kl}$ using any tie-breaking rule. Construct $\widetilde{\bz} \in [K]^m$ through
\[
\widetilde{z}_j = \begin{cases}
z^{\star}_j &, \mbox{ if } j \in S \\
\sigma(\widehat{z}_j) &, \mbox{ if } j \in S^c
\end{cases}.
\]
Thanks to the balancedness assumption $\min_{k \in[K]} |T(k)|  \geq \frac{ \alpha m }{K} (1 + \frac{1}{6 \kappa K })$ and  $\varepsilon \leq \alpha / (6 \kappa K^2)$, we have
\begin{align*}
\min_{k \in[K]} 
|\{ j \in [m] :~ \widetilde{z}_j = k \}|
& \geq 
\min_{k \in[K]} 
|\{ j \in [m] :~ z^{\star}_j = k \}| - |S^c| \\
& \geq \frac{ \alpha m }{K} \bigg(1 + \frac{1}{6 \kappa K } \bigg) - \varepsilon m \geq \alpha m / K.
\end{align*}
Hence, both $\widetilde{\bz}$ also satisfies the cardinality constraint in \eqref{eqn-obj-clustered}. By the optimality of $(\widehat{\bTheta}, \widehat{\bB} , \widehat{\bz})$,
\begin{align}
F(\widehat{\bB} , \widehat{\bz}) = F(\widehat{\bTheta}, \widehat\bB, \widehat\bz) \leq
\inf_{\bTheta \in \RR^{d\times m}}
F(\bTheta, \bB^{\star},  \widetilde\bz )
= F(\bB^{\star}, \widetilde\bz ).
\label{eqn-clustered-robustness-0}
\end{align}
This is our starting point for analyzing $(\widehat{\bTheta}, \widehat{\bB},   \widehat{\bz})$. By definition,
\begin{align}
F (\widehat{\bB} , \widehat{\bz} ) - F ( \bB^{\star} , \widetilde\bz ) =
\underbrace{
	\sum_{j \in S} [ \widetilde{f}_j (\widehat{\bbeta}_{\widehat{z}_j}) - \widetilde{f}_j (\bbeta^{\star}_{z^{\star}_j})]
}_{L_1} + 
\underbrace{
	\sum_{j \in S^c} [ \widetilde{f}_j (\widehat{\bbeta}_{\widehat{z}_j}) - \widetilde{f}_j (\bbeta^{\star}_{\sigma( \widehat{z}_j)} )]
}_{L_2}.
\label{eqn-clustered-robustness-1}
\end{align}
\Cref{lem-inf-conv-lip} implies that $\widetilde{f}_j$ is $\lambda$-Lipschitz and then
\begin{align*}
L_2 \geq - \lambda \sum_{j \in S^c} \| \widehat{\bbeta}_{\widehat{z}_j}
- \bbeta^{\star}_{\sigma( \widehat{z}_j)} \|_2.
\end{align*}
Choose any $u \in \argmax_{k \in [K]} \| \widehat{\bbeta}_{k}
- \bbeta^{\star}_{\sigma( k)} \|_2$ and let $U = \| \widehat{\bbeta}_{u}
- \bbeta^{\star}_{\sigma(u)} \|_2$. We have
\begin{align}
L_2 \geq - \lambda |S^c| U .
\label{eqn-clustered-robustness-2}
\end{align}

Define $r = \frac{\alpha \lambda}{2 C \kappa   \rho}$. To study $L_1$, we use \eqref{eqn-lambda} to get
\[
\lambda - \bigg( \frac{3 \eta L}{\rho} + L r \bigg) = \lambda \bigg( 1 - \frac{\alpha L}{2 C \kappa   \rho} \bigg) - \frac{3 \eta L}{\rho}
\geq \frac{ 8 \kappa K \eta }{\sqrt{C}} \cdot \bigg( 1 - \frac{\alpha }{2 C} \bigg) - 3 \kappa \eta .
\]
In light of \eqref{eqn-C},
\[
\lambda - \bigg( \frac{3 \eta L}{\rho} +L  r \bigg) 
\geq  \kappa \eta \bigg[ \frac{8K}{\sqrt{C}} \bigg( 1 - \frac{3  }{5} \bigg) -  3 \bigg]  > 0.
\]
Applying \Cref{lem-clustered-lower-bound} to $\{ f_j \}_{j \in S}$, we get
\begin{align}
L_1 & \geq \rho \sum_{j \in S} H [
( \| \widehat{\bbeta}_{\widehat{z}_j} - \bbeta^{\star}_{z^{\star}_j} \|_2 - \eta / \rho )_+
] - \frac{ |S| \eta^2 }{\rho} 
\geq \rho T_{u \sigma(u)} H [
( U - \eta / \rho )_+
] - \frac{ |S| \eta^2 }{\rho} .
\label{eqn-clustered-robustness-2.5}
\end{align}
Here $H(t) = t^2 / 2$ if $0 \leq t \leq r$ and $H(t) = r(t-r/2)$ if $t > r$. Note that 
\[
T_{u \sigma(u)} \geq \frac{1}{K} \sum_{l=1}^{K} T_{ul} = \frac{1}{K} |\{ j \in S:~ \widehat{z}_j = u \}|.
\]

According to the constraint $\min_{k \in [K]} |\{ j \in [m]:~ \widehat z_j = k \}| \geq \alpha m / K $,
\begin{align*}
\min_{k \in [K]} |\{ j \in S:~ \widehat z_j = k \}| \geq \alpha m / K - |S^c| \geq (\alpha / K - \varepsilon) m = C m / K.
\end{align*}
Therefore, $T_{u \sigma(u)} \geq Cm / K^2$. By \eqref{eqn-clustered-robustness-2.5},
\begin{align}
L_1 \geq  \frac{Cm \rho}{K^2} H [
( U - \eta / \rho )_+
] - \frac{ |S| \eta^2 }{\rho} .
\label{eqn-clustered-robustness-3}
\end{align}
By Equations (\ref{eqn-clustered-robustness-0}), (\ref{eqn-clustered-robustness-1}), (\ref{eqn-clustered-robustness-2}) and (\ref{eqn-clustered-robustness-3}), we have
\begin{align}
0 \geq \frac{ F (\widehat{\bB} , \widehat{\bz} ) - F ( \bB^{\star} , \widetilde\bz ) }{m}
\geq  \frac{C \rho}{K^2} H [
( U - \eta / \rho )_+
] - \frac{  \eta^2 }{\rho} 
- \varepsilon \lambda U
\triangleq h(U)
.
\label{eqn-clustered-robustness-4}
\end{align}

\begin{claim}\label{claim-clustered-robustness-1}
	We have $h(t) > 0$ when $t \geq r + \eta / \rho$. As a result, $U <  r + \eta / \rho$.
\end{claim}
\begin{proof}[\bf Proof of Claim \ref{claim-clustered-robustness-1}]
	Suppose that $t \geq r + \eta / \rho$ and let $R = t - \eta / \rho$. Since $H(x) \geq r (x - r/2)$ for all $x$, we use $R \geq r$ to get
	\begin{align*}
	& h(t) \geq \frac{C \rho}{K^2} (r R - r^2/2) - \frac{  \eta^2 }{\rho} 
	- \varepsilon \lambda  (R + \eta / \rho) 
	= \bigg( \frac{C \rho r}{K^2} - \varepsilon \lambda  \bigg) R - \bigg( \frac{C \rho r^2}{2K^2} + \frac{\eta^2}{\rho} + \frac{\varepsilon\lambda \eta}{\rho} \bigg) .
	\end{align*}
We claim that $\varepsilon \lambda  \leq  \frac{C \rho r}{3 K^2} $, $\frac{\eta^2}{\rho}  < \frac{C \rho r^2}{12K^2} $ and $\frac{\varepsilon\lambda \eta}{\rho}  < \frac{C \rho r^2}{12K^2} $. If those are true, then
\begin{align*}
& h(t) >  \frac{2C \rho r}{3K^2} \cdot R -  \frac{2C \rho r^2}{3K^2} \geq \frac{2C \rho r}{3K^2} \cdot r -  \frac{2C \rho r^2}{3K^2} = 0,
\end{align*}
and we derive Claim \ref{claim-clustered-robustness-1}. Below we prove the three claimed relations.
First, we use $r =\frac{\alpha \lambda}{2 C \kappa \rho} $ and $\varepsilon \leq \frac{\alpha}{6 \kappa K^2 }$ to get
\[
\varepsilon \lambda  \bigg/  \frac{C \rho r}{3 K^2} = \frac{3 K^2 \varepsilon \lambda}{C \rho \cdot \frac{\alpha \lambda}{2 C \kappa \rho}} = \frac{6 \kappa K^2 \varepsilon}{\alpha} \leq 1.
\]
Second, we use \eqref{eqn-lambda} to get
\begin{align*}
\frac{\eta^2}{\rho}  \bigg/ \frac{C \rho r^2}{12K^2} 
& = \frac{12 K^2 \eta^2}{C \rho^2 r^2} = \bigg(
\frac{\sqrt{12} K \eta}{ \sqrt{C} \rho \cdot \frac{\alpha\lambda}{2 C \kappa \rho} }
\bigg)^2 = \bigg(
\frac{2 \sqrt{12 C} \kappa K \eta}{ \alpha \lambda }
\bigg)^2 \\
& \leq \bigg(
\frac{2 \sqrt{12 C} \kappa K \eta}{ \alpha \cdot 8 \kappa K \eta / \sqrt{C} }
\bigg)^2  =  \bigg(
\frac{  \sqrt{3 C}  }{2 \alpha }
\bigg)^2 < 1.
\end{align*}
Third, we combine $\varepsilon \lambda  \leq  \frac{C \rho r}{3 K^2} $, $\frac{\eta^2}{\rho}  < \frac{C \rho r^2}{12K^2} $, $C \leq 1$ and $K \geq 2$ to get
\[
\frac{\varepsilon\lambda \eta}{\rho}  <  \frac{C \rho r}{3 K^2} \cdot \sqrt{ \frac{C r^2}{12K^2} } =
\frac{C \rho r^2}{12K^2} \cdot \frac{2 \sqrt{C}}{\sqrt{3} K} < \frac{C \rho r^2}{12K^2} .
\]
\end{proof}

Claim \ref{claim-clustered-robustness-1} helps control the contribution of outliers in $S^c$ to the loss function: by \eqref{eqn-clustered-robustness-2}, $r =\frac{\alpha \lambda}{2 C \kappa \rho} $, $\varepsilon \leq \frac{\alpha}{6 \kappa K^2}$ and Claim \ref{claim-clustered-robustness-1},
\[
L_2 \geq - \lambda \cdot \varepsilon m \cdot (r + \eta / \rho).
\]
Then, we use \Cref{eqn-clustered-robustness-0,eqn-clustered-robustness-1,eqn-clustered-robustness-2.5} to get
\begin{align}
0 \geq   F (\widehat{\bB} , \widehat{\bz} ) - F ( \bB^{\star} , \widetilde\bz ) 
\geq \rho \sum_{j \in S} H [
( \| \widehat{\bbeta}_{\widehat{z}_j} - \bbeta^{\star}_{z^{\star}_j} \|_2 - \eta / \rho )_+
] - \frac{ m \eta^2 }{\rho}  -   m \varepsilon\lambda (r + \eta / \rho).
\label{eqn-clustered-robustness-5}
\end{align}
The rest of proof is similar to that of \Cref{thm-adaptivity-clustered}.
Define $S_{kl} = | \{ j \in S :~ z_j^{\star} = k,~ \widehat{z}_j = l \}|$, $\tau(k) = \argmax_{l \in [K]} S_{k l}$ for $k, l \in [K]$,
\begin{align*}
&E  = \sum_{j \in S} H \Big(
(\| \widehat\bbeta_{\widehat z_j} - \bbeta^{\star}_{z^{\star}_j} \|_2 -  \eta/\rho
)_{+} 
\Big) .
\end{align*}
By $\min_{k \in [K]} |\{ j \in [m]:~z^{\star}_j = k \}| \geq \alpha m/ K $, we have 
\begin{align*}
\min_{k \in [K]} |\{ j \in S:~ z^{\star}_j = k \}| \geq \frac{\alpha m}{K}   - |S^c| \geq (\alpha / K - \varepsilon) m = C m / K.
\end{align*}
Then $\min_{k \in [K]} S_{k \tau(k)} \geq C m/ K^2$ and
\begin{align}
E & \geq  \sum_{k, l \in [K]} S_{kl} H \Big(
(\| \widehat{\bbeta}_{l} - \bbeta^{\star}_{k} \|_2 -  \eta/\rho
)_{+} 
\Big)  \geq \frac{ C m }{ K^2 } \sum_{k=1}^K H \Big(
(\| \widehat{\bbeta}_{\tau(k)} - \bbeta^{\star}_{k} \|_2 -  \eta/\rho
)_{+} 
\Big) .
\label{eqn-clustered-robustness-6}
\end{align}

\begin{claim}\label{claim-clustered-robustness-w}
	$\max_{k \in [K]} \| \widehat{\bbeta}_{\tau(k)} - \bbeta^{\star}_k \|_2 \leq r + \eta / \rho$.
\end{claim}

\begin{proof}[\bf Proof of Claim \ref{claim-clustered-robustness-w}]
	By \Cref{eqn-clustered-robustness-5,eqn-clustered-robustness-6}, for any $k \in [K]$ we have
	\begin{align}
	0 & \geq    \frac{Cm \rho}{K^2} H [
	( \| \widehat{\bbeta}_{\tau(k)} - \bbeta^{\star}_{k} \|_2 - \eta / \rho )_+
	] - \frac{ m \eta^2 }{\rho}  - m \varepsilon\lambda (r + \eta / \rho) \notag \\
	& = m h(r + \eta / \rho) + \frac{Cm \rho}{K^2} \Big( H [
	( \| \widehat{\bbeta}_{\tau(k)} - \bbeta^{\star}_{k} \|_2 - \eta / \rho )_+]
	- H (r)
	\Big)
	.
	\label{eqn-clustered-robustness-7}
	\end{align}
	The function $h(\cdot)$ was defined in \Cref{eqn-clustered-robustness-4}. The desired result follows from Claim \ref{claim-clustered-robustness-1}, \Cref{eqn-clustered-robustness-7} and the monotonicity of $H(\cdot)$.
\end{proof}

\begin{claim}\label{claim-9}
	$\tau:~[K] \to [K]$ is a permutation.
\end{claim}
\begin{proof}[\bf Proof of Claim \ref{claim-9}]
	It suffices to show that $\tau$ is a bijection and below we prove it by contradiction. If there are $k \neq l$ such that $\tau(k) = \tau(l) = t$, then $\lambda \leq   \rho \Delta / 6$ yields
	\[
	\max \{ \| \widehat{\bbeta}_{t} - \bbeta^{\star}_{l} \|_2, \| \widehat{\bbeta}_{t} - \bbeta^{\star}_{k} \|_2 \} \geq \Delta / 2 
	> \frac{2\lambda}{  \rho}.
	\]
On the other hand, we use \eqref{eqn-lambda}, $C \leq 1$ and \eqref{eqn-C} to get
\begin{align}
r + \eta / \rho = \frac{\alpha \lambda}{2C\kappa \rho } + \frac{\eta}{\rho}
\leq \frac{\alpha \lambda}{2C\kappa \rho } + \frac{\sqrt{C} \lambda}{8 \kappa K \rho}
= \frac{\lambda}{\kappa \rho} \bigg( \frac{\alpha}{2 C} + \frac{\sqrt{C}}{8 K} \bigg)
\leq  \frac{\lambda}{\kappa \rho} \bigg( \frac{3}{5} + \frac{1}{8} \bigg) < \frac{\lambda}{\kappa \rho}.
\label{eqn-claim-9}
\end{align}
The implication $\max \{ \| \widehat{\bbeta}_{t} - \bbeta^{\star}_{l} \|_2, \| \widehat{\bbeta}_{t} - \bbeta^{\star}_{k} \|_2 \} > r + \eta / \rho$ contradicts Claim \ref{claim-clustered-robustness-w}.
\end{proof}

\begin{claim}\label{claim-clustered-robustness-z-0}
For any $j \in S$, $\min_{k \neq \tau(z_j^{\star}) }\widetilde{f}_j( \widehat{\bbeta}_k  ) > \widetilde{f}_j( \widehat{\bbeta}_{\tau(z_j^{\star})} )$.
\end{claim}
\begin{proof}[\bf Proof of Claim \ref{claim-clustered-robustness-z-0}]
	Choose any $j \in S$. On the one hand, $\widetilde{f}_j (\bbeta) = \min_{\bxi} \{ f_j(\bbeta - \bxi) + \lambda \| \bxi \|_2 \} \leq f_j(\bbeta)$ for all $\bbeta$. Then
\begin{align}
\widetilde{f}_j (\widehat{\bbeta}_{\tau ( z_j^{\star} )})
& \leq f_j (\widehat{\bbeta}_{\tau ( z_j^{\star} )})
\leq  f_j (\widetilde{\bbeta}_j)  + \frac{L}{2} \| \widehat{\bbeta}_{\tau ( z_j^{\star} )} - \widetilde{\bbeta}_j \|_2^2 
\notag \\ & 
\leq  f_j (\widetilde{\bbeta}_j) + \frac{L}{2} \Big( \| \widehat{\bbeta}_{\tau ( z_j^{\star} )} - \bbeta^{\star}_{ z_j^{\star} } \|_2 
+ \| \bbeta^{\star}_{ z_j^{\star} }  - \widetilde{\bbeta}_j \|_2 
\Big)^2 
\leq f_j (\widetilde{\bbeta}_j) + \frac{L}{2} 
\bigg( r + 	\frac{ 2\eta}{\rho}  \bigg) ^2,
\label{eqn-clustered-robustness-8}
\end{align}
where the last inequality follows from Claim \ref{claim-clustered-robustness-w} and \Cref{lem-clustered-lower-bound}.

On the other hand, $\widetilde{f}_j(\bbeta) =  f_j (\bbeta)$, $\forall \bbeta \in B ( \widetilde{\bbeta}_j, r)$ according to \Cref{lem-clustered-lower-bound}. \Cref{lem-cvx-lower} and $H(x) \geq r (x - r/2)$, $\forall x \geq 0$ yield
\begin{align*}
\widetilde{f}_j( \widehat{\bbeta}_k  ) - f_j( \widetilde{\bbeta}_j ) \geq \rho H ( \| \widehat{\bbeta}_k  - \widetilde{\bbeta}_j \|_2 )
\geq \rho r (\| \widehat{\bbeta}_k  - \widetilde{\bbeta}_j \|_2 - r/2)
, \qquad\forall k \in [K].
\end{align*}
Choose any $k \neq \tau(z_j^{\star})$. We have $\tau^{-1}(k) \neq z_j^{\star}$ and
\begin{align*}
\| \widehat{\bbeta}_k - \widetilde{\bbeta}_j \|_2 & \geq \| \widetilde{\bbeta}_j - \bbeta^{\star}_{\tau^{-1}(k)} \|_2 - \| \bbeta^{\star}_{\tau^{-1}(k)} - \widehat\bbeta_k \|_2 \\
& \geq \|  \bbeta^{\star}_{\tau^{-1}(k)} - \bbeta^{\star}_{z_j^{\star}} \|_2 - \| \bbeta^{\star}_{z_j^{\star}} - \widetilde{\bbeta}_j  \|_2 - \| \bbeta^{\star}_{\tau^{-1}(k)} - \widehat\bbeta_k \|_2
\geq \Delta - \frac{ \eta}{\rho} - \bigg( r +  \frac{ \eta}{\rho} \bigg),
\end{align*}
where the last inequality follows from $\|  \bbeta^{\star}_{\tau^{-1}(k)} - \bbeta^{\star}_{z_j^{\star}} \|_2 \geq \Delta $, \Cref{lem-clustered-lower-bound} and Claim \ref{claim-clustered-robustness-w}. Consequently,
\begin{align}
\widetilde{f}_j( \widehat{\bbeta}_k  ) \geq f_j( \widetilde{\bbeta}_j ) 
+\rho r \bigg( \Delta - \frac{ 2\eta}{\rho} -   \frac{3r}{2} \bigg).
\label{eqn-clustered-robustness-9}
\end{align}

By \Cref{eqn-clustered-robustness-8,eqn-clustered-robustness-9},
\begin{align*}
\widetilde{f}_j( \widehat{\bbeta}_k  ) - \widetilde{f}_j (\widehat{\bbeta}_{\tau ( z_j^{\star} )})
\geq
\rho r \bigg( \Delta - \frac{ 2\eta}{\rho} -  \frac{3r}{2} \bigg) -  \frac{L}{2} 
\bigg( r + 	\frac{ 2 \eta}{\rho}  \bigg) ^2.
\end{align*}
According to \eqref{eqn-claim-9} and $\lambda < \rho \Delta / 6$,
\[
\Delta - \frac{ 2\eta}{\rho} -  \frac{3r}{2} \geq \Delta - 2 (r + \eta / \rho) \geq \Delta - \frac{ 2 \lambda }{\kappa \rho} > \frac{4 \lambda  }{ \rho }
\quad\text{and}\quad r + 	\frac{ 2 \eta}{\rho} \leq 2 (r + \eta / \rho) \leq \frac{ 2 \lambda }{\kappa \rho}.
\]
Since $r = \frac{\alpha \lambda}{2 C \kappa \rho}$, we finally get
\begin{align*}
\widetilde{f}_j( \widehat{\bbeta}_k  ) - \widetilde{f}_j (\widehat{\bbeta}_{\tau ( z_j^{\star} )})
>
\rho \cdot \frac{\alpha \lambda}{2 C \kappa \rho} \cdot \frac{4 \lambda}{\rho} -  \frac{L}{2} 
\bigg( \frac{ 2 \lambda }{\kappa \rho} \bigg) ^2
= \frac{2\lambda^2}{\kappa \rho} \bigg(
\frac{ \alpha}{C} - 1
\bigg) \geq 0.
\end{align*}
for all $k \neq \tau(z_j^{\star})$ and $j \in S$.
\end{proof}

\begin{claim}\label{claim-clustered-robustness-z}
	$\widehat{z}_j = \tau ( z_j^{\star} )$, $\forall j \in S$.
\end{claim}
\begin{proof}[\bf Proof of Claim \ref{claim-clustered-robustness-z}]
Define $\bar\bz \in [K]^m$ through
\begin{align*}
\bar{z}_j = \begin{cases}
\tau(z_j^{\star}) &,\mbox{ if } j \in S \\
\widehat{z}_j &,\mbox{ if } j \in S^c
\end{cases}.
\end{align*}
Claim \ref{claim-clustered-robustness-z-0} implies that $F(\widehat{\bB} , \bar\bz) \leq F(\widehat{\bB} , \widehat\bz)$ and the equality holds if and only if $\widehat{z}_j = \tau(z_j^{\star})$ holds for all $j \in S$.

By definition, $\widehat{\bz}$ is a solution to the constrained program
\begin{align}
\min_{ \bz \in [K]^m  } F (\widehat{\bB},  \bz) ,\qquad\text{s.t.}\qquad
\min_{k \in [K]} |\{ j \in [m] :~ z_j = k \}| \geq \alpha m / K  .
\label{eqn-constrained-program}
\end{align}
According to the balancedness assumption $\min_{k \in[K]} |T(k)|  \geq \frac{ \alpha m }{K} (1 + \frac{1}{6 \kappa K })$ and $\varepsilon \leq \alpha / (6 \kappa K^2)$, we have
\begin{align*}
\min_{k \in[K]} 
|\{ j \in [m] :~ \bar{z}_j = k \}| 
& \geq 
\min_{k \in[K]} 
|\{ j \in [m] :~ z^{\star}_j = k \}| - |S^c| \\
& \geq \frac{ \alpha m }{K} \bigg(1 + \frac{1}{6 \kappa K } \bigg) - \varepsilon m \geq \alpha m / K.
\end{align*}
Therefore, $\bar{\bz}$ is also feasible for the program \eqref{eqn-constrained-program}. In light of the $\widehat{\bz}$'s optimality, $\bar\bz$ must also be optimal and $F(\widehat{\bB} , \bar\bz) = F(\widehat{\bB} , \widehat\bz)$. We get $\widehat{z}_j = \tau(z_j^{\star})$ for all $j \in S$.
\end{proof}

\begin{claim}\label{claim-12}
	$\| \widehat{\bbeta}_{\tau(k)} - \widehat{\bbeta}_{k}^{\mathrm{oracle}} \|_2 \leq  \frac{ \lambda |S^c|     }{\rho |T(k)| }$ for all $k \in [K]$ and $\widehat{\btheta}_j = \widehat{\bbeta}_{k} $ for all $j \in S$.
\end{claim}
\begin{proof}[\bf Proof of Claim \ref{claim-12}]
	By definition, $\widehat{\bbeta}_{k} \in \argmin_{\bbeta \in \RR^d} \{ \sum_{j  :~ \widehat{z}_j = k } \widetilde{f}_j (\bbeta) \}$. Claim \ref{claim-clustered-robustness-z} yields
	\[
	\widehat{\bbeta}_{\tau(k)} \in \argmin_{\bbeta \in \RR^d}  \sum_{j \in Q(k) } \widetilde{f}_j (\bbeta) ,
	\]
	where we define $	Q(k)  = \{ j \in [m]  :~ \tau( \widehat{z}_j ) = k \} $. Note that
\begin{align*}
Q(k) \cap S  =  \{ j \in S:~ \tau( \widehat{z}_j ) = k \}    = \{ j \in S:~ z^{\star}_j =  k \} = T(k) \cap S 
\end{align*}
and $|T(k) \cap S| \geq |T(k)| - |S^c| \geq \alpha m / K$. We want to apply \Cref{thm-robustness} to $\{ f_j \}_{j \in Q(k)}$, which requires verifying the condition
	\begin{align}
	\lambda 
	> \frac{3 \kappa \max_{j \in Q(k) \cap S } \| \nabla f_j ( \bbeta^{\star}_{z^{\star}_j} ) \|_2 }{
		1 -  \kappa |S^c| / | Q(k) \cap S|
	}.
	\label{eqn-clustered-robustness-lbd}
	\end{align}
	On the one hand, the facts $\max_{j \in S } \| \nabla f_j ( \bbeta^{\star}_{z^{\star}_j} ) \|_2 \leq \eta$, $|Q(k) \cap S | \geq \alpha m / K$ and $|S^c| = \varepsilon m$ imply that
	\[
	\frac{3 \kappa \max_{j \in Q(k)\cap S } \| \nabla f_j ( \bbeta^{\star}_{z^{\star}_j} ) \|_2 }{
		1 -  \kappa |S^c| / |Q(k) \cap S|
	}
	\leq  \frac{3 \kappa \eta}{1 -  K \kappa \varepsilon / \alpha} .
	\]
	On the other hand, the assumption $\varepsilon \leq \frac{\alpha}{6 \kappa K^2}$ and \eqref{eqn-lambda} force
	\[
	\frac{3 \kappa \eta}{1 -  K \kappa \varepsilon / \alpha}  \leq  \frac{3 \kappa \eta}{1 - 1/6} = 18 \kappa \eta / 5 <  \lambda.
	\]
	and proves (\ref{eqn-clustered-robustness-lbd}). Hence \Cref{thm-robustness} and $Q(k) \cap S   = T(k) \cap S $ assert that
	\begin{align*}
	& \widehat{\btheta}_j = \widehat{\bbeta}_{k} , \qquad \forall j \in T(k) \cap S 
	\quad\text{and}\quad
	\bigg\| \widehat{\bbeta}_k - \argmin_{\bbeta \in \RR^d} \bigg\{  \sum_{j \in T(k) \cap S} f_j (\bbeta) \bigg\} \bigg\|_2 \leq \frac{   \lambda |S^c|   }{ \rho |T(k) \cap S| }.
	\end{align*}
	Then the proof is finished by combining the results for all $k \in [K]$.
\end{proof}

\subsection{Proof of \Cref{thm-armul-lowrank-deterministic}}\label{sec-thm-armul-lowrank-deterministic-proof}

We invoke the following lemma, whose proof is in \Cref{sec-lem-lowrank-proof}.

\begin{lemma}[Low-rank ARMUL]\label{lem-lowrank}
	Suppose that for $j \in [m]$, 
	\begin{align*}
	&\bm{0} \preceq \rho \bI \preceq \nabla^2 f_j (\btheta) \preceq L \bI , ~~ \forall \btheta \in B(\bB^{\star} \bz^{\star}_j, M) , \\
	& \| \nabla f_j ( \bB^{\star} \bz^{\star}_j ) \|_2 \leq \eta ,~~ \| \bP^{\star}  \nabla f_j (  \bB^{\star} \bz^{\star}_j )  \|_2 \leq \eta_0
	~~\text{and}~~ \| \bz^{\star}_j \|_2 \leq \alpha_1.
	\end{align*}
	In addition,
	\begin{align*}
	& \Big\| \Big( \nabla f_1 (\bB^{\star} \bz^{\star}_1) , \cdots, \nabla f_m (\bB^{\star} \bz^{\star}_m) \Big) \Big\|_2 \leq \tau
	\qquad\text{and}\qquad
	\sum_{j=1}^{m}   \bz_j^{\star} \bz_j^{\star \top} \succeq ( \alpha_2^2 m / K  ) \bI_{K } .
	\end{align*}
	There exist positive constants $C_0, C \geq 1$ such that when
	\[
	C_0 K   \eta \bigg( \frac{ L}{ \rho} \bigg)^{3/2}  \bigg( \frac{ \alpha_1 }{\alpha_2}  \bigg)^2
	< \lambda \leq LM,
	\]
	we have $\widehat{\bTheta} = \widehat{\bB} \widehat{\bZ}$, $\| \widehat{\bB} \widehat{\bZ}  - \bB^{\star} \bZ^{\star} \|_{\mathrm{F}} \leq 2\sqrt{2K} \tau / \rho$ and
	\begin{align*}
	& \max_{j \in [m]} \| \widehat{\bB} \widehat{\bz}_j  - \bB^{\star} \bz^{\star}_j   \|_2  \leq C \bigg[
	\frac{\eta_0}{\rho} +
	\frac{    K}{  \alpha_2  \sqrt{m} } \cdot \frac{\tau}{\rho} \bigg(
	\frac{\eta}{ \rho} + \alpha_1 \sqrt{\frac{L}{\rho}} 
	\bigg)
	\bigg].
	\end{align*}
\end{lemma}

We first use \Cref{lem-lowrank} to derive the following intermediate result. Let $C_0$ and $C$ be the constants defined therein.

\begin{claim}\label{thm-armul-lowrank-deterministic-claim}
	Define $g = \max_{j \in [m]} \| \nabla f_j (\btheta_j^{\star}) \|_2 $, $\alpha_1 = \max_{j \in [m] } \| \bz^{\star}_j \|_2 $, $\alpha_2 = \sqrt{K/m} \sigma_{\min} (\bZ^{\star})$ and $\mu = \alpha_1/\alpha_2$. Under the assumptions in \Cref{thm-armul-lowrank-deterministic}, if 
	\[
	2C_0 L \delta \leq g + \frac{\lambda }{ 2 C_0 \kappa^{3/2} \mu^2 K  } ,
	\]
	then $\widehat{\bTheta} = \widehat{\bB} \widehat\bZ$ and
	\begin{align*}
	\max_{j \in [m] } \| \widehat{\btheta}_j  - \btheta^{\star}_j \|_2  
	\leq C \bigg[ \frac{ \max_{j \in [m] }  \| \bP^{\star}  \bg_j \|_2  }{\rho} + \kappa \delta +
	\frac{    K}{  \alpha_2  } \bigg( \frac{ \| \bG \|_2 }{\sqrt{m}  \rho } + \kappa \delta \bigg) \bigg(
	\frac{ g }{ \rho} + \kappa \delta +  \sqrt{\kappa} \alpha_1 
	\bigg) \bigg]
	.
	\end{align*}
\end{claim}
\begin{proof}[\bf Proof of Claim \ref{thm-armul-lowrank-deterministic-claim}]
	We obtain from the assumption
	\begin{align}
	2 C_0   \kappa^{3/2}  \mu^2 K  \max_{j \in [m]} \| \nabla f_j (\btheta_j^{\star}) \|_2 
	< \lambda \leq \rho \min \{ M / 2,~\alpha_1 \}.
	\label{eqn-armul-deterministic-lowrank-lambda}
	\end{align}
	that $g \leq \frac{ \rho M }{4 C_0} $, $\lambda < L M /2$ and $g + \lambda < \frac{2C_0+1}{4 C_0} LM  $. When
	\[
	2C_0 L \delta \leq g + \frac{\lambda }{ 2 C_0  \kappa^{3/2} \mu^2 K  }  \leq g + \lambda,
	\]
	we have $ \delta < \frac{2C_0 + 1}{8 C_0^2} M \leq M/2$ since $C_0 \geq 1$. Thus $\max_{j \in [m]} \| \btheta^{\star}_j  - \bbeta^{\star}_{z_j^{\star}} \|_2 \leq M/2$; $\rho \bI \preceq \nabla^2 f_j(\btheta) \preceq L \bI$ holds for all $j \in [m]$ and $\btheta \in B( \bbeta^{\star}_{z_j^{\star}}, M/2 )$.

	For any $j$, the regularity condition $ \nabla^2 f_j (\btheta) \preceq L \bI$, $\forall \btheta \in B(\btheta_j^{\star}, M)$ leads to $ \| \nabla f_j (\btheta^{\star}_j)  - \nabla f_j (\bbeta^{\star}_{z_j^{\star}})  \|_2 \leq L \delta $. By triangle's inequality,
	\begin{align*}
	\| \nabla f_j (\bbeta^{\star}_{z_j^{\star}}) \|_2 
	&\leq   \| \nabla f_j (\btheta^{\star}_j) \|_2  +   \| \nabla f_j (\bbeta^{\star}_{z_j^{\star}})  - \nabla f_j (\btheta^{\star}_j)  \|_2   	\leq g + L \delta
	\\&
	<  \frac{(2C_0+1)g}{2C_0} + \frac{ \lambda }{ 4 C_0^2 \kappa^{3/2} \mu^2 K  }  .
	\end{align*}
	Consequently,
	\begin{align*}
	C_0 \kappa^{3/2} \mu^2 K \max_{j \in [m]}  \| \nabla f_j (\bbeta^{\star}_{z_j^{\star}}) \|_2 
	< (C_0 + 1/2) \kappa^{3/2} \mu^2 K g + \frac{ \lambda }{ 4 C_0 } 
	\overset{\mathrm{(i)}}{<}  \lambda  \overset{\mathrm{(ii)}}{<} L \cdot \frac{M}{2}.
	\end{align*}
	The inequalities $\mathrm{(i)}$ and $\mathrm{(ii)}$ follow from \eqref{eqn-armul-deterministic-lowrank-lambda} as well as $C_0 \geq 1$.
	Based on the above estimates, \Cref{lem-lowrank} asserts that $\widehat{\bTheta} = \widehat{\bB} \widehat{\bZ}$ and
	\begin{align*}
	& \max_{j \in [m]} \| \widehat{\bB} \widehat{\bz}_j  - \bB^{\star} \bz^{\star}_j   \|_2  \leq C \bigg[
	\frac{\eta_0}{\rho} +
	\frac{    K}{  \alpha_2  \sqrt{m} } \cdot \frac{\tau}{\rho} \bigg(
	\frac{\eta}{ \rho} + \alpha_1 \sqrt{\frac{L}{\rho}} 
	\bigg)
	\bigg].
	\end{align*}
	Here $\eta_0 = \max_{j \in [m] }  \| \bP^{\star}  \nabla f_j (  \bB^{\star} \bz^{\star}_j )  \|_2 $, $\eta =\max_{j \in [m] }  \|  \nabla f_j (  \bB^{\star} \bz^{\star}_j )  \|_2 $ and
	\[
	\tau = \Big\| \Big( \nabla f_1 (\bB^{\star} \bz^{\star}_1) , \cdots, \nabla f_m (\bB^{\star} \bz^{\star}_m) \Big) \Big\|_2 .
	\]
	
	By the triangle's inequality,
	\[
	\| \widehat{\btheta}_j  - \btheta^{\star}_j \|_2 \leq \| \widehat{\btheta}_j  - \bB^{\star} \bz^{\star}_j   \|_2 +  \| \bB^{\star} \bz^{\star}_j   - \btheta^{\star}_j \|_2 \leq \| \widehat{\bB} \widehat{\bz}_j  -\bB^{\star} \bz^{\star}_j \|_2  + \delta .
	\]
	Note that $\eta_0 \leq \max_{j \in [m] }  \| \bP^{\star}  \bg_j \|_2 + L \delta$, $\eta \leq \max_{j \in [m] }  \| \bg_j \|_2 + L \delta$ and
	\begin{align*}
	\tau & \leq \| \bG \|_2 + \Big\| \Big( \nabla f_1 (\bB^{\star} \bz^{\star}_1) -  \nabla f_1 ( \btheta^{\star}_1)  , \cdots, \nabla f_m (\bB^{\star} \bz^{\star}_m) 
	- \nabla f_m ( \btheta^{\star}_m)  
	\Big) \Big\|_2 \\
	& \leq \| \bG\|_2 + \sqrt{m} \max_{j \in [m] } \| \nabla f_j (\bB^{\star} \bz^{\star}_j) -  \nabla f_j ( \btheta^{\star}_j)  \|_2 \leq \| \bG \|_2 + \sqrt{m} L \delta.
	\end{align*}

	Based on the above estimates,
	\begin{align*}
	& \max_{j \in [m] } \| \widehat{\btheta}_j  - \btheta^{\star}_j \|_2 / C \\
	&  \leq 
	\frac{ \max_{j \in [m] }  \| \bP^{\star}  \bg_j \|_2 + L \delta }{\rho} +
	\frac{    K}{  \alpha_2  \sqrt{m} } \cdot \frac{ \| \bG \|_2 + \sqrt{m} L \delta}{\rho} \bigg(
	\frac{\max_{j \in [m] }  \| \bg_j \|_2 + L \delta}{ \rho} +  \sqrt{\kappa} \alpha_1 
	\bigg) \\
	& = \frac{ \max_{j \in [m] }  \| \bP^{\star}  \bg_j \|_2  }{\rho} + \kappa \delta +
	\frac{    K}{  \alpha_2  } \bigg( \frac{ \| \bG \|_2 }{\sqrt{m}  \rho } + \kappa \delta \bigg) \bigg(
	\frac{ g }{ \rho} + \kappa \delta +  \sqrt{\kappa} \alpha_1 
	\bigg) 
	.
	\end{align*}
\end{proof}

We now come back to \Cref{thm-armul-lowrank-deterministic}. The condition \eqref{eqn-armul-deterministic-lowrank-lambda} forces $\lambda > 	2 C_0   \kappa^{3/2}  \mu^2 K  g$.
Claim \ref{thm-armul-lowrank-deterministic-claim} implies that when 
$	2C_0 L \delta \leq g + \frac{\lambda }{ 2 C_0 \kappa^{3/2} \mu^2 K  }   $,
\begin{align*}
\max_{j \in [m] } \| \widehat{\btheta}_j  - \btheta^{\star}_j \|_2  
\leq C \bigg[ \frac{ \max_{j \in [m] }  \| \bP^{\star}  \bg_j \|_2  }{\rho} + \kappa \delta +
\frac{    K}{  \alpha_2  } \bigg( \frac{ \| \bG \|_2 }{\sqrt{m}  \rho } + \kappa \delta \bigg) \bigg(
\frac{ g }{ \rho} + \kappa \delta +  \sqrt{\kappa} \alpha_1 
\bigg) \bigg]
.
\end{align*}
Note that
\[
\kappa \delta = L \delta / \rho \leq \frac{1}{2 C_0 \rho} \bigg( g + \frac{\lambda }{ 2 C_0 \kappa^{3/2} \mu^2 K  } \bigg) \leq \frac{g}{2 \rho} + \frac{\lambda }{ 4 \kappa^{3/2} \mu^2 K \rho  }. 
\]
Consequently,
\begin{align*}
\max_{j \in [m] } \| \widehat{\btheta}_j  - \btheta^{\star}_j \|_2  
& \leq C \bigg[ \frac{ \max_{j \in [m] }  \| \bP^{\star}  \bg_j \|_2  }{\rho} + \kappa \delta +
\frac{    K}{  \alpha_2  } \bigg( \frac{ \| \bG \|_2 }{\sqrt{m}  \rho } + \kappa \delta \bigg) \bigg(
\frac{ 2g }{ \rho} \\
&~~~~ + \frac{\lambda }{ 4 \kappa^{3/2} \mu^2 K \rho  } +  \sqrt{\kappa} \alpha_1 
\bigg) \bigg]
.
\end{align*}
Based on \eqref{eqn-armul-deterministic-lowrank-lambda}, $2 g / \rho \leq \alpha_1 $ and $\frac{\lambda }{ 4 \kappa^{3/2} \mu^2 K \rho  } \leq \alpha_1$. Hence
\begin{align*}
& \max_{j \in [m] } \| \widehat{\btheta}_j  - \btheta^{\star}_j \|_2   \leq C \bigg[ \frac{ \max_{j \in [m] }  \| \bP^{\star}  \bg_j \|_2  }{\rho} + \kappa \delta +
\frac{    K}{  \alpha_2  } \bigg( \frac{ \| \bG \|_2 }{\sqrt{m}  \rho } + \kappa \delta \bigg) \cdot 3\sqrt{\kappa} \alpha_1 
\bigg] \\
& \leq C \bigg( \frac{ \max_{j \in [m] }  \| \bP^{\star}  \bg_j \|_2  }{\rho} + \frac{3 \kappa^{3/2} \mu K  \| \bG \|_2 }{\sqrt{m}} + 4 \kappa^{3/2} \mu K  \delta
\bigg) \\
& = C \bigg[  
\frac{ \max_{j \in [m] }  \| \bP^{\star}  \bg_j \|_2  }{\rho} + \frac{3 \kappa^{3/2} \mu K  \| \bG \|_2 }{\sqrt{m}} + 4 \kappa^{3/2} \mu K \min\bigg\{  \delta,~ 
\frac{1}{2C_0 L} \bigg(
g \\
&~~~~ + \frac{\lambda }{ 2 C_0 \kappa^{3/2} \mu^2 K  }  
\bigg)
\bigg\}
\bigg]  .
\end{align*}
The last step follows from the assumption $	2C_0 L \delta \leq g + \frac{\lambda }{ 2 C_0 \kappa^{3/2} \mu^2 K  }  $.

On the other hand, when $2C_0 L \delta > g + \frac{\lambda }{ 2 C_0 \kappa^{3/2} \mu^2 K  } $, we use \Cref{thm-armul-personalization} to get 
\begin{align*}
\max_{j \in [m]} \| \widehat{\btheta}_j  - \btheta^{\star}_j \|_2 & \leq \frac{g + \lambda}{\rho } \leq \bigg( \frac{1}{ 	2 C_0   \kappa^{3/2}  \mu^2 K  } + 1 \bigg)\frac{\lambda}{\rho  } \leq \frac{3 \lambda}{2 \rho  } \\
& = \frac{\lambda }{ 2 C_0 \kappa^{3/2} \mu^2 K  }  \cdot  \frac{ 3C_0 \kappa^{3/2} \mu^2 K }{\rho} \\
& \leq  \frac{ 3C_0 \kappa^{3/2} \mu^2 K }{\rho}    \bigg(
g + \frac{\lambda }{ 2 C_0 \kappa^{3/2} \mu^2 K  }  
\bigg)  \\
& = \frac{ 3C_0 \kappa^{3/2} \mu^2 K }{\rho} \cdot  \min\bigg\{ 2 C_0 L  \delta,~ 
g + \frac{\lambda }{ 2 C_0 \kappa^{3/2} \mu^2 K  }  
\bigg\}  \\
& =    6C_0^2 \kappa^{5/2} \mu^2 K  \cdot 
\min\bigg\{  \delta,~ 
\frac{1}{2C_0 L} \bigg(
g + \frac{\lambda }{ 2 C_0 \kappa^{3/2} \mu^2 K  }  
\bigg)
\bigg\}.
\end{align*}
On top of the above and $C \geq 1$, we have
\begin{align*}
& \max_{j \in [m] } \| \widehat{\btheta}_j  - \btheta^{\star}_j \|_2  /C \\
& \leq 
\frac{ \max_{j \in [m] }  \| \bP^{\star}  \bg_j \|_2  }{\rho} + \frac{3 \kappa^{3/2} \mu K  \| \bG \|_2 }{\sqrt{m}} +  6C_0^2 \kappa^{5/2} \mu^2 K  \min\bigg\{  \delta,~ 
\frac{1}{2C_0 L} \bigg(
g  \\
&~~~~ + \frac{\lambda }{ 2 C_0 \kappa^{3/2} \mu^2 K  }  
\bigg)
\bigg\} \\
& \leq 
\frac{ \max_{j \in [m] }  \| \bP^{\star}  \bg_j \|_2  }{\rho} + \frac{3 \kappa^{3/2} \mu K  \| \bG \|_2 }{\sqrt{m}} +  6C_0^2 \kappa^{5/2} \mu^2 K  \min\bigg\{  \delta,~ 
\frac{1}{2C_0 L}  \cdot  \frac{\lambda }{  C_0 \kappa^{3/2} \mu^2 K  }  
\bigg\} \\
& \leq 
\frac{ \max_{j \in [m] }  \| \bP^{\star}  \bg_j \|_2  }{\rho} + \frac{3 \kappa^{3/2} \mu K  \| \bG \|_2 }{\sqrt{m}} +  \min\bigg\{  6C_0^2 \kappa^{5/2} \mu^2 K  \delta,~ 
\frac{3\lambda }{ \rho }  
\bigg\}
.
\end{align*}

Finally, if $ \delta \leq  \frac{\lambda }{ 4 C_0^2 \kappa^{3/2} \mu^2 K L } $, then we use Claim \ref{thm-armul-lowrank-deterministic-claim} to get $\widehat{\bTheta} = \widehat{\bB} \widehat\bZ$.

\subsection{Proof of \Cref{lem-lowrank}}\label{sec-lem-lowrank-proof}

As a matter of fact, we have
\begin{align}
& (\widehat{\bB} , \widehat\bZ) \in \argmin_{ 
\bB \in \RR^{d\times K}, ~\bZ \in \RR^{K \times m}
} \bigg\{ 
\sum_{j=1}^{m}  \widetilde{f}_j (  \bB \bz_j )
\bigg\},
\label{eqn-lowrank-mtl-inf} \\
&\widehat{\btheta}_j \in \argmin_{\btheta \in \RR^d} \{ f_j (\btheta) + \lambda \| \widehat{\bB} \widehat\bz_j - \btheta \|_2 \}, \label{eqn-lowrank-mtl-theta} 
\end{align}
where $\widetilde{f}_j = f_j \square (\lambda \| \cdot \|_2)$. Therefore,
\begin{align}
&\widehat{\bz}_j \in \argmin_{\bz \in \RR^K} \widetilde{f}_j (  \widehat{\bB} \bz ) ,
\label{eqn-lowrank-mtl-1} \\
& \sum_{j=1}^{m}  \widetilde{f}_j (  \widehat{\bB}  \widehat{\bz}_j ) \leq \sum_{j=1}^{m}  \widetilde{f}_j (  \bB^{\star} \bz^{\star}_j ) .
\label{eqn-lowrank-mtl-2}
\end{align}

When $\lambda$ is sufficiently large, $\widetilde{f}_j = f_j$ in a neighborhood of $\bB^{\star} \bz^{\star}_j$. The strong convexity of $f_j$ therein implies the same property of $\widetilde{f}_j$. In Claims \ref{claim-lowrank-subspace}, \ref{claim-lowrank-z} and \ref{claim-lowrank-crude} we prove that $\widehat{\bB} \widehat{\bz}_j $ lives in that ``nice'' neighborhood of $\bB^{\star} \bz^{\star}_j$. Then, in Claims \ref{claim-lowrank-subspace-sharp} and \ref{claim-lowrank-z-sharp} we leverage the strong convexity of $\widetilde{f}_j$ to control $\| \widehat{\bB} \widehat{\bZ} - \bB^{\star} \bZ^{\star} \|_{\mathrm{F}}$ and get sharper bounds on $\max_{j \in [m]} \| \widehat{\bB} \widehat{\bz}_j - \bB^{\star} \bz^{\star}_j \|_2$. The relation $\widehat{\bTheta} = \widehat{\bB} \widehat{\bZ}$ is proved by Claim \ref{claim-lowrank-crude}.

\begin{claim}\label{claim-lowrank-subspace}
Let $\widehat{\bP}$ be the projection onto $\Range(\widehat{\bB})$. When 
\[
\frac{L\eta}{\rho} \bigg( 3 + \frac{2 \sqrt{K } \alpha_1 }{\alpha_2}  \bigg)
\leq \lambda \leq LM
\]
we have
\[
\| (\bI - \widehat{\bP}) \bB^{\star} \|_2 \leq  \frac{ \sqrt{2 K } \eta}{ \alpha_2 \rho } \bigg( 1 + \frac{2 \sqrt{K } \alpha_1 }{\alpha_2} \bigg)
\]
\end{claim}

\begin{proof}[\bf Proof of Claim \ref{claim-lowrank-subspace}]
Define $r = \frac{  \eta}{ \rho } \cdot \frac{2 \sqrt{K } \alpha_1 }{\alpha_2} $. Since $3 \eta L / \rho + Lr \leq  \lambda \leq LM$, \Cref{lem-clustered-lower-bound} and \eqref{eqn-lowrank-mtl-2} lead to
\begin{align*}
0 \geq \sum_{j=1}^{m} \widetilde{f}_j ( \widehat{\bB} \widehat{\bz}_{ j}  ) - \sum_{j=1}^{m} \widetilde{f}_j ( \bB^{\star} \bz^{\star}_{j} ) \geq
\rho \sum_{j =1}^m H \Big(
(\| \widehat{\bB} \widehat{\bz}_{ j} - \bB^{\star} \bz^{\star}_{j} \|_2 -  \eta/\rho
)_{+} 
\Big) 
- \frac{ m \eta^2}{\rho}.
\end{align*}
Here $H(t) = t^2 / 2$ if $0 \leq t \leq r$ and $H(t) = r(t-r/2)$ if $t > r$. By the monotonicity of $H(\cdot)$ and
\begin{align*}
\|  \widehat{\bB} \widehat{\bz}_j - \bB^{\star} \bz^{\star}_j \|_2 \geq 
\| (\bI - \widehat{\bP}) ( \widehat{\bB} \widehat{\bz}_j - \bB^{\star} \bz^{\star}_j ) \|_2 
= \| (\bI - \widehat{\bP}) \bB^{\star} \bz^{\star}_j \|_2 ,
\end{align*}
we have
\begin{align}
\sum_{j =1}^m H \Big(
[ \| (\bI - \widehat{\bP}) \bB^{\star} \bz^{\star}_j \|_2 -  \eta/\rho
]_{+} 
\Big) 
\leq  \frac{ m \eta^2}{\rho^2}.
\label{eqn-lowrank-lower-1}
\end{align}

We will prove the claim by contradiction. Suppose that 
\begin{align}
\| (\bI - \widehat{\bP}) \bB^{\star} \|_2 >  \frac{ \sqrt{2 K } \eta}{ \alpha_2 \rho } \bigg( 1 + \frac{2 \sqrt{K } \alpha_1 }{\alpha_2} \bigg)
\label{eqn-lowrank-cond}
\end{align}
and define
\[
T = \{ j \in [m]:~ \| (\bI - \widehat{\bP}) \bB^{\star} \bz^{\star}_j \|_2 > \alpha_2 \| (\bI - \widehat{\bP}) \bB^{\star} \|_2 / \sqrt{2K } \}.
\]
We have
\begin{align*}
&  \| (\bI - \widehat{\bP}) \bB^{\star} \bz^{\star}_j \|_2 >\frac{  \eta}{ \rho } \bigg( 1 + \frac{2 \sqrt{K } \alpha_1 }{\alpha_2}  \bigg) , \qquad \forall j \in T , \\
& \sum_{j =1}^m H \Big(
[ \| (\bI - \widehat{\bP}) \bB^{\star} \bz^{\star}_j \|_2 -  \eta/\rho
]_{+} 
\Big) 
> |T| \cdot H\bigg(
\frac{  \eta}{ \rho } \cdot \frac{2 \sqrt{K } \alpha_1 }{\alpha_2}  \bigg).
\end{align*}
Recall that $r = \frac{  \eta}{ \rho } \cdot \frac{2 \sqrt{K } \alpha_1 }{\alpha_2}$. Hence
\begin{align*}
H\bigg(
\frac{  \eta}{ \rho } \cdot \frac{2 \sqrt{K } \alpha_1 }{\alpha_2}  \bigg)
=  \frac{1}{2} \bigg(
\frac{  \eta}{ \rho } \cdot \frac{2 \sqrt{K } \alpha_1 }{\alpha_2}  \bigg)^2.
\end{align*}

Meanwhile, \Cref{lem-frac} yields
\[
|T| / m \geq \frac{\alpha_2^2 / K - (\alpha_2/\sqrt{2 K })^2}{\alpha_1^2 - (\alpha_2/\sqrt{2 K })^2} = \frac{1}{ 2 (\sqrt{ K  } \alpha_1 / \alpha_2 )^2 - 1} \geq  \frac{1}{ 2 (\sqrt{ K  } \alpha_1 / \alpha_2 )^2 }.
\]
As a result,
\begin{align}
 \sum_{j =1}^m H \Big(
[ \| (\bI - \widehat{\bP}) \bB^{\star} \bz^{\star}_j \|_2 -  \eta/\rho
]_{+} 
\Big) > \frac{m}{ 2 (\sqrt{ K  } \alpha_1 / \alpha_2 )^2 } \cdot \frac{1}{2} \bigg(
\frac{  \eta}{ \rho } \cdot \frac{2 \sqrt{K } \alpha_1 }{\alpha_2}  \bigg)^2 = \frac{m \eta^2}{\rho^2}.
\end{align}
This strict lower bound contradicts \Cref{eqn-lowrank-lower-1}. Therefore, the condition (\ref{eqn-lowrank-cond}) does not hold.
\end{proof}

\begin{claim}\label{claim-lowrank-z}
There exists a constant $c >0$ such that when
\[
\frac{\eta L }{\rho} \bigg[ 3 +  c K   \sqrt{ \frac{ L}{ \rho} } \bigg( \frac{ \alpha_1 }{\alpha_2} \bigg)^2 \bigg]  \leq  \lambda \leq LM,
\]
we have
\[
 \| \widehat{\bB} \widehat{\bz}_j - \widehat{\bP} \bB^{\star} \bz^{\star}_j \|_2 \leq  \frac{c K   \eta}{\rho} \sqrt{ \frac{ L}{ \rho} } \bigg( \frac{ \alpha_1 }{\alpha_2} \bigg)^2 .
\]
\end{claim}
\begin{proof}[\bf Proof of Claim \ref{claim-lowrank-z}]
Since $\widehat{\bP}$ is the projection onto $\Range(\widehat{\bB})$, there exists $\bu_j \in \RR^{K}$ such that $\widehat{\bB} \bu_j = \widehat{\bP} \bB^{\star} \bz^{\star}_j $. According to \eqref{eqn-lowrank-mtl-1},
\begin{align}
\widetilde{f}_j (\widehat{\bB} \widehat{\bz}_j ) \leq \widetilde{f}_j (\widehat{\bB} \bu_j ) = \widetilde{f}_j ( \widehat{\bP} \bB^{\star} \bz^{\star}_j ) \leq f_j ( \widehat{\bP} \bB^{\star} \bz^{\star}_j ) .
\label{eqn-lowrank-mtl-10}
\end{align}
Note that
\[
\frac{\eta L }{\rho} \bigg[ 3 +  c K   \sqrt{ \frac{ L}{ \rho} } \bigg( \frac{ \alpha_1 }{\alpha_2} \bigg)^2 \bigg]  \leq  \lambda \leq LM.
\]
When $c$ is large enough, the assumptions in Claim \ref{claim-lowrank-subspace} are satisfied. Then Claim \ref{claim-lowrank-subspace} yields
\[
\| \widehat{\bP} \bB^{\star} \bz^{\star}_j -  \bB^{\star} \bz^{\star}_j \|_2 \leq \| (\bI - \widehat{\bP} ) \bB^{\star} \|_2 \| \bz^{\star}_j \|_2 \leq \frac{ \sqrt{2 K } \alpha_1 }{ \alpha_2 } \bigg( 1 + \frac{2 \sqrt{K } \alpha_1 }{\alpha_2} \bigg) \frac{\eta}{\rho} \leq \frac{5 K  \alpha_1^2}{\alpha_2^2} \cdot \frac{\eta}{\rho}.
\]
When $c \geq 5$, we have $M \geq \frac{5 K  \alpha_1^2}{\alpha_2^2} \cdot \frac{\eta}{\rho}$. Then $\widehat{\bP} \bB^{\star} \bz^{\star}_j \in B( \bB^{\star} \bz^{\star}_j , M )$ and
\begin{align*}
& f_j (\widehat{\bP} \bB^{\star} \bz^{\star}_j) - f_j ( \bB^{\star} \bz^{\star}_j )  \leq \langle \nabla f_j (  \bB^{\star} \bz^{\star}_j ) , (\widehat{\bP} - \bI)  \bB^{\star} \bz^{\star}_j \rangle + \frac{L}{2} \| (\widehat{\bP} - \bI)  \bB^{\star} \bz^{\star}_j \|_2^2 \\
& \leq \bigg( \|  \nabla f_j (  \bB^{\star} \bz^{\star}_j ) \|_2 + \frac{L}{2} \| ( \widehat{\bP} - \bI)  \bB^{\star} \bz^{\star}_j \|_2 \bigg)  \| (\widehat{\bP} - \bI)  \bB^{\star} \bz^{\star}_j \|_2 \\
& \leq \bigg( \eta + \frac{L}{2} \frac{5 K  \alpha_1^2}{\alpha_2^2} \cdot \frac{\eta}{\rho} \bigg)  \frac{5 K  \alpha_1^2}{\alpha_2^2} \cdot \frac{\eta}{\rho}
\leq 20 \cdot \frac{\eta^2}{\rho} \cdot \frac{K^{2} L}{ \rho} \bigg( \frac{ \alpha_1 }{\alpha_2} \bigg)^4 .
\end{align*}
This inequality and \eqref{eqn-lowrank-mtl-10} lead to
\begin{align}
\widetilde{f}_j (\widehat{\bB} \widehat{\bz}_j) - f_j ( \bB^{\star} \bz^{\star}_j ) 
\leq f_j (\widehat{\bP} \bB^{\star} \bz^{\star}_j) - f_j ( \bB^{\star} \bz^{\star}_j ) 
\leq 20 \frac{\eta^2}{\rho} \cdot \frac{K^{2} L}{ \rho} \bigg( \frac{ \alpha_1 }{\alpha_2} \bigg)^4 .
\label{eqn-lowrank-lower-10}
\end{align}

On the other hand, let
\[
r = \frac{c K   \eta}{\rho} \sqrt{ \frac{ L}{ \rho} } \bigg( \frac{ \alpha_1 }{\alpha_2} \bigg)^2 .
\]
Then $3 \eta L/\rho + Lr \leq  \lambda \leq LM$. \Cref{lem-clustered-lower-bound} applied to $f_j(\cdot)$ yields
\begin{align*}
\widetilde{f}_j (\widehat{\bB} \widehat{\bz}_j) - f_j ( \bB^{\star} \bz^{\star}_j ) \geq \rho \cdot H \Big(
( \| \widehat{\bB} \widehat{\bz}_j - \bB^{\star} \bz^{\star}_j \|_2 - \eta / \rho )_+
\Big) - \frac{\eta^2}{\rho}.
\end{align*}
Here $H(t) = t^2 / 2$ if $0 \leq t \leq r$ and $H(t) = r(t-r/2)$ if $t > r$. It is easily seen that
\begin{align*}
\| \widehat{\bB} \widehat{\bz}_j - \bB^{\star} \bz^{\star}_j \|_2	
\geq \| \widehat{\bP} ( \widehat{\bB} \widehat{\bz}_j - \bB^{\star} \bz^{\star}_j ) \|_2	
= \| \widehat{\bB} \widehat{\bz}_j - \widehat{\bP} \bB^{\star} \bz^{\star}_j \|_2
\end{align*}

We complete the proof by contradiction. Suppose that $\| \widehat{\bB} \widehat{\bz} - \widehat{\bP} \bB^{\star} \bz^{\star}_j \|_2 > r$. The monotonicity of $H(\cdot)$ forces
\begin{align*}
\widetilde{f}_j (\widehat{\bB} \widehat{\bz}_j) - f_j ( \bB^{\star} \bz^{\star}_j ) > \rho H(r - \eta / \rho) - \frac{\eta^2}{\rho} = \frac{\rho (r - \eta / \rho)^2}{2} - \frac{\eta^2}{\rho}.
\end{align*}
When $c$ is sufficiently large, this lower bound will contradict \Cref{eqn-lowrank-lower-10}. In that case, we must have $\| \widehat{\bB} \widehat{\bz}_j - \widehat{\bP} \bB^{\star} \bz^{\star}_j \|_2 \leq r$.
\end{proof}

\begin{claim}\label{claim-lowrank-crude}
Define $\xi =  \frac{K   \eta}{\rho} \sqrt{ \frac{ L}{ \rho} }  ( \frac{ \alpha_1 }{\alpha_2}  )^2$. There exist positive constants $C_1$ and $C_2$ such that when
\begin{align}
 C_1 L \xi  \leq  \lambda \leq LM,
 \label{eqn-lowrank-lambda}
\end{align} 
we have $\| \widehat{\bB} \widehat{\bz}_j -  \bB^{\star} \bz^{\star}_j \|_2 \leq C_2 \xi \leq M$ and $\widehat{\btheta}_j = \widehat{\bB} \widehat{\bz}_j $. In addition, $\widetilde{f}_j = f_j$ in $B( \bB^{\star} \bz^{\star}_j , C_2 \xi )$.
\end{claim}

\begin{proof}[\bf Proof of Claim \ref{claim-lowrank-crude}]
Let \eqref{eqn-lowrank-lambda} hold. When $C_1$ is sufficiently large, the assumptions in Claims \ref{claim-lowrank-subspace} and \ref{claim-lowrank-z} are satisfied. Then
\begin{align*}
& \| (\bI - \widehat{\bP}) \bB^{\star} \|_2 \leq  \frac{ \sqrt{2 K } \eta}{ \alpha_2 \rho } \bigg( 1 + \frac{2 \sqrt{K } \alpha_1 }{\alpha_2} \bigg) , \\
&\| \widehat{\bP} ( \widehat{\bB} \widehat{\bz}_j -  \bB^{\star} \bz^{\star}_j ) \|_2 =
\| \widehat{\bB} \widehat{\bz}_j - \widehat{\bP} \bB^{\star} \bz^{\star}_j \|_2 \leq  \frac{c K   \eta}{\rho} \sqrt{ \frac{ L}{ \rho} } \bigg( \frac{ \alpha_1 }{\alpha_2} \bigg)^2 .
\end{align*}
Here $c$ is the constant in Claim \ref{claim-lowrank-z}. Note that
\[
\| (\bI - \widehat{\bP}) ( \widehat{\bB} \widehat{\bz}_j -  \bB^{\star} \bz^{\star}_j ) \|_2 
= \| (\bI - \widehat{\bP})  \bB^{\star} \bz^{\star}_j \|_2 
\leq \| (\bI - \widehat{\bP})  \bB^{\star} \|_2 \| \bz^{\star}_j \|_2 \leq \alpha_1 \| (\bI - \widehat{\bP}) \bB^{\star} \|_2 .
\]
Then, the claimed bound on $\| \widehat{\bB} \widehat{\bz}_j -  \bB^{\star} \bz^{\star}_j \|_2$ follows from
\[
\| \widehat{\bB} \widehat{\bz}_j -  \bB^{\star} \bz^{\star}_j \|_2^2 = 
\| \widehat{\bP} ( \widehat{\bB} \widehat{\bz}_j -  \bB^{\star} \bz^{\star}_j ) \|_2^2 +
\| (\bI - \widehat{\bP}) ( \widehat{\bB} \widehat{\bz}_j -  \bB^{\star} \bz^{\star}_j ) \|_2^2 .
\]
The other claims are implied by \eqref{eqn-lowrank-mtl-theta}, \Cref{lem-inf-conv} and \Cref{lem-clustered-lower-bound}.
\end{proof}

\begin{claim}\label{claim-lowrank-subspace-sharp}
Under the conditions in Claim \ref{claim-lowrank-crude}, there exists a constant $C_3>0$ such that
\[
\| \widehat{\bB} \widehat{\bZ}  - \bB^{\star} \bZ^{\star} \|_{\mathrm{F}} \leq \frac{2\sqrt{2K} \tau}{\rho} \qquad\text{and}\qquad
\| \widehat{\bP} - \bP^{\star} \|_{\mathrm{F}} \leq \frac{C_3 K \tau}{\alpha_2 \rho \sqrt{m}}.
\]
\end{claim}

\begin{proof}[\bf Proof of Claim \ref{claim-lowrank-subspace-sharp}]
By Claim \ref{claim-lowrank-crude} and the strong convexity of $f_j$ near $\bB^{\star} \bz^{\star}_j$,
\begin{align*}
\widetilde{f}(\widehat{\bB} \widehat{\bz}_j) - \widetilde{f} (\bB^{\star} \bz^{\star}_j)
& = f (\widehat{\bB} \widehat{\bz}_j) - f (\bB^{\star} \bz^{\star}_j)  \\
& \geq 
\langle \nabla f_j (\bB^{\star} \bz^{\star}_j ) , \widehat{\bB} \widehat{\bz}_j - \bB^{\star} \bz^{\star}_j \rangle + \frac{\rho}{2} \| \widehat{\bB} \widehat{\bz}_j - \bB^{\star} \bz^{\star}_j \|_2^2, \qquad\forall j \in [m].
\end{align*}
Therefore
\begin{align*}
0 &\geq \sum_{j=1}^{m}
\widetilde{f}(\widehat{\bB} \widehat{\bz}_j) - \sum_{j=1}^{m} \widetilde{f} (\bB^{\star} \bz^{\star}_j) \\
& \geq 
\Big\langle\Big( \nabla f_1 (\bB^{\star} \bz^{\star}_1) , \cdots, \nabla f_m (\bB^{\star} \bz^{\star}_m) \Big) , \widehat{\bB} \widehat{\bZ}  - \bB^{\star} \bZ^{\star}  \Big\rangle + \frac{\rho}{2} \| \widehat{\bB} \widehat{\bZ}  - \bB^{\star} \bZ^{\star} \|_{\mathrm{F}}^2 \\
& \geq -  \Big\| \Big( \nabla f_1 (\bB^{\star} \bz^{\star}_1) , \cdots, \nabla f_m (\bB^{\star} \bz^{\star}_m) \Big) \Big\|_2 \| \widehat{\bB} \widehat{\bZ}  - \bB^{\star} \bZ^{\star} \|_{*}
+\frac{\rho}{2} \| \widehat{\bB} \widehat{\bZ}  - \bB^{\star} \bZ^{\star} \|_{\mathrm{F}}^2 \\
& \geq - \tau \sqrt{2 K} \| \widehat{\bB} \widehat{\bZ}  - \bB^{\star} \bZ^{\star} \|_{\mathrm{F}} +\frac{\rho}{2} \| \widehat{\bB} \widehat{\bZ}  - \bB^{\star} \bZ^{\star} \|_{\mathrm{F}}^2 \\
& = \frac{\rho}{2} \| \widehat{\bB} \widehat{\bZ}  - \bB^{\star} \bZ^{\star} \|_{\mathrm{F}}
\bigg(
\| \widehat{\bB} \widehat{\bZ}  - \bB^{\star} \bZ^{\star} \|_{\mathrm{F}} - \frac{2\sqrt{2K} \tau}{\rho}
\bigg).
\end{align*}
Here we used the fact that $\mathrm{rank} (\widehat{\bB} \widehat{\bZ}  - \bB^{\star} \bZ^{\star} ) \leq 2 K$. We have
\[
\| \widehat{\bB} \widehat{\bZ}  - \bB^{\star} \bZ^{\star} \|_{\mathrm{F}} \leq \frac{2\sqrt{2K} \tau}{\rho}.
\]

Note that $\bB^{\star} \in \cO_{d,K }$. Then,
\[
(\bB^{\star} \bZ^{\star})^{\top} (\bB^{\star} \bZ^{\star} ) = \bZ^{\star\top} \bZ^{\star} = \sum_{j=1}^{m} \bz^{\star}_j \bz^{\star\top}_j \succeq \frac{ \alpha_2^2 m}{K} \bI_K.
\]
The matrix $\bB^{\star} \bZ^{\star}$ has $K$ positive singular values and they are no smaller than $\alpha_2 \sqrt{m/K}$. The proof is finished by Wedin's theorem \citep{Wed72}.
\end{proof}

\begin{claim}\label{claim-lowrank-z-sharp}
We have
\[
\max_{j \in [m]} \| \widehat{\bB} \widehat{\bz}_j  - \bB^{\star} \bz^{\star}_j   \|_2  \lesssim \max\bigg\{
\frac{\eta_0}{\rho} ,~ 
\frac{   \alpha_1  K}{  \alpha_2  \sqrt{m} } \cdot \frac{\tau}{\rho} \bigg(
\frac{\eta}{\rho} + \sqrt{\frac{L}{\rho}} 
\bigg)
\bigg\}.
\]
\end{claim}
\begin{proof}[\bf Proof of Claim \ref{claim-lowrank-z-sharp}]
By Claim \ref{claim-lowrank-crude} and \Cref{eqn-lowrank-mtl-10}, $f_j (\widehat{\bB} \widehat{\bz}_j ) =
\widetilde{f}_j (\widehat{\bB} \widehat{\bz}_j ) \leq f_j ( \widehat{\bP} \bB^{\star} \bz^{\star}_j ) $.
By the strong convexity and smoothness of $f_j$ near $\bB^{\star} \bz^{\star}_j$,
\begin{align*}
& f_j ( \widehat{\bP} \bB^{\star} \bz^{\star}_j )  \leq f_j (  \bB^{\star} \bz^{\star}_j ) 
+ \langle \nabla f_j (  \bB^{\star} \bz^{\star}_j )  , (\widehat{\bP} - \bI) \bB^{\star} \bz^{\star}_j  \rangle + \frac{L}{2} \| (\widehat{\bP} - \bI) \bB^{\star} \bz^{\star}_j   \|_2^2 ,\\
& f_j (\widehat{\bB} \widehat{\bz}_j ) \geq f_j (  \bB^{\star} \bz^{\star}_j ) 
+ \langle \nabla f_j (  \bB^{\star} \bz^{\star}_j )  ,\widehat{\bB} \widehat{\bz}_j - \bB^{\star} \bz^{\star}_j  \rangle + \frac{\rho}{2} \| \widehat{\bB} \widehat{\bz}_j  - \bB^{\star} \bz^{\star}_j   \|_2^2.
\end{align*}
Combining the estimates above, we get
\begin{align}
&\frac{\rho}{2} \| \widehat{\bB} \widehat{\bz}_j  - \bB^{\star} \bz^{\star}_j   \|_2^2 +
\langle \nabla f_j (  \bB^{\star} \bz^{\star}_j )  , \widehat{\bB} \widehat{\bz}_j - \widehat{\bP} \bB^{\star} \bz^{\star}_j  \rangle \leq  \frac{L}{2} \| (\widehat{\bP} - \bI) \bB^{\star} \bz^{\star}_j   \|_2^2 \notag\\
&=  \frac{L}{2} \| (\widehat{\bP} - \bP^{\star}) \bB^{\star} \bz^{\star}_j   \|_2^2
=  \frac{L}{2} \Big( \| \widehat{\bP} - \bP^{\star} \|_2  \| \bB^{\star} \|_2 \| \bz^{\star}_j   \|_2 \Big)^2  \leq \frac{L}{2} \bigg(
\frac{C_3 K \tau \alpha_1}{\alpha_2 \rho \sqrt{m}}
\bigg)^2.
\label{eqn-lowrank-sharp-1}
\end{align}
The last inequality follows from Claim \ref{claim-lowrank-subspace-sharp}, $\bB^{\star} \in \cO_{d,K }$ and $\| \bz^{\star}_j \|_2 \leq \alpha_1$.

Since $\widehat{\bP} \widehat{\bB} \widehat{\bz}_j  = \widehat{\bB} \widehat{\bz}_j $,
\[
\widehat{\bB} \widehat{\bz}_j - \widehat{\bP} \bB^{\star} \bz^{\star}_j
= \widehat{\bP} ( \widehat{\bB} \widehat{\bz}_j - \bB^{\star} \bz^{\star}_j )
=  \bP^{\star} ( \widehat{\bB} \widehat{\bz}_j - \bB^{\star} \bz^{\star}_j ) + (\bP^{\star} - \widehat{\bP} ) ( \widehat{\bB} \widehat{\bz}_j - \bB^{\star} \bz^{\star}_j ).
\]
Based on
\begin{align*}
| \langle \nabla f_j (  \bB^{\star} \bz^{\star}_j )  , \bP^{\star} ( \widehat{\bB} \widehat{\bz}_j - \bB^{\star} \bz^{\star}_j ) \rangle |
& \leq \| \bP^{\star}  \nabla f_j (  \bB^{\star} \bz^{\star}_j )  \|_2 \|  \widehat{\bB} \widehat{\bz}_j - \bB^{\star} \bz^{\star}_j \|_2 \\
& \leq \eta_0 \|  \widehat{\bB} \widehat{\bz}_j - \bB^{\star} \bz^{\star}_j \|_2
\end{align*}
and
\begin{align*}
& | \langle \nabla f_j (  \bB^{\star} \bz^{\star}_j )  , (\bP^{\star} - \widehat{\bP} ) ( \widehat{\bB} \widehat{\bz}_j - \bB^{\star} \bz^{\star}_j ) \rangle | \\
& \leq \| \nabla f_j (  \bB^{\star} \bz^{\star}_j ) \|_2 \| \bP^{\star} - \widehat{\bP} \|_2 \| \widehat{\bB} \widehat{\bz}_j - \bB^{\star} \bz^{\star}_j \|_2 
\\ &
 \leq \eta \frac{C_3 K \tau \alpha_1}{\alpha_2 \rho \sqrt{m}}  \| \widehat{\bB} \widehat{\bz}_j - \bB^{\star} \bz^{\star}_j \|_2,
\end{align*}
we get
\[
 | \langle \nabla f_j (  \bB^{\star} \bz^{\star}_j )  ,   \widehat{\bB} \widehat{\bz}_j - \bB^{\star} \bz^{\star}_j  \rangle | \leq \bigg(
\eta_0 + \eta \frac{C_3 K \tau \alpha_1}{\alpha_2 \rho \sqrt{m}} 
\bigg)  \| \widehat{\bB} \widehat{\bz}_j - \bB^{\star} \bz^{\star}_j \|_2 
\]
and
\begin{align}
& \frac{\rho}{2} \| \widehat{\bB} \widehat{\bz}_j  - \bB^{\star} \bz^{\star}_j   \|_2^2 +
\langle \nabla f_j (  \bB^{\star} \bz^{\star}_j )  , \widehat{\bB} \widehat{\bz}_j - \widehat{\bP} \bB^{\star} \bz^{\star}_j  \rangle  \notag\\
& \geq \frac{\rho}{2} \| \widehat{\bB} \widehat{\bz}_j  - \bB^{\star} \bz^{\star}_j   \|_2 \bigg[
\| \widehat{\bB} \widehat{\bz}_j  - \bB^{\star} \bz^{\star}_j   \|_2 - \frac{2}{\rho} \bigg(
\eta_0 + \eta \frac{C_3 K \tau \alpha_1}{\alpha_2 \rho \sqrt{m}} 
\bigg)
\bigg].
\label{eqn-lowrank-sharp-2}
\end{align}

From \Cref{eqn-lowrank-sharp-1,eqn-lowrank-sharp-2},
\begin{align*}
 \| \widehat{\bB} \widehat{\bz}_j  - \bB^{\star} \bz^{\star}_j   \|_2 \bigg[
\| \widehat{\bB} \widehat{\bz}_j  - \bB^{\star} \bz^{\star}_j   \|_2 - \frac{2}{\rho} \bigg(
\eta_0 + \eta \frac{C_3 K \tau \alpha_1}{\alpha_2 \rho \sqrt{m}} 
\bigg)
\bigg]
\leq \frac{L}{\rho} \bigg(
\frac{C_3 K \tau \alpha_1}{\alpha_2 \rho \sqrt{m}}
\bigg)^2.
\end{align*}
Therefore,
\begin{align*}
\| \widehat{\bB} \widehat{\bz}_j  - \bB^{\star} \bz^{\star}_j   \|_2 & \lesssim \max\bigg\{
\frac{1}{\rho} \bigg( \eta_0 + \eta \frac{K \tau \alpha_1}{\alpha_2 \rho \sqrt{m}} \bigg)
,~\sqrt{\frac{L}{\rho}}   \frac{ K \tau \alpha_1}{\alpha_2 \rho \sqrt{m}}
\bigg\} \\
& \lesssim \max\bigg\{
\frac{\eta_0}{\rho} ,~ 
\frac{    K}{  \alpha_2  \sqrt{m} } \cdot \frac{\tau}{\rho} \bigg(
\frac{\eta}{\rho} + \alpha_1  \sqrt{\frac{L}{\rho}} 
\bigg)
\bigg\}.
\end{align*}
\end{proof}

\section{Proofs of Section 2}\label{sec-warmup-proof}

\subsection{Proof of \Cref{thm-warmup-analytical}}\label{sec-thm-warmup-analytical-proof}
Fix $\theta \in \RR$. We have
\[
f_j (\theta) = \frac{1}{2n}  \sum_{i=1}^n (  x_{ji} - \theta )^2 = \frac{1}{2}    ( \theta - \bar{x}_j )^2 + \frac{1}{2n}  \sum_{i=1}^n (  x_{ji} - \bar{x}_j )^2.
\]
Define $g(\xi) =  f_j(\xi) + \lambda |\theta - \xi| $. Then,
\[
\partial g(\xi) = f_j'(\xi) + \lambda \partial |\xi - \theta| = \xi - \bar{x}_j + \lambda \partial |\xi - \theta| = 
\begin{cases}
\xi - \bar{x}_j + \lambda &,\mbox{ if } \xi > \theta \\
\xi - \bar{x}_j + [-\lambda, \lambda] &,\mbox{ if } \xi = \theta \\
\xi - \bar{x}_j - \lambda &,\mbox{ if } \xi < \theta 
\end{cases}.
\]
Setting $0 \in \partial g(\xi)$, we get
\begin{align}
\xi 
= \begin{cases}
\bar{x}_j - \lambda &,\mbox{ if } \theta < \bar{x}_j - \lambda \\
\theta &,\mbox{ if } | \theta - \bar{x}_j | \leq \lambda  \\
\bar{x}_j + \lambda &,\mbox{ if }  \theta >  \bar{x}_j + \lambda 
\end{cases}.
\label{eqn-thm-warmup-analytical}
\end{align}
Plugging this into $g(\xi) =  f_j(\xi) + \lambda |\theta - \xi| $, we get $\widetilde{f}_j (\theta) = \rho_{\lambda}(\theta - \bar{x}_j) + \frac{1}{2n}  \sum_{i=1}^n (  x_{ji} - \bar{x}_j )^2$. Then,
\begin{align*}
\widehat{\theta} \in \argmin_{\theta \in \RR}  \sum_{j=1}^{m} \rho_{\lambda}(\theta - \bar{x}_j) .
\end{align*}

From \eqref{eqn-thm-warmup-analytical} we obtain that
\begin{align*}
\argmin_{\xi \in \RR} \{ f_j(\xi) + \lambda |\theta - \xi|  \}
& = \theta + \bigg( 1 - \frac{\lambda}{| \bar{x}_j - \theta |} \bigg)_+ ( \bar{x}_j - \theta) \\
& =  \bar{x}_j - \min \{ \lambda, |  \bar{x}_j - \theta | \} \sgn( \bar{x}_j -  \theta ).
\end{align*}
Then, the desired expressions of $\widehat{\theta}_j$ become obvious.

\subsection{Proof of \Cref{thm-warmup-armul}}\label{sec-thm-warmup-armul-proof}

\subsubsection{Case 1: $\varepsilon \leq 1/12$}

 Let $S = S(\btheta^{\star})$. 
According to \Cref{defn-armul-relatedness}, the loss functions $\{ f_j \}_{j=1}^m$ are $( \frac{\varepsilon}{1 - \varepsilon} ,\delta)$-related with regularity parameters $(\{ \theta^{\star}_j  \}_{j=1}^m , +\infty, 1, 1)$. \Cref{thm-armul-deterministic} asserts that when $\varepsilon< 1/2$ and $\lambda > \frac{5}{1-2\varepsilon} \max_{j \in S} | f'(\theta^{\star}_j) |$,
\begin{align}
\max_{j \in S} | \widehat{\theta}_j  - \theta^{\star}_j |
\leq  
\bigg| \frac{1}{|S|}  \sum_{j \in S}  f_j' (\theta^{\star}_j) \bigg|  +     \frac{6  }{1 - 2 \varepsilon } \min\bigg\{ 3  \delta ,~    \frac{ 2 \lambda }{ 5   }   \bigg\}  +   2 \varepsilon \lambda  .
\label{eqn-warmup-armul-1}
\end{align}
Since $f_j' (\theta^{\star}_j) = \theta^{\star}_j - \bar{x}_j = \frac{1}{n} \sum_{ i = 1 }^n (\theta^{\star}_j - x_{ji}) $, $\{ f_j' (\theta^{\star}_j) \}_{j=1}^m$ are i.i.d.~$N(0,1/n)$.

To control the derivatives, we show a standard tail bound on the Gaussian distribution. Let $Z \sim N(0, 1)$. By direct calculation,
\begin{align*}
\PP (Z \geq t ) &= \int_{t}^{\infty} \frac{1}{\sqrt{2 \pi} } e^{-s^2/2} \rd s \leq \int_{t}^{\infty} \frac{s}{t} \cdot \frac{1}{\sqrt{2 \pi} } e^{-s^2/2} \rd s = \frac{1}{\sqrt{2 \pi} t } \int_{t}^{\infty}  e^{-s^2/2} \rd (s^2/2) \\
&= \frac{1}{\sqrt{2 \pi} t } \int_{t^2/2}^{\infty} e^{-u} \rd u = \frac{e^{-t^2/2}}{\sqrt{2 \pi} t } ,\qquad \forall t > 0.
\end{align*}
Hence $\PP (|Z| \geq s ) \leq e^{-s^2/2} / 2$, $\forall s \geq 2$.
When $t \geq 2$, we obtain from $f_j' (\theta^{\star}_j) \sim N(0,\frac{1}{n})$, $ \frac{1}{|S|}  \sum_{j \in S}  f_j' (\theta^{\star}_j)  \sim N(0, \frac{1}{n|S|})$ and $|S| \geq ( 1 - \varepsilon) m$ that
\begin{align}
& \PP  \bigg( | f_j' (\theta^{\star}_j) | \geq \sqrt{ \frac{ 2 \log m + 2 t }{n} } \bigg) \leq  e^{-(2 \log m  + 2 t)/2 } /2 = m^{-1}e^{-t}/2 ,\quad\forall j \in S; \notag\\
& \PP \bigg( \bigg|
\frac{1}{|S|}  \sum_{j \in S}  f_j' (\theta^{\star}_j) \bigg| \geq \sqrt{ \frac{2t}{ ( 1 - \varepsilon)mn } } \bigg) 
\leq \PP \bigg( \bigg|
\frac{1}{|S|}  \sum_{j \in S}  f_j' (\theta^{\star}_j) \bigg| \geq \sqrt{ \frac{2t}{n|S|} } \bigg)
\leq e^{-t} /2 .
\end{align}

Now, fix any $t \geq 2$ and define an event
\[
\cA_t = \bigg\{
\max_{j \in [m]} | f_j'( \theta^{\star}_j ) | <  \sqrt{ \frac{2 \log m + 2 t}{n} }
~~\text{and}~~
\bigg| \frac{1}{|S|} \sum_{j\in S} f_j'( \theta^{\star}_j )  \bigg| < \sqrt{ \frac{2 t}{ ( 1 - \varepsilon) mn} }
\bigg\}.
\]
By union bounds, $\PP (\cA_t) \geq 1 - e^{-t}$. Take
\[
\lambda = 6 \sqrt{\frac{2 \log m + 2 t}{n}}
\]
and let $\cA_t$ happen. We use \eqref{eqn-warmup-armul-1} and the assumption $\varepsilon \leq 1/12$ to get
\begin{align*}
\max_{j \in S} | \widehat{\theta}_j  - \theta^{\star}_j |
& <  
\sqrt{ \frac{2t}{ ( 1 - \varepsilon) mn} } +   \frac{6  }{1 - 2 \varepsilon } \min\bigg\{ 3  \delta ,~    \frac{ 2 \lambda }{ 5   }   \bigg\}  +   2 \varepsilon \lambda \\
&\lesssim
\sqrt{ \frac{t}{mn} } +   \min\bigg\{   \delta ,~ \sqrt{\frac{\log m + t}{n} } \bigg\}  +   \varepsilon  \sqrt{\frac{\log m + t}{n} } ,
\end{align*}
where $\lesssim$ only hides a universal constant. Meanwhile, \Cref{thm-armul-personalization} implies that
\begin{align}
\max_{j \in [m]} | \widehat{\theta}_j  - \theta^{\star}_j |
&\leq \max_{j \in [m]} |f'(\theta^{\star}_j )| + \lambda 
 \lesssim  \sqrt{\frac{ \log m + t}{n}}
.
\label{eqn-warmup-armul-2}
\end{align}
This implies the desired upper bound on $\max_{j \in S^c} | \widehat{\theta}_j  - \theta^{\star}_j |$. We easily get the mean squared error bound:
\begin{align*}
\frac{1}{m} \sum_{j=1}^{m}  | \widehat{\theta}_j  - \theta^{\star}_j |^2
& \leq 
\frac{1}{m} \bigg(
|S| \max_{j \in S} | \widehat{\theta}_j  - \theta^{\star}_j |^2
+ |S^c| \max_{j \in S^c} | \widehat{\theta}_j  - \theta^{\star}_j |^2
\bigg) \\
& \lesssim  
 \frac{t}{mn} +   \min\bigg\{   \delta^2 ,~  \frac{\log m + t}{n}  \bigg\}  +   \varepsilon  \frac{\log m + t}{n} . 
\end{align*}

\subsubsection{Case 2: $\varepsilon > 1/12$}

When $\varepsilon > 1/12$ and the event \eqref{eqn-warmup-armul-2} happens (which has probability at least $1 - e^{-t}$), the desired error bounds trivially hold.

\subsection{Proof of \Cref{lem-warmup-mean}}\label{sec-lem-warmup-mean-proof}
It is easily seen that
\begin{align*}
\widetilde{\theta} \in \argmin_{\theta \in \RR}  \sum_{j=1}^{m} g_j(\theta)
\qquad\text{and}\qquad
\widetilde{\theta}_j \in \argmin_{\theta \in \RR} \{ (\theta - x_j)^2 + \lambda ( \theta - \widetilde\theta )^2 \}, ~~\forall j \in [m] ,
\end{align*}
where
\[
g_j(\theta) = \min_{\xi \in \RR} \{ (\xi - x_j)^2 + \lambda ( \theta - \xi )^2 \} = \frac{\lambda}{1+\lambda} (\theta - x_j)^2.
\]
Hence $\widetilde{\theta} = \bar{x}$ and
\[
\widetilde{\theta}_j = \frac{1}{1+\lambda} x_j + \frac{\lambda}{1+\lambda} \widetilde\theta = \bar{x} + \frac{1}{1+\lambda} (x_j - \bar{x}).
\]
The rest of the proof follows from simple algebra.

\section{Analysis of general sample sizes}\label{sec-general-nj}

In this section, we analyze the ARMUL \eqref{eqn-armul} with possibly different sample sizes $\{ n_j \}_{j=1}^m$. We provide personalization guarantees for general ARMUL, and then study the adaptivity and robustness of vanilla ARMUL.

\subsection{Personalization}\label{sec-general-personalization-proof}

Consider the estimators $\{ \widehat\btheta_j \}_{j=1}^m$ returned by ARMUL \eqref{eqn-armul} with arbitrary $w_j > 0$ and $\lambda_j \geq 0$.

\begin{theorem}[Personalization]\label{thm-personalization}
	Let Assumptions \ref{as-stat-1} and \ref{as-stat-2} hold. Define $n = \min_{j \in [m]} n_j$. There exist constants $C$, $C_1$ and $C_2$ such that under the conditions $\max_{j \in [m] } \lambda_j < \rho M / 4$, $n  > C_1 d ( \log n  )(  \log m )$ and $0 \leq t <  C_2 n  / ( d \log n )$, the followings hold with probability at least $1 -  e^{-t}$:
	\begin{align*}
		& \| \widetilde{\btheta}_j - \btheta^{\star}_j \|_2 \leq
		C \sigma \sqrt{ \frac{ d + \log m + t}{n_j} } 
		\qquad\text{and}\qquad
		\| \widehat{\btheta}_j - \widetilde{\btheta}_j \|_2 \leq \frac{ 2 \lambda_j }{\rho} , \qquad\forall j \in [m].
	\end{align*}
\end{theorem}

\Cref{thm-personalization} immediately follows from the lemma below and \Cref{thm-armul-personalization}.

\begin{lemma}\label{lem-landscape}
	Let Assumptions \ref{as-stat-1} and \ref{as-stat-2} hold. Define $n  = \min_{j \in [m]} n_j$ and $N = \sum_{ j = 1 }^m n_j$. There exist constants $C$, $C_1$ and $C_2$ such that under the conditions $n  > C_1 d ( \log n  )(  \log m )$ and $0 \leq t <  C_2 n  / ( d \log n )$, the followings hold with probability at least $1 -  e^{-t}$:
	\begin{align*}
		& \| \nabla f_j ( \btheta^{\star}_j   ) \|_2 
		< C \sigma \sqrt{ \frac{ d + \log m + t}{n_j} } 
		\leq \frac{\rho M}{4} , \qquad\forall j \in [m] ; \\
		& \bigg\| \frac{1}{N} \sum_{ j = 1 }^m n_j  \nabla f_j ( \btheta^{\star}_j   ) \bigg\|_2 < C  \sigma \sqrt{ \frac{ d  + t}{N} } ;
		\\
		& \frac{ \rho }{2} \bI \preceq \nabla^2 f_j (\btheta) \preceq \frac{3L}{2} \bI,\qquad\forall \btheta \in B(\btheta^{\star}_j, M),~~ j \in [m].
	\end{align*}
	On the same event, for any $j \in [m]$, $f_j$ is $(\btheta_j^{\star}, M, \rho/2, 3L/2)$-regular (\Cref{defn-armul-regularity}).
\end{lemma}

\begin{proof}[\bf Proof of \Cref{lem-landscape}]
	Choose any $j \in [m]$. Note that $ \EE [ \nabla f_j ( \btheta^{\star}_j  ) ] = \bm{0} $ and $ \| \nabla f_j ( \btheta^{\star}_j   ) \|_{\psi_2} \lesssim \sigma / \sqrt{n_j} $. By Theorem 2.1 in \cite{HKZ12}, there exists a universal constant $c$ such that
	\[
	\PP \bigg(
	\| \nabla f_j ( \btheta^{\star}_j   ) \|_2^2 \geq \frac{c^2 \sigma^2}{n_j} (d + 2 \sqrt{dt} + 2 t)
	\bigg) \leq e^{-t} , \qquad\forall t \geq 0.
	\]
	From $2 \sqrt{dt} \leq d + t$ we see that
	\[
	\PP \bigg(
	\| \nabla f_j ( \btheta^{\star}_j   ) \|_2 \geq  c_1 \sigma \sqrt{ \frac{ d + \log m + t}{n_j} } 
	\bigg) \leq \frac{e^{-t} }{3m} , \qquad\forall t \geq 0
	\]
	holds for a sufficiently large universal constant $c_1$. By union bounds,
	\[
	\PP \bigg( 
	\| \nabla f_j ( \btheta^{\star}_j   ) \|_2  
	< c_1 \sigma \sqrt{ \frac{ d + \log m + t}{n_j} } , ~~ \forall j \in [m] 
	\bigg) \geq 1 - e^{-t} /3 , \qquad\forall t \geq 0.
	\]
	Hence, when $ 0 \leq t \leq \max \{
	( \frac{\rho M}{4 c_1 \sigma}  )^2 n  - d - \log m ,~ 0
	\}$, we have
	\begin{align}
		\PP \bigg(
		\max_{j \in [m]}
		\| \nabla f_j ( \btheta^{\star}_j   ) \|_2 
		< c_1 \sigma \sqrt{ \frac{ d + \log m + t}{n_j} } 
		\leq \frac{\rho M}{4}
		\bigg) \geq 1 - e^{-t} /3 .
		\label{eqn-lem-landscape-1}
	\end{align}
	Since $\rho M / \sigma \gtrsim 1$, we can find constants $c_2$ and $c_3$ such that when $n  > c_2 (d + \log m)$ and $0 \leq t < c_3 n $, the tail bound \eqref{eqn-lem-landscape-1} holds.
	
	On the other hand, $\sum_{ j = 1 }^m n_j \nabla f_j ( \btheta^{\star}_j   ) = \sum_{ j = 1 }^m \sum_{ i = 1 }^{n_j} \nabla \ell_j (\btheta^{\star}_j ; \bxi_{ji} )$ has zero mean and
	\[
	\| \sum_{ j = 1 }^m n_j \nabla f_j ( \btheta^{\star}_j   )  \|_{\psi_2} \lesssim \sqrt{N} \sigma .
	\]
	Similar to the analysis of $\| \nabla f_j ( \btheta^{\star}_j   ) \|_2$ above, we can find a universal constant $c_1'$ such that 
	\begin{align}
		\PP \bigg(
		\bigg\| \frac{1}{N} \sum_{ j = 1 }^m n_j  \nabla f_j ( \btheta^{\star}_j   ) \bigg\|_2 \geq  c_1' \sigma \sqrt{ \frac{ d  + t}{N} } 
		\bigg) \leq e^{-t} /3 , \qquad\forall t \geq 0.
		\label{eqn-lem-landscape-3}
	\end{align}
	
	According to the proof of Theorem 1 Part (b) in \cite{MBM18}, there exists a constant $C_0 > 0$ such that for any $t > 0$, the followings hold with $C  = C_0 \max \{ \log(M \tau ) + t,  p  + 1 \}$: when $n > C d \log d$,
	\begin{align*}
		\PP \bigg(
		\sup_{ \btheta \in B( \btheta^{\star}_j, M ) }
		\| \nabla^2 f_j (\btheta) - \nabla^2 F_j (\btheta) \|_2 \geq \tau^2 \sqrt{ \frac{C d \log n_j}{n_j} }
		\bigg) \leq e^{-t} .
	\end{align*}
	Under the condition
	\[
	C_0 \max \{ \log (3mM\tau) + t ,  p  + 1 \} < \frac{n_j}{d \log (n_j+d) } \min \bigg\{ 
	\bigg( \frac{\rho}{2 \tau^2} \bigg)^2 , 1
	\bigg\},
	\]
	we have
	\begin{align*}
		\PP \bigg( 
		\sup_{ \btheta \in B( \btheta^{\star}_j, M ) }
		\| \nabla^2 f_j (\btheta) - \nabla^2 F_j (\btheta) \|_2< \frac{\rho}{2}
		\bigg) \geq 1 - \frac{ e^{-t} }{3m} .
	\end{align*}
	
	We claim that when $d \geq 1$, the function $g(x) = \frac{x}{\log(x+d)}$ is increasing in $[1,+\infty)$. To prove it, observe that
	\begin{align*}
		g'(x) & = \frac{1}{\log(x+d)} + x\cdot \frac{-1}{\log^2(x+d)} \cdot \frac{1}{x+d} = 
		\frac{1}{\log(x+d)} \bigg(
		1 - \frac{x}{ ( x+d)\log(x+d)} 
		\bigg) . 
	\end{align*}
	It remains to show that $x < (x+d) \log(x+d)$ for all $x \geq 1$, i.e. $t - d < t \log t$ for all $t \geq d+1$. Let $h(t) = t \log t$. Since $h'(t) = \log t + 1$ and $h''(t) = 1/t > 0$, we have
	\[
	t \log t = h(t) \geq h(1) + h'(1) (t-1) = t - 1 > t - d, \qquad t \geq 1.
	\]
	Therefore, $g'(x) > 0$ when $x \geq 1$.

	Consequently, $\min_{j \in [m]} \frac{n_j}{ \log (n_j+d) } = \frac{n }{ \log (n +d) }$. Under the condition
	\[
	C_0 \max \{ \log (3mM\tau) + t ,  p  + 1 \} < \frac{n }{d \log (n  + d) } \min \bigg\{ 
	\bigg( \frac{\rho}{2 \tau^2} \bigg)^2 , 1
	\bigg\},
	\]
	we apply union bounds to get
	\begin{align}
		\PP \bigg(
		\max_{j \in [m] }
		\sup_{ \btheta \in B( \btheta^{\star}_j, M ) }
		\| \nabla^2 f_j (\btheta) - \nabla^2 F_j (\btheta) \|_2< \frac{\rho}{2}
		\bigg) \geq 1 - e^{-t} / 3.
		\label{eqn-lem-landscape-2}
	\end{align}
	Since $\rho / \tau^2  \gtrsim 1$, we can find constants $c_4$ and $c_5$ such that when $n  > c_4 d ( \log n  )(  \log m )$ and $0 \leq t < \frac{c_5 n }{d \log n }$, the concentration inequality \eqref{eqn-lem-landscape-2} holds. The proof is finished by re-defining the constants and combining \eqref{eqn-lem-landscape-1}, \eqref{eqn-lem-landscape-3} and \eqref{eqn-lem-landscape-2}.
\end{proof}

\subsection{Vanilla ARMUL}\label{sec-general-vanilla-proof}

In this subsection, we analyze the vanilla ARMUL estimators $\{ \widehat\btheta_j \}_{j=1}^m$ returned by \eqref{eqn-armul} with $w_j = n_j $ and $\lambda_j \propto 1 / \sqrt{ n_j }$. In other words, we choose some $\lambda_0 \geq 0$ and let
\begin{align}
	( \widehat{\bTheta}, \widehat{\bbeta} ) \in
	\argmin_{
		\bTheta \in \RR^{d\times m},~ \bbeta \in \RR^d} \bigg\{ \sum_{j=1}^{m} n_j \bigg(  f_j (\btheta_j) + \frac{\lambda_0}{\sqrt{n_j}} \| \btheta_j - \bbeta \|_2 \bigg) \bigg\}.
	\label{eqn-armul-vanilla}
\end{align}
We introduce an assumption on task relatedness. It generalizes Assumption \ref{as-armul-vanilla-00} to allow for different sample sizes of the $m$ tasks.

\begin{assumption}[Task relatedness]\label{as-armul-vanilla}
	There exists $\varepsilon ,\delta_0 \geq 0$ and subset $S \subseteq [m]$ such that 
	\[
	\min_{\btheta \in \RR^d}  \max_{j \in S} \{ \sqrt{n_j}\| \btheta^{\star}_j - \btheta  \|_2 \} \leq \delta_0
	\qquad\text{and}\qquad
	\sum_{ j \in S^c } \sqrt{n_j} \leq \varepsilon \sum_{j \in S} n_j / (\max_{j \in S } \sqrt{n_j} ) .
	\]
\end{assumption}

When $n_1 = \cdots = n_m = n$, Assumption \ref{as-armul-vanilla} reduces to
$\min_{\btheta \in \RR^d}  \max_{j \in S} \| \btheta^{\star}_j - \btheta  \|_2   \leq \delta_0 / \sqrt{n}$ and $ |S^c| \leq \varepsilon |S|$. It is essentially the same as Assumption \ref{as-armul-vanilla-00}. In Assumption \ref{as-armul-vanilla} we compare the tasks in $S^c$ with those in $S$ (rather than $[m]$) for technical convenience when $\{ n_j \}_{j=1}^m$ are different.
The theorem below presents upper bounds on estimation errors of vanilla ARMUL \eqref{eqn-armul-vanilla}.

\begin{theorem}[Vanilla ARMUL]\label{thm-vanilla}
	Let Assumptions \ref{as-stat-1}, \ref{as-stat-2} and \ref{as-armul-vanilla} hold. 
	Define $n  = \min_{j \in [m]} n_j$, $N = \sum_{ j = 1 }^m n_j$ and $\kappa_w = \max_{j \in S } \sqrt{n_j} \cdot \sum_{ j \in S } \sqrt{n_j} /\sum_{j \in S} n_j $. 
	There exist positive constants $\{ C_i \}_{i=0}^5$ such that under the conditions $n  > C_1 d ( \log n  )(  \log m )$, $0 \leq t <  C_2 n  / ( d \log n )$, $C_3 \kappa_w \sigma  \sqrt{ d + \log m + t} < \lambda_0 < C_4 \sigma \sqrt{n}  $ and $0 \leq \varepsilon < C_5$, the following bounds hold with probability at least $1 -  e^{-t}$:
	\begin{align*}
		& \| \widehat{\btheta}_j  - \btheta^{\star}_j \|_2  
		\leq
		C_0 \bigg(  \sigma \sqrt{ \frac{ d  + t}{N} }
		+     \frac{ \min \{ \kappa_w  \delta_0 ,   \lambda_0   \}  +  \varepsilon \lambda_0  }{  \sqrt{n_j} } 
		\bigg)
		, \qquad\forall j \in S ; \\
		& \| \widehat{\btheta}_j  - \btheta^{\star}_j \|_2  
		\leq     \frac{ C_0 \lambda_0  }{  \sqrt{n_j} } 
		, \qquad\forall j \in S^c , \\
&	\frac{1}{N} \sum_{j=1}^{m} n_j [ F_j ( \widehat{\btheta}_j ) - F_j ( \btheta^{\star}_j ) ]
\leq  \frac{L}{N} \sum_{j=1}^{m} n_j \| \widehat{\btheta}_j - \btheta^{\star}_j \|_2^2 \\
&\qquad  \qquad\qquad\qquad\qquad~~~
\leq C_0 L \bigg( 
	\sigma^2 \frac{ d  + t}{N} + 
	\frac{ |S| }{ N }
	\min \{ \kappa_w^2  \delta_0^2 ,   \lambda_0^2   \} + 
	\lambda_0^2 
	\frac{ \varepsilon^2 |S| + |S|^c} {N} 
	\bigg).
	\end{align*}
	Moreover, there exists a constant $C_6$ such that under the conditions $\varepsilon = 0$ and $C_6 \kappa_w  \delta_0 < \sigma \sqrt{ d + \log m }$, we have $\widehat{\btheta}_1 = \cdots = \widehat{\btheta}_m = \argmin_{\btheta \in \RR^d} \{ \sum_{j = 1}^m n_j f_j (\btheta) \}$ with probability at least $1 - e^{-t}$.
\end{theorem}

\Cref{thm-vanilla} simultaneously controls the estimation errors for all individual tasks and suggests choosing $\lambda_0 \asymp \kappa_w \sigma \sqrt{d + \log m} $ and thus
\[
\lambda_j \asymp \kappa_w \sigma \sqrt{ \frac{d + \log m}{n_j}  }.
\]
The quantity $\kappa_w$ measures the heterogeneity of sample sizes $\{ n_j \}_{j=1}^m$. By Cauchy-Schwarz inequality, we have $( \sum_{ j \in S} \sqrt{n_j} )^2 \leq |S| \sum_{ j \in S } n_j$ and
\[
\kappa_{w} = \frac{ \max_{j \in S } \sqrt{n_j} \cdot \sum_{ j \in S } \sqrt{n_j} }{ \sum_{j \in S} n_j }
\leq \sqrt{ \frac{ \max_{j \in S }  n_j }{ |S|^{-1}  \sum_{j \in S} n_j  } }.
\]
In words, $\kappa_w^2$ is bounded by the ratio between the maximum sample size and the average sample size of the tasks in $S$. When $n_1= \cdots = n_m$, $\kappa_{w}$ attains its minimum value $1$.

\begin{proof}[\bf Proof of \Cref{thm-vanilla}]
	For sufficiently large $C_1$ and sufficiently small $C_2$, Assumption \ref{as-stat-1}, Assumption \ref{as-stat-2} and \Cref{lem-landscape} imply the existence of a constant $c_1$ such that with probability $1 - e^{-t}$,
	\begin{itemize}
		\item $\max_{j \in [m]} \{ \sqrt{n_j}  \| \nabla f_j (\btheta_j^{\star}) \|_2 \} \leq c_1 \sigma \sqrt{ d + \log m + t}$;
		\item for any $j \in [m]$, $f_j$ is $(\btheta_j^{\star}, M, \rho/2, 3L/2)$-regular in the sense of \Cref{defn-armul-regularity}.
	\end{itemize}
	Let the above event happen. Assumption \ref{as-armul-vanilla} implies that $\{ f_j, n_j \}_{j=1}^m$ are $(\varepsilon, \delta_0 )$-related with regularity parameters $(\{ \btheta_j^{\star} \}_{j=1}^m, M, \rho / 2, 3 L / 2 )$ in the sense of \Cref{defn-armul-relatedness}.
	By \Cref{thm-armul-deterministic} and the assumptions $ \rho, L, M \asymp 1$, there exist constants $c$, $c_2$, $c_3$ and $c_4$ such that when $0 \leq \varepsilon \leq c$ and $c_2 \kappa_w \sigma \sqrt{ d + \log m + t} < \lambda_0 < c_3 \sigma \sqrt{n}  $, we have
	\begin{align*}
		& \| \widehat{\btheta}_j  - \btheta^{\star}_j \|_2  
		\leq
		c_4 \bigg( \sigma \sqrt{ \frac{ d  + t}{N} }
		+     \frac{ \min \{ \kappa_w  \delta ,   \lambda_0  \}  +  \varepsilon \lambda_0  }{  \sqrt{n_j} } 
		\bigg)
		, \qquad\forall j \in S.
	\end{align*}
	\Cref{thm-armul-personalization} applied to the tasks in $S^c$ yields
	\[
	\| \widehat{\btheta}_j  - \btheta^{\star}_j \|_2  
	\leq     \frac{ c_5 \lambda_0 }{  \sqrt{n_j} } 
	, \qquad\forall j \in S^c.
	\]
Since $\nabla^2 F_j \preceq L \bI$ in $B( \btheta^{\star}_j , M )$,
\begin{align*}
\frac{1}{N} \sum_{j=1}^{m} n_j [ F_j ( \widehat{\btheta}_j ) - F_j ( \btheta^{\star}_j ) ]
& \leq  \frac{L}{N} \sum_{j=1}^{m} n_j \| \widehat{\btheta}_j - \btheta^{\star}_j \|_2^2 \\
& \lesssim 
 \frac{L}{N} \sum_{j \in S }
\bigg(  \sigma^2  n_j \frac{ d  + t}{N} 
	+  \min \{ \kappa_w^2  \delta_0^2 ,   \lambda_0^2   \}  +  \varepsilon^2 \lambda_0^2  
	\bigg)
+  \frac{L}{N} \sum_{j \in S^c }   \lambda_0^2 \\
& \lesssim L \bigg( 
\sigma^2 \frac{ d  + t}{N} + 
\frac{ |S| }{ N }
 \min \{ \kappa_w^2  \delta_0^2 ,   \lambda_0^2   \} + 
 \lambda_0^2 
 \frac{ \varepsilon^2 |S| + |S|^c} {N} 
\bigg).
\end{align*}
The proof is finished by re-defining the constants.
\end{proof}

\section{Proofs of Section 4}

\subsection{Proof of \Cref{thm-minimax-00}}\label{sec-thm-minimax-proof}

We invoke an elementary lemma which follows from a standard minimax argument \cite{Tsy09}. The proof is omitted.

\begin{lemma}\label{lem-gaussian}
Let $\btheta^{\star}  \in \RR^d$ and $ \bz \sim N ( \btheta^{\star} , \bI_d ) $. There exist universal constants $c_1, C_1 > 0$ such that
\begin{align*}
	&  \inf_{\widehat{\btheta} } \sup_{   \btheta^{\star} \in \RR^d } 
	\PP_{ \btheta^{\star} }
	\bigg(
	\| \widehat{\btheta} ( \bz ) - \btheta^{\star}  \|_2^2 \geq C_1  d
	\bigg) 
	\geq c_1.
\end{align*}
Here the infimum is taken over all estimators $\widehat{\btheta} = \widehat{\btheta} ( \bz ) $.
\end{lemma}

Note that $\Omega ( 0, 0 ) = \{ \btheta \bm{1}_m^{\top} :~ \btheta \in \RR^d \}$. When $\bTheta^{\star} \in \Omega ( 0, 0 ) $, we have $\btheta^{\star}_1 = \cdots = \btheta^{\star}_m = \btheta^{\star}$, and $\{ \bx_{ji} \}_{ (i, j) \in [n] \times [m] } $ are i.i.d.~$N( \btheta^{\star}, \bI_d )$. A sufficient statistic is the pooled mean $\bar{x} = \frac{1}{mn} \sum_{ (i, j) \in [n] \times [m]  } \bx_{ji}$, which has distribution $N( \btheta^{\star}, \frac{1}{mn} \bI_d )$. \Cref{lem-gaussian} then implies that
\begin{align}
	&  \inf_{\widehat{\bTheta} } \sup_{   \bTheta^{\star} \in \Omega ( 0,0 ) } 
	\PP_{\bTheta^{\star}}
	\bigg( 
\frac{1}{m}	\| \widehat{\bTheta} - \bTheta^{\star} \|_{\mathrm{F}}^2 \geq C_1   \frac{d}{ m n }  \bigg)
	\geq c_1.
\label{eqn-minimax-1}
\end{align}

On the other hand, define $ \cA = \{ \bTheta^{\star} \in \RR^{d \times m} :~ \btheta^{\star}_j = \bm{0}_d \text{ for all } j > \varepsilon m \}$. We have $\cA \subseteq \Omega ( \varepsilon, 0 )$. Each $\bTheta$ in $\cA$ has $\lceil \varepsilon m \rceil d$ free parameters. We can use \Cref{lem-gaussian} to get
\begin{align*}
	&  \inf_{\widehat{\bTheta} } \sup_{   \bTheta^{\star} \in \cA } 
	\PP_{\bTheta^{\star}}
	\bigg( 
	\| \widehat{\bTheta} - \bTheta^{\star} \|_{\mathrm{F}}^2 \geq C_1   \frac{ \lceil \varepsilon m \rceil d}{ n }  \bigg)
	\geq c_1.
\end{align*}
Hence, for all $\varepsilon \geq 1/m$, we have
\begin{align}
	&  \inf_{\widehat{\bTheta} } \sup_{   \bTheta^{\star} \in \Omega ( \varepsilon, 0 ) } 
	\PP_{\bTheta^{\star}}
	\bigg( 
	\frac{1}{m} \| \widehat{\bTheta} - \bTheta^{\star} \|_{\mathrm{F}}^2 \geq \frac{C_1}{2} \cdot   \frac{   \varepsilon   d}{ n }  \bigg)
	\geq c_1.
\label{eqn-minimax-2}
\end{align}

Define $\cW = \{ 0, 1 \}^{d \times m}$ and let $\cS = \{  \min \{ \delta / \sqrt{d} , 1 / \sqrt{n} \} \bw :~ \bw \in \cW  \}$. It is easily seen that $\cS \subseteq \Omega ( 0, \delta )$. Below we prove that
\begin{align}
	& \inf_{\widehat{\bTheta} } \sup_{  \bTheta^{\star} \in \cS } 
	\PP_{\bTheta^{\star}}
	\bigg(
\frac{1}{m} \| \widehat{\bTheta} - \bTheta^{\star} \|_{\mathrm{F}}^2 \geq  \frac{\Phi(-1/2)}{4} \min \bigg\{
	\delta^2 , \frac{d}{n}
	\bigg\}
	\bigg) 
	\geq \frac{1}{4 / \Phi(-1/2) - 1} , \label{eqn-thm-warmup-minimax-1} 
\end{align}
where $\Phi$ is the cumulative distribution function of $N(0,1)$. If that is true, then we immediately finish the proof by combining \eqref{eqn-minimax-1}, \eqref{eqn-minimax-2} and \eqref{eqn-thm-warmup-minimax-1}.

It remains to prove \eqref{eqn-thm-warmup-minimax-1}. For any $t \geq 0$,
\begin{align*}
& \inf_{\widehat{\bTheta} } \sup_{   \bTheta^{\star} \in \cS } 
\PP_{\bTheta^{\star}}
\bigg(
\frac{1}{m} \| \widehat{\bTheta} - \bTheta^{\star} \|_{\mathrm{F}}^2 \geq  t
\min \{ \delta^2 , d/n \}
\bigg) \notag \\
& = \inf_{\widehat{\bw} } \sup_{   \bw \in \cW } 
\PP_{\bw} \bigg( \| \widehat{\bw} - \bw \|_{\mathrm{F}}^2 \geq \frac{ m
 t \min \{
\delta^2 , d/n
\}
}{ \min \{ \delta^2 / d , 1 / n \} } 
\bigg) 
= \inf_{\widehat{\bw} } \sup_{   \bw \in \cW } 
\PP_{\bw}  ( \| \widehat{\bw} - \bw \|_{\mathrm{F}}^2 \geq  t m d
).
\end{align*}
Here $\inf_{\widehat{\bw} }$ denotes the infimum over all estimators $\widehat{\bw} $ taking values in $\cW$, and $\EE_{\bw}$ is the expectation given $\bTheta^{\star} = \min \{ \delta / \sqrt{d} , 1 / \sqrt{n} \}  \bw$.

\begin{lemma}\label{lem-prob-lower}
	Let $X$ be a random variable and $0 \leq X \leq M$ a.s. Then $\PP ( X \geq t ) \geq ( \EE X - t ) / (M - t)$, $\forall t \in [0, M ]$.
\end{lemma}
\begin{proof}[\bf Proof of \Cref{lem-prob-lower}]
	The inequality directly follows from
	\[
	\EE X \leq M \PP ( X \geq t ) + t \PP ( 0 \leq  X < t ) = M \PP ( X \geq t ) + t [1 - \PP ( X \geq t )] = t + (M-t) \PP ( X \geq t ).
	\]
\end{proof}

According to \Cref{lem-prob-lower} and the fact that $\| \widehat{\bw} - \bw \|_{\mathrm{F}}^2 \leq md$,
\begin{align*}
\inf_{\widehat{\bw} } \sup_{   \bw \in \cW } 
\PP_{\bw}  ( \| \widehat{\bw} - \bw \|_{\mathrm{F}}^2 \geq  t m d
) \geq
\frac{ 
	\inf_{  \widehat\bw  } \sup_{  \bw \in \cW } 
	\EE_{\bw}  \| \widehat\bw - \bw \|_{\mathrm{F}}^2
	- t m d
}{
	m d - t m d
}
, \qquad \forall 0 \leq t \leq 1.
\end{align*}
We will invoke Theorem 2.12 in \cite{Tsy09} again to derive
\begin{align}
\inf_{\widehat{\bw} } \sup_{   \bw \in \cW } 
\EE_{\bw} \| \widehat{\bw} - \bw \|_{\mathrm{F}}^2
\geq \frac{\Phi(-1/2)}{2} m d.
\label{eqn-thm-warmup-minimax-3}
\end{align}
Once that is done, we take $t = \Phi(-1/2) / 4$ and immediately get \eqref{eqn-thm-warmup-minimax-1}.

Let $r = \min \{ \delta / \sqrt{d} , 1 / \sqrt{n} \}$ and $\bw_j = (w_{1j} , \cdots, w_{d j })^{\top}$. When the true mean matrix is $r \bw$, the density function of data $\{ \bx_{ji} \}_{(i, j) \in [n] \times [m]}$ is
\begin{align*}
p ( \bX ; \bw ) = \prod_{j=1}^{m} \prod_{i=1}^{n} (2\pi)^{-d/2} e^{-\| \bx_{ji} - r \bw_j \|_2^2 / 2}.
\end{align*}
Define $\bar{\bx}_j = \frac{1}{n} \sum_{ i = 1 } \bx_{ji}$. Then,
\begin{align*}
& \frac{ p ( \bX ; \bw' ) }{ p ( \bX ; \bw ) }  = \prod_{j=1}^{m} \prod_{i=1}^{n} 
\exp \Big(
\| \bx_{ji} - r \bw_j \|_2^2 / 2 - \| \bx_{ji} - r \bw_j' \|_2^2 / 2
\Big) \\
& = \exp \bigg(
\sum_{(i,j) \in [n] \times [m]} \bigg\langle
r ( \bw_j' - \bw_j),~ \bx_{ji} - \frac{r ( \bw_j' +  \bw_j )}{2}
\bigg\rangle
\bigg) \\
& = \exp \bigg(
\sum_{ j = 1 }^m
\bigg\langle
n r ( \bw_j' - \bw_j),~ \bar\bx_{j} - \frac{ r( \bw_j' +  \bw_j ) }{2}
\bigg\rangle
\bigg).
\end{align*}
Denote by $\rho(\cdot, \cdot)$ the Hamming distance between two binary arrays of the same shape. Choose any $\bw, \bw' \in \cW$ such that $\rho(\bw', \bw) = 1$. There exists a unique $j \in [m]$ such that $\rho( \bw_{j}', \bw_{j} ) = 1$. Then
\begin{align*}
& \PP_{\bw} \bigg(
\frac{ p ( \bX ; \bw' ) }{ p ( \bX ; \bw ) } 
\geq 1 \bigg)
= 
\PP_{\bw} \bigg(
\bigg\langle
\bw_j' - \bw_j ,~ \bar\bx_{j} - \frac{ r ( \bw_j' +  \bw_j ) }{2}
\bigg\rangle \geq 0
\bigg)
\\
& = 
\PP_{\bw} \bigg(
\langle
\bw_j' - \bw_j ,~ \bar\bx_{j} - \bw_j
\rangle
\geq 
\frac{ r \| \bw_j' - \bw_j \|_2^2  }{2}
\bigg)
= 1 - \Phi ( r / 2 )
.
\end{align*}
The last inequality follows from $\| \bw_j' - \bw_j  \|_2 = 1$ and $\langle
\bw_j' - \bw_j ,~ \bar\bx_{j} - \bw_j
\rangle \sim N(0,1)$ under $\PP_{\bw} $. 

The fact $r \leq 1 / \sqrt{n} \leq 1$ forces $1 - \Phi ( r / 2 ) \geq 1 - \Phi (1/2)$. From there, Theorem 2.12 in \cite{Tsy09} leads to \eqref{eqn-thm-warmup-minimax-3}.

\subsection{Proof of \Cref{thm-clustered}}\label{sec-thm-clustered-proof}

The results directly follow from \Cref{lem-landscape} and \Cref{thm-armul-clustered-deterministic-0}. We omit the proof because it is almost identical to that of \Cref{thm-vanilla}.

\subsection{Clustered ARMUL with cardinality constraint}\label{sec-thm-clustered-robust}

We propose a constrained version of clustered ARMUL \eqref{eqn-armul-clustered-1}:
\begin{align}
( \widehat{\bTheta}, \widehat{\bB}, \widehat{\bz} )
\in \argmin_{
	\substack{ \bTheta \in \RR^{d\times m},~ \bB\in \RR^{d\times K}, \bz \in [K]^m \\ 
		\min_{k \in [K]} |\{ j \in [m] :~ z_j = k \}| \geq \alpha m / K
	}
} \bigg\{ \sum_{j=1}^{m} [ f_j (\btheta_j) + \lambda  \| \btheta_j - \bbeta_{z_j} \|_2 ] \bigg\}.
\label{eqn-armul-clustered}
\end{align}
Here $\alpha \geq 0$ is a tuning parameter.
We add the cardinality constraint to facilitate theoretical analysis in the presence of arbitrary outlier tasks. Intuitively, it helps identify meaningful task clusters with non-negligible sizes rather than small groups of outliers.

\begin{theorem}[Clustered ARMUL with cardinality constraints]\label{thm-clustered-robust}
	Let Assumptions \ref{as-stat-1}, \ref{as-stat-2} and \ref{as-armul-clustered} hold. 
	There exist positive constants $\{ C_i \}_{i=0}^6$ such that under the conditions $n  > C_1 K d ( \log n  )(  \log m )$, $0 \leq t <  C_2 n  / ( d \log n )$, $0 < \alpha \leq C_3$, $C_4 K \sigma \sqrt{ \frac{d + \log m + t }{ \alpha n }  }  < \lambda < C_5 \sigma $ and $0 \leq \varepsilon < C_6 \alpha / K^2$, the following bounds hold for the estimator $\widehat{\bTheta}$ in \eqref{eqn-armul-clustered} with probability at least $1 -  e^{-t}$:
	\begin{align*}
	& \max_{j \in S } \| \widehat{\btheta}_j  - \btheta^{\star}_j \|_2  
	\leq
	C_0 \bigg( \sigma \sqrt{ \frac{ K( d  + t ) }{mn} }
	+     \min \{ K  \delta / \sqrt{\alpha} ,   \lambda   \}  + \varepsilon \lambda  
	\bigg) , \\
	& \max_{j \in S^c }  \| \widehat{\btheta}_j  - \btheta^{\star}_j \|_2  
	\leq     C_0 \lambda  , \\
&  \frac{1}{m} \sum_{j=1}^{m} [ F_j ( \widehat{\btheta}_j ) - F_j ( \btheta^{\star}_j ) ]
\leq  \frac{L}{m} \sum_{j=1}^{m}  \| \widehat{\btheta}_j  - \btheta^{\star}_j \|_2^2 \\
& \leq C_0 L \bigg(
\sigma^2  \frac{ K (d + t) }{mn} +   \min \{ K^2  \delta^2 /  \alpha ,   \lambda^2  \}  + \varepsilon \lambda^2  
\bigg).
	\end{align*}
	In addition, there exists a positive constant $C_7$ that makes the following holds: when $K \delta \leq C_7 \sigma \sqrt{ \frac{ \alpha ( d + \log m ) }{ n } }$, with probability at least $1 - e^{-t}$ there is a permutation $\tau$ of $[K]$ such that $\widehat{\btheta}_j = \widehat{\bbeta}_{\widehat{\bz}_j}$ and $\widehat{z}_j = \tau(z^{\star}_j)$ hold for all $j \in S$.
\end{theorem}

The results directly follow from \Cref{lem-landscape} and \Cref{thm-armul-clustered-deterministic}. We omit the proof because it is almost identical to that of \Cref{thm-vanilla}. 

Suppose that $K$ and $\alpha$ are constants. \Cref{thm-clustered-robust} shows that with high probability, clustered ARMUL with $\lambda \asymp \sigma \sqrt{ \frac{ d + \log m }{n} }$ satisfies the following MSE bound 
\[
 \frac{1}{m} \sum_{j=1}^{m}  \| \widehat{\btheta}_j  - \btheta^{\star}_j \|_2^2
\lesssim  
 \frac{  d }{mn} +   \min \bigg\{   \delta^2  ,   \frac{ d  }{n} \bigg\}  + \frac{ \varepsilon d  }{n} .
\]
where $\lesssim$ hides logarithmic factors. We now establish a matching minimax lower bound.

\begin{theorem}[Minimax lower bound]\label{thm-minimax-clustered}
	Consider the setup in \Cref{eg-gaussian}. Denote by $\Omega ( \varepsilon, \delta )$ the set of all $\bTheta^{\star} = ( \btheta_1^{\star} , \cdots, \btheta_m^{\star} )$ such that Assumption \ref{as-armul-clustered} holds with $K = 2$ and $c_1 = c_2 = 1$.
	There exist universal constants $C,c>0$ such that for any $\varepsilon \in [0, 1/2) $ and $\delta \geq 0$,
	\begin{align*}
		\inf_{\widehat{\bTheta} } \sup_{  \bTheta^{\star} \in \Omega ( \varepsilon , \delta ) } 
		\PP_{\bTheta^{\star}}
		\bigg[
		\frac{1}{m} \sum_{j=1}^{m} \| \widehat{\btheta}_j  - \btheta^{\star}_j \|_2^2
		\geq C \bigg(
		\frac{d}{mn} +
		\min \bigg\{
		\delta^2 , \frac{d}{n}
		\bigg\} + \frac{\varepsilon d}{n}
		\bigg)
		\bigg] 
		\geq c.
	\end{align*}
\end{theorem}

\begin{proof}[\bf Proof of \Cref{thm-minimax-clustered}]
For simplicity, assume that $m$ is even and $\varepsilon m$ is an integer. Define
\begin{align*}
& \cA_1 =  \{ \bTheta \in \RR^{d \times m} :~ 
\btheta^{\star}_{ 1 } = \cdots = \btheta^{\star}_{   m / 2   } = \bm{0} \text{  and }
\btheta^{\star}_{   m / 2    + 1 } = \cdots = \btheta^{\star}_{ m } = \bu \text{ for some } \| \bu \|_2 \geq 1
\}, \\
& \cA_2 =  \{ \bTheta \in \RR^{d \times m} :~ 
		\max_{ 1 \leq j \leq m / 2} \| \btheta^{\star}_{ j }  \|_2 \leq \delta \text{  and }
\max_{ m / 2 + 1 \leq j \leq m } \| \btheta^{\star}_{ j } - \be_1  \|_2 \leq \delta 
\}, \\
& \cA_3 = \{ \bTheta \in \RR^{d \times m} :~ 
\btheta^{\star}_{\varepsilon m + 1 } = \cdots = \btheta^{\star}_{   m / 2   } = \bm{0} \text{  and }
\btheta^{\star}_{   m / 2    + 1 } = \cdots = \btheta^{\star}_{ m } = \be_1
\} .
\end{align*}
We have $\cA_1 \cup \cA_2 \cup \cA_3 \subseteq \Omega ( \varepsilon , \delta ) $. Following the proof of \Cref{thm-minimax-00}, it is easy to show that
\begin{align*}
& \inf_{\widehat{\bTheta} } \sup_{  \bTheta^{\star} \in \cA_1 } 
\PP_{\bTheta^{\star}}
\bigg( 
\frac{1}{m} \sum_{j=1}^{m} \| \widehat{\btheta}_j  - \btheta^{\star}_j \|_2^2
\geq C \frac{d}{m n}
\bigg)
\bigg] 
\geq c, \\
& \inf_{\widehat{\bTheta} } \sup_{  \bTheta^{\star} \in \cA_2 } 
\PP_{\bTheta^{\star}}
\bigg( 
\frac{1}{m} \sum_{j=1}^{m} \| \widehat{\btheta}_j  - \btheta^{\star}_j \|_2^2
\geq C \min \bigg\{ \delta^2,~ \frac{d}{n} \bigg\}
\bigg)
\bigg] 
\geq c, \\
& \inf_{\widehat{\bTheta} } \sup_{  \bTheta^{\star} \in \cA_3 } 
\PP_{\bTheta^{\star}}
\bigg( 
\frac{1}{m} \sum_{j=1}^{m} \| \widehat{\btheta}_j  - \btheta^{\star}_j \|_2^2
\geq C \frac{\varepsilon d}{n}
\bigg)
\bigg] 
\geq c.
\end{align*}
hold for some universal constants $C$ and $c$. The proof is completed by combining the above three bounds.
\end{proof}

\subsection{Proof of \Cref{thm-lowrank} and a minimax lower bound}\label{sec-thm-lowrank-proof}

\Cref{thm-lowrank} directly follow from \Cref{lem-landscape}, \Cref{thm-armul-lowrank-deterministic} and the lemma below. We omit the rest of proof because it is almost identical to that of \Cref{thm-vanilla}.

\begin{lemma}\label{lem-thm-lowrank}
Consider the setup in \Cref{thm-lowrank}.
Define $\bG \in \RR^{d\times m} $ with $\bg_j = \nabla f_j (\btheta^{\star}_j) $. Let $\bP^{\star} \in \RR^{d\times d}$ be the projection onto $\Range(\bB^{\star})$.
There are constants $C_1$ and $C_2$ such that the followings hold with probability at least $1 - e^{-t}$:
\begin{itemize}
\item $ \max_{j \in [m]} \| \bP^{\star} \bg_j \|_2 < C_1 \sigma \sqrt{ \frac{ K + \log m + t}{n} } $;
\item $\| \bG \|_2 \leq \frac{C_2 \sigma }{\sqrt{n}}  (\sqrt{d} + \sqrt{m} + \sqrt{t}) $.
\end{itemize}
\end{lemma}

\begin{proof}[\bf Proof of \Cref{lem-thm-lowrank}]\label{sec-lem-thm-lowrank-proof}
By Assumption \ref{as-stat-2}, $\bP^{\star}  \bg_j $ is a zero-mean sub-Gaussian random vector living in a $K$-dimensional Euclidean space $\Range(\bB^{\star})$, with
\[
\| \bP^{\star}  \bg_j  \|_{\psi_2} \leq \| \bg_j  \|_{\psi_2} \leq \sigma / \sqrt{n}.
\]
By applying analysis of $\max_{j \in [m]} \| \bg_j \|_2$ in \Cref{lem-landscape} to $\max_{j \in [m]} \| \bP^{\star} \bg_j \|_2$, we get a constant $C_1$ such that
\[
\PP \bigg( \max_{j \in [m]} \| \bP^{\star} \bg_j \|_2 < C_1 \sigma \sqrt{ \frac{ K + \log m + t}{n} } \bigg) \geq 1 - e^{-t}/2.
\]
Note that $\bG \in \RR^{d\times m}$ has independent sub-Gaussian columns. By Theorem 5.39 in \cite{Ver10}, there exists a constant $C_2$ such that
\[
\PP \bigg(
\| \bG \|_2 \leq \frac{C_2 \sigma }{\sqrt{n}} (\sqrt{d} + \sqrt{m} + \sqrt{t}) 
\bigg) \geq 1 - e^{-t} / 2.
\]
The proof is finished by union bounds.
\end{proof}

Finally, we establish a minimax lower bound that matches the upper bound in \eqref{eqn-lowrank-upper} for the case $\varepsilon = 0$.

\begin{theorem}[Minimax lower bound]\label{thm-minimax-lowrank}
	Consider the setup in \Cref{eg-gaussian}. Denote by $\Omega ( \delta )$ the set of all $\bTheta^{\star} = ( \btheta_1^{\star} , \cdots, \btheta_m^{\star} )$ such that Assumption \ref{as-armul-lowrank} holds with $K = 1$, $c_1 = 1$, $c_2 = 1 / 2 $ and $\varepsilon = 0$.
	There exist universal constants $C,c>0$ such that for any $\varepsilon \in [0, 1/2) $ and $\delta \geq 0$,
	\begin{align*}
		\inf_{\widehat{\bTheta} } \sup_{  \bTheta^{\star} \in \Omega ( \varepsilon , \delta ) } 
		\PP_{\bTheta^{\star}}
		\bigg[
		\frac{1}{m} \sum_{j=1}^{m} \| \widehat{\btheta}_j  - \btheta^{\star}_j \|_2^2
		\geq C \bigg(
		\frac{d}{mn} + \frac{1}{n}	+ 
		\min \bigg\{
\delta^2 , \frac{d}{n}
		\bigg\}  
		\bigg)
		\bigg] 
		\geq c.
	\end{align*}
\end{theorem}

\begin{proof}[\bf Proof of \Cref{thm-minimax-lowrank}]
	For simplicity, assume that $m$ is even. Define
	\begin{align*}
		& \cA_1 =  \{ \bTheta \in \RR^{d \times m} :~ 
		\btheta^{\star}_{ 1 } = \cdots = \btheta^{\star}_{  m / 2 } = \bm{0},~
		 \btheta^{\star}_{ m / 2 + 1 } =  \cdots = \btheta^{\star}_{  m } =  \bu \text{ for some } \bu \in \SSS^{d-1}
		\}, \\
		& \cA_2 =  \{ \be_1 \bv^{\top} :~ 
		v_{ 1 } = \cdots = v_{  m / 2 } = 1,~ | v_j | \leq 1 \text{ for } m / 2 + 1 \leq j \leq m
		\}, \\
		& \cA_3 = \{ \bTheta \in \RR^{d \times m} :~ 
		\max_{ 1 \leq j \leq m / 2} \| \btheta^{\star}_{ j }  \|_2 \leq \delta \text{  and }
		\max_{ m / 2 + 1 \leq j \leq m } \| \btheta^{\star}_{ j } - \be_1  \|_2 \leq \delta 
		\} .
	\end{align*}
We have $\cA_1 \cup \cA_2 \cup \cA_3 \subseteq \Omega ( \delta ) $. 
	Following the proof of \Cref{thm-minimax-clustered}, it is easy to show that
	\begin{align*}
		& \inf_{\widehat{\bTheta} } \sup_{  \bTheta^{\star} \in \cA_1 } 
		\PP_{\bTheta^{\star}}
		\bigg( 
		\frac{1}{m} \sum_{j=1}^{m} \| \widehat{\btheta}_j  - \btheta^{\star}_j \|_2^2
		\geq C \frac{d}{m n}
		\bigg)
		\bigg] 
		\geq c, \\
		& \inf_{\widehat{\bTheta} } \sup_{  \bTheta^{\star} \in \cA_2 } 
\PP_{\bTheta^{\star}}
\bigg( 
\frac{1}{m} \sum_{j=1}^{m} \| \widehat{\btheta}_j  - \btheta^{\star}_j \|_2^2
\geq \frac{C}{n}
\bigg)
\bigg] 
\geq c, \\		
		& \inf_{\widehat{\bTheta} } \sup_{  \bTheta^{\star} \in \cA_3 } 
		\PP_{\bTheta^{\star}}
		\bigg( 
		\frac{1}{m} \sum_{j=1}^{m} \| \widehat{\btheta}_j  - \btheta^{\star}_j \|_2^2
		\geq C \min \bigg\{ \delta^2 ,~ \frac{d}{n} \bigg\}
		\bigg)
		\bigg] 
		\geq c .
	\end{align*}
	hold for some universal constants $C$ and $c$. The proof is completed by combining the bounds.
\end{proof}

\section{Technical lemmas}

\begin{lemma}\label{lem-consensus-min}
	Let $f:\RR^d \to \RR$ be a convex function. Suppose there exist $\bx_0 \in \RR^d$, $\bg \in \partial f(\bx_0)$, $\rho > 0$, $r > 0$ such that $\| \bg \|_{2} < r \rho / 2$ and
	\[
	f (\bx) \geq f(\bx_0) + \langle \bg, \bx - \bx_0 \rangle + \frac{\rho}{2} \| \bx - \bx_0 \|_2^2, \qquad \forall \bx  \in B(\bx_0, r).
	\]
	Then $\argmin_{\bx \in \RR^d} f (\bx) \subseteq B (\bx_0, 2 \| \bg \|_{2} / \rho )$. Furthermore, if $\nabla^2 f (\bx) \succeq \rho \bI$ for all $\bx \in B(\bx_0, r)$, then $f $ has a unique minimizer and it belongs to $B (\bx_0, \| \nabla f(\bx_0) \|_{2} / \rho )$.
\end{lemma}

\begin{proof}[\bf Proof of \Cref{lem-consensus-min}]
	For any $\bx \in B(\bx_0, r)$,
	\begin{align*}
	f (\bx) & \geq f(\bx_0) + \langle \bg, \bx - \bx_0 \rangle + \frac{\rho}{2} \| \bx - \bx_0 \|_2^2
	\geq  f(\bx_0) - \| \bg \|_{2} \| \bx - \bx_0 \|_2 + \frac{\rho}{2} \| \bx - \bx_0 \|_2^2 \\
	& \geq  f(\bx_0) + \frac{\rho}{2} \| \bx - \bx_0 \|_2 (
	\| \bx - \bx_0 \|_2 -  2 \| \bg \|_{2} / \rho ).
	\end{align*}
	Hence $f(\bx) > f(\bx_0)$ when $2 \| \bg \|_{2} / \rho < \| \bx - \bx_0 \|_2 \leq r$.
	When $\| \bx - \bx_0 \|_2 > r$, there exists $\bz = (1-t) \bx_0 + t \bx $ for some $t \in ( 0,1 )$ such that $\| \bz - \bx_0 \|_2 = r$. By $f (\bz) > f(\bx_0)$ and the convexity of $f$, we have 
	\[
	f(\bx_0) < f(\bz) \leq (1-t) f(\bx_0) + t f(\bx)
	\]
	and hence $f(\bx) > f(\bx_0)$. Therefore, $\argmin_{\bx} f(\bx) \subseteq B(\bx_0, 2 \| \bg \|_{2} / \rho) $.

	Now, suppose that $\nabla^2 f (\bx) \succeq \rho \bI$ for all $\bx \in B(\bx_0, r)$. From
	\[
	\argmin_{\bx} f(\bx) \subseteq B(\bx_0, 2 \| \bg \|_{2} / \rho) \subseteq B(\bx_0, r)
	\]
	and the strong convexity of $f$ therein we get the uniqueness of $f$'s minimizer. Denote it by $\bx^{\star}$. Then $\nabla f(\bx^{\star}) = \bm{0}$ and $\| \bx^{\star} - \bx_0 \|_2 \leq r$, the proof is finished by
\begin{align*}
\| \nabla f(\bx_0) \|_2 & = \| \nabla f(\bx_0) - \nabla f(\bx^{\star}) \|_2 \\
& = \bigg\| \bigg( \int_{0}^{1} \nabla^2 f[(1-t) \bx^{\star} + t \bx_0 ] \rd t \bigg) (\bx_0 - \bx^{\star}) \bigg\|_2 \geq \rho \| \bx_0 - \bx^{\star} \|_2.
\end{align*}
\end{proof}

\begin{lemma}\label{lem-cvx-reg}
	Let $f:\RR^d \to \RR$ be a convex function and $\bx_0 = \argmin_{\bx} f(\bx)$. Suppose there exist $\rho > 0$ and $r > 0$ such that $\nabla^2 f(\bx) \succeq \rho \bI$, $\forall \bx \in B(\bx_0, r)$. We have
	\begin{align}
	\| \bm{f} \|_2 \geq \rho \min\{ \| \bx - \bx_0 \|_2, r \} , \qquad \forall \bm{f} \in \partial f(\bx), ~~ \bx \in \RR^d.
	\label{eqn-lem-cvx-reg}
	\end{align}
	If $g:\RR^d \to \RR$ is convex and $\lambda$-Lipschitz for some $\lambda < \rho r $, then $f (\bx) + g(\bx) $ has a unique minimizer and it belongs to $B (\bx_0, \lambda / \rho )$.
\end{lemma}

\begin{proof}[\bf Proof of \Cref{lem-cvx-reg}]
	The optimality of $\bx_0$ and the strong convexity of $f$ near $\bx_0$ implies $\nabla f(\bx_0) = \bm{0}$. If $\| \bx - \bx_0 \|_2 \leq r$, then
	\begin{align}
	&\| \nabla  f(\bx) \|_2 \| \bx - \bx_0 \|_2 \geq  \langle \nabla f(\bx) , \bx - \bx_0 \rangle   = \langle  \nabla f(\bx) - \nabla f(\bx_0) , \bx - \bx_0 \rangle \notag \\
	& =  \bigg\langle  \bigg( \int_{0}^{1} \nabla^2 f[(1-t) \bx_0 + t \bx ] \rd t \bigg) (\bx - \bx_0) , \bx - \bx_0 \bigg\rangle  \geq \rho \| \bx - \bx_0 \|_2^2
	\label{eqn-lem-cvx-reg-1}
	\end{align}
	and $\| \nabla f(\bx) \|_2 \geq \rho \| \bx - \bx_0 \|_2$. If $\| \bx - \bx_0 \|_2 > r$, there exists $\bz = (1-t) \bx_0 + t \bx $ for some $t \in ( 0,1 )$ such that $\| \bz - \bx_0 \|_2 = r$. Choose any $\bm{f} \in \partial f(\bx)$. By the convexity of $f$, $\langle \bm{f} -  \nabla f(\bz) , \bx - \bz \rangle \geq 0$ and hence $\langle \bm{f} -  \nabla f(\bz) , \bx - \bx_0 \rangle \geq 0$. Then
	\begin{align*}
	\| \bm{f} \|_2 \| \bx - \bx_0 \|_2 & \geq \langle \bm{f}  , \bx - \bx_0 \rangle
	= 
	\langle \bm{f} - \nabla f(\bx_0) , \bx - \bx_0 \rangle \\
	& = \langle \bm{f} - \nabla f(\bz) , \bx - \bx_0 \rangle + \langle \nabla f(\bz) - \nabla f(\bx_0) , \bx - \bx_0 \rangle \\
	& \geq \langle \nabla f(\bz) - \nabla f(\bx_0) , \bx - \bx_0 \rangle
	\geq \rho r \| \bx - \bx_0 \|_2,
	\end{align*}
	where the last inequality follows from \Cref{eqn-lem-cvx-reg-1}. We have verified \Cref{eqn-lem-cvx-reg}.

	Choose any $\bx^{\star} \in \argmin_{\bx \in \RR^d} \{ f (\bx) + g(\bx) \}$. There exists $\bm{f} \in \partial f(\bx^{\star})$ and $\bg \in \partial g(\bx^{\star})$ such that $\bm{f} + \bg = \bm{0}$. The Lipschitz property of $g$ yields $\| \bm{f} \|_2 = \| \bg \|_2 \leq \lambda$. Since $\lambda < \rho r$, we obtain from \Cref{eqn-lem-cvx-reg} that $\| \bx^{\star} - \bx_0 \|_2 \leq \lambda / \rho$. The minimizer $\bx^{\star}$ of $f + g$ belongs to $B(\bx_0,\lambda / \rho)$, where the function is strongly convex. Consequently, it is the unique minimizer.
\end{proof}

\begin{lemma}\label{lem-inf-conv}
	Let $f$ be convex and differentiable. Suppose that
	\begin{align*}
	& \nabla^2 f (\bx) \preceq L \bI,\qquad \forall \bx \in B(\bx_0, M )
	\end{align*}
	holds for some $\bx_0 \in \RR^d$, $0 < M \leq +\infty$ and $0  \leq  L< + \infty$. If $ \lambda > \| \nabla f (\bx_0) \|_{2} $, then
\begin{align*}
 f (\bx) = f \square (\lambda \| \cdot \|_2) (\bx) \qquad\text{and}\qquad \argmin_{\by \in \RR^d} \{ f(\by) + \lambda \| \bx - \by \|_2 \} = \bx
\end{align*}
hold for all $\bx \in B ( \bx_0 , \min\{ ( \lambda - \| \nabla f (\bx_0) \|_{2}  ) / L , M \} )$.
\end{lemma}

\begin{proof}[\bf Proof of \Cref{lem-inf-conv}]
Let $\| \bx - \bx_0 \|_2 \leq \min\{ ( \lambda - \| \nabla f (\bx_0) \|_{2}  ) / L , M \}$. We have
\begin{align*}
\| \nabla f(\bx) \|_2 & \leq \| \nabla f(\bx_0) \|_2 + \| \nabla f(\bx) - \nabla f (\bx_0) \|_2 \leq \| \nabla f (\bx_0) \|_{2} + L \| \bx - \bx_0 \|_2 \leq \lambda.
\end{align*}
It follows from $\partial \| \by \|_2 |_{\by = 0} = \{ \bg \in \RR^d:~ \| \bg \|_{2} \leq 1 \}$ that $f \square (\lambda \| \cdot \|_2) (\bx) = f (\bx)$. For any $\by \neq \bx$,
\begin{align*}
f(\by) + \lambda \| \bx - \by \|_2 & \geq f(\bx) + \langle \nabla f(\bx) , \by - \bx \rangle + \lambda \| \bx - \by \|_2  \\
& \geq f(\bx) + (\lambda - \| \nabla f(\bx) \|_2) \| \bx - \by \|_2 > f(\bx).
\end{align*}
Hence $\bx$ is the unique minimizer of $\by \mapsto f(\by) + \lambda \| \bx - \by \|_2$.
\end{proof}

\begin{lemma}\label{lem-inf-conv-lip}
	If $f:~\RR^d \to \RR$ is convex, $\inf_{\bx \in \RR^d} f(\bx) > -\infty$, $R:~\RR^d \to \RR$ is convex and $L$-Lipschitz with respect to a norm $\| \cdot \|$ for some $L \geq 0$, then $f \square R$ is convex and $L$-Lipschitz with respect to $\| \cdot \|$.
\end{lemma}

\begin{proof}[\bf Proof of \Cref{lem-inf-conv-lip}]
	The convexity of $f \square R$ can be found in standard textbooks of convex analysis \citep{HLe13}. Now we prove the Lipschitz property. Note that $f \square R (\bx) = \inf_{\by \in \RR^d} \{ f(\by) + R(\bx - \by) \}$. For any $\by \in \RR^d$, $\bx \mapsto R( \bx - \by)$ is $L$-Lipschitz. Then $f \square R $ is also $L$-Lipschitz since it is the infimum of such functions.
\end{proof}

\begin{lemma}\label{lem-cvx-lower}
	Let $f:~\RR^d \to \RR$ be a convex function. Suppose there exist $\bx^{\star} \in \RR^d$, $\rho > 0$ and $r > 0$ such that $\nabla^2 f(\bx) \succeq \rho \bI$ in $B(\bx^{\star}, r)$. Then
	\begin{align*}
	& f(\bx) \geq  f(\bx^{\star}) + \langle \nabla  f (\bx^{\star}) , \bx - \bx^{\star} \rangle  + \rho \cdot H ( \| \bx - \bx^{\star} \|_2 ), \qquad \forall \bx \in \RR^d,
	\end{align*}
	where
	\[
	H(x) = \begin{cases}
	x^2 / 2 ,&  \mbox{if } 0 \leq x \leq r \\
	r(x - r/2) ,&  \mbox{if } x > r
	\end{cases}.
	\]
\end{lemma}

\begin{proof}[\bf Proof of \Cref{lem-cvx-lower}]
	There is nothing to prove for $\bx \in B ( \bx^{\star} , r)$. If $\bx \notin B ( \bx^{\star} , r)$, define $t = r / \| \bx - \bx^{\star} \|_2$ and $\bu = (1 - t) \bx^{\star} + t \bx$. We have $\| \bu - \bx^{\star} \|_2 = r$ and
	\begin{align*}
	& f(\bx) - f(\bx^{\star}) - \langle \nabla  f (\bx^{\star}) , \bx - \bx^{\star} \rangle \\
&	=
[	f(\bx) - f(\bu) - \langle \nabla  f (\bx^{\star}) , \bx - \bu \rangle ]
+ [ f(\bu) - f(\bx^{\star}) - \langle \nabla  f (\bx^{\star}) , \bu - \bx^{\star} \rangle ] \\
&	\geq \langle \nabla f (\bu) - \nabla f (\bx^{\star}) , \bx - \bu \rangle + \frac{\rho}{2} \| \bu - \bx^{\star} \|_2^2 \\
& = \bigg\langle
	\bigg(
	\int_{0}^{1} \nabla^2 f [ (1 - s) \bx^{\star} + s \bu ] \rd s
	\bigg) (\bu - \bx^{\star}), \bx - \bu
	\bigg\rangle + \frac{\rho r^2}{2} .
	\end{align*}
	Note that $\bu - \bx^{\star} = t (\bx - \bx^{\star})$, $\bx - \bu = (1 - t) (\bx - \bx^{\star})$ and $\nabla^2 f [ (1 - s) \bx^{\star} + s \bu ] \succeq \rho \bI$ for $s \in [0, 1]$. Then,
	\begin{align*}
	& f(\bx) - f(\bx^{\star}) - \langle \nabla  f (\bx^{\star}) , \bx - \bx^{\star} \rangle  \geq \rho t (1-t) \| \bx - \bx^{\star} \|_2^2 + \frac{\rho r^2}{2} \\
	& = \rho \cdot  \frac{r }{\| \bx - \bx^{\star} \|_2} \bigg( 1 -  \frac{r }{\| \bx - \bx^{\star} \|_2} \bigg)  \| \bx - \bx^{\star} \|_2^2 + \frac{\rho r^2}{2}  
	= \rho r ( \| \bx - \bx^{\star} \|_2 - r / 2 ).
	\end{align*}
\end{proof}

\begin{lemma}\label{lem-clustered-lower-bound}
Suppose there are $0 < \rho, L ,\delta < +\infty$, $0 < M \leq +\infty$ and $\{ \btheta^{\star}_j \}_{j=1}^m \subseteq \RR^d$ such that 
	\begin{align*}
	&\rho \bI \preceq \nabla^2 f_j (\btheta) \preceq L \bI , ~~ \forall \btheta \in B(\btheta^{\star}_j, M)
	\qquad\text{and}\qquad
	\| \nabla f_j ( \btheta^{\star}_{j} ) \|_2 \leq \delta
	\end{align*}
	hold for all $j \in [m] $. Define $\widetilde{f}_j = f_j \square (\lambda \| \cdot \|_2)$ and $\widetilde{\btheta}_j = \argmin_{\btheta \in \RR^d} \widetilde{f}_j(\btheta)$. When 
	\[
	3 \delta L / \rho + Lr <  \lambda < LM
	\]
	for some $r > 0$, we have
	\begin{align*}
	& \widetilde{f}_j(\btheta) =  f_j (\btheta) , \qquad \forall \btheta \in B ( \btheta^{\star}_j, 2 \delta / \rho + r) ~~\text{and}~~j \in [m]; \\
	& \| \widetilde{\btheta}_j - \btheta^{\star}_{j} \|_2 \leq  \delta / \rho , \qquad\forall j \in [m]; \\
	& \sum_{j=1}^{m} \widetilde{f}_j ( \btheta_{j} ) - \sum_{j=1}^{m} \widetilde{f}_j ( \btheta^{\star}_{j} ) \geq \rho  \sum_{j =1}^m H \Big(
	(\| \btheta_{ j} - \btheta^{\star}_{j} \|_2 -  \delta/\rho
	)_{+} 
	\Big) 
	- \frac{ m \delta^2}{\rho}, \qquad \forall \bTheta \in \RR^{d\times m}.
	\end{align*}
	Here $H(t) = t^2 / 2$ if $0 \leq t \leq r$ and $H(t) = r(t-r/2)$ if $t > r$.
\end{lemma}

\begin{proof}[\bf Proof of \Cref{lem-clustered-lower-bound}]
	Note that $\| \nabla f_j ( \btheta^{\star}_{j} ) \|_2 \leq \delta$ and
	\begin{align*}
	& \| \nabla f_j (\btheta^{\star}_{j} ) \|_{2} + \frac{2  L   \| \nabla f_j  (\btheta^{\star}_{j} ) \|_2 
	}{ \rho } 
	\leq \delta +  \frac{2  L  \delta }{ \rho }  \leq   \frac{3  L  \delta }{  \rho } < \lambda - Lr \leq \lambda, \\
	&M \geq ( \lambda - \| \nabla f_j (\btheta^{\star}_{j} ) \|_{2} ) / L \geq ( \lambda - \delta ) / L > 2 \delta / \rho + r.
	\end{align*}
	\Cref{thm-consensus} applied to $f_j$ shows that
	\begin{align}
	& \widetilde{f}_j(\btheta) =  f_j (\btheta) , \qquad \forall \btheta \in B ( \btheta^{\star}_{j}, 2 \delta / \rho + r) ;
	\label{eqn-clustered-1}
	\\
	& \widetilde{\btheta}_j  = \argmin_{\btheta} \widetilde{f}_j(\btheta) = \argmin_{\btheta \in \RR^d}  f_j (\btheta) \in B ( \btheta^{\star}_{j} ,  \delta / \rho ).
	\label{eqn-clustered-2}
	\end{align}
	By the convexity of $\widetilde{f}_j$,
	\begin{align*}
	\widetilde{f}_j(\btheta^{\star}_{j}) &\leq  \widetilde{f}_j(\widetilde{\btheta}_j) - \langle \nabla \widetilde{f}_j(\btheta^{\star}_{j}) , \widetilde{\btheta}_j - \btheta^{\star}_{j} \rangle 
	= f_j(\widetilde{\btheta}_j) - \langle \nabla f_j(\btheta^{\star}_{j}) , \widetilde{\btheta}_j - \btheta^{\star}_{j} \rangle \notag \\
	& \leq f_j(\widetilde{\btheta}_j) + \| \nabla f_j(\btheta^{\star}_{j}) \|_2 \| \widetilde{\btheta}_j - \btheta^{\star}_{j} \|_2 
	\leq  f_j(\widetilde{\btheta}_j) + \frac{ \delta^2}{\rho}.
	\end{align*}
	We have
	\begin{align}
\sum_{j=1}^{m} \widetilde{f}_j(\btheta^{\star}_{j}) 
	\leq \sum_{j=1}^{m} f_j(\widetilde{\btheta}_j) + \frac{m \delta^2}{\rho}.
	\label{eqn-clustered-3}
	\end{align}
	
	We now establish a lower bound on $\sum_{j=1}^{m} \widetilde{f}_j ( \btheta_{j} ) $ for general $\bTheta$. By (\ref{eqn-clustered-1}) and (\ref{eqn-clustered-2}),
	\begin{align*}
	\widetilde{f}_j(\btheta) =  f_j (\btheta) , \qquad \forall \btheta \in B ( \widetilde{\btheta}_j, r) .
	\end{align*}
	By $\nabla^2 f_j (\btheta)  \succeq \rho \bI$ for $\btheta \in B(\btheta^{\star}_j, M)$, $M \geq r$, $\nabla f_j (\widetilde{\btheta}_j) = \bm{0}$ and \Cref{lem-cvx-lower},
	\begin{align*}
	& \widetilde{f}_j(\btheta) - \widetilde{f}_j(\widetilde{\btheta}_j) \geq \rho \cdot H ( \| \btheta - \widetilde{\btheta}_j \|_2 )
	\quad\text{where}\quad
	H(x) = \begin{cases}
	x^2 / 2 ,&  \mbox{if } 0 \leq x \leq r \\
	r(x - r/2) ,&  \mbox{if } x > r
	\end{cases}.
	\end{align*}
	We get a lower bound
	\begin{align}
 \sum_{j =1}^m \widetilde{f}_j(\btheta_{j}) \geq 
	\sum_{j=1}^{m} f_j(\widetilde{\btheta}_j) + \rho \sum_{j =1}^m H ( \| \btheta_{j} - \widetilde{\btheta}_j \|_2 ), \qquad \forall \bTheta .
	\label{eqn-clustered-5}
	\end{align}
	
	By the triangle's inequality and (\ref{eqn-clustered-2}),
	\begin{align*}
	\| \btheta_{j} - \widetilde{\btheta}_j \|_2& \geq \Big|
	\| \btheta_{j} - \btheta^{\star}_{j} \|_2 - \| \btheta^{\star}_{j} - \widetilde{\btheta}_j \|_2
	\Big| \geq (\| \btheta_{j} - \btheta^{\star}_{j} \|_2 - \| \btheta^{\star}_{j} - \widetilde{\btheta}_j \|_2
	)_{+} \\
	& \geq (\| \btheta_{j} - \btheta^{\star}_{j} \|_2 - \delta/\rho
	)_{+} .
	\end{align*}
	Since $H(\cdot)$ is increasing in $[0, +\infty)$, the proof is finished by combining \Cref{eqn-clustered-3} and \Cref{eqn-clustered-5}.
\end{proof}

\begin{lemma}\label{lem-frac}
	Let $\bA \in \RR^{d\times K} \backslash \{ \bm{0} \}$, $\{ \bz_{j} \}_{j=1}^m \subseteq \RR^K$ and $\{ w_j \}_{j=1}^m \subseteq [0,+\infty)$. Suppose that $\max_{j \in [m]} \| \bz_j \|_2 \leq \alpha_1$, $\sum_{j=1}^{m} w_j \bz_j \bz_j^{\top} \succeq ( \alpha_2^2 \sum_{j=1}^{m} w_j / K ) \bI_K$ for some $0 < \alpha_2 \leq \alpha_1$ and $\sum_{j=1}^{m} w_j > 0$. We have
	\[
	\sum_{ j \in [m]:~\| \bA \bz_j \|_2 > t \| \bA \|_2} w_j \geq \frac{\alpha_2^2 / K - t^2}{\alpha_1^2 - t^2} \sum_{j = 1}^m w_j , \qquad \forall t \in [0, \alpha_2 / \sqrt{K}] .
	\]
\end{lemma}

\begin{proof}[\bf Proof of \Cref{lem-frac}]
	Without loss of generality, assume that $\sum_{j=1}^{m} w_j = 1$. On the one hand,
	\[
	\sum_{j = 1}^m w_j \| \bA \bz_j \|_2^2 = \bigg\langle \bA^{\top} \bA , \sum_{j=1}^{m} w_j \bz_j \bz_j^{\top} \bigg\rangle \geq \frac{\alpha_2^2}{K} \| \bA \|_2^2.
	\]
	On the other hand, for $S = \{ j \in [m]:~\| \bA \bz_j \|_2 > t \| \bA \|_2 \}$ we have
	\begin{align*}
	\sum_{j = 1}^m w_j \| \bA \bz_j \|_2^2 & \leq \sum_{j \in S} w_j (\alpha_1 \| \bA \|_2)^2 + \sum_{j \in S^c} w_j (t \| \bA \|_2 )^2  \\
	& =\| \bA\|_2^2 \bigg[
	\alpha_1^2 \sum_{j \in S} w_j + \bigg( 1 - \sum_{j \in S} w_j \bigg) t^2
	\bigg].
	\end{align*}
	The claim directly follows from these estimates.
\end{proof}

{
\bibliographystyle{ims}
\bibliography{bib}
}

\end{document}